%% file: colt2026_ALFCG.tex
\documentclass[final,12pt]{arxiv2025}

\input{myinit.tex}

\usepackage{algorithm}
\usepackage{graphicx}

\renewcommand{\cite}{\citep}

\def\v{\mathbf{v}} \def\x{\mathbf{x}} \def\s{\mathbf{s}} \def\g{\mathbf{g}} \def\X{\mathbf{X}} \def\a{\mathbf{a}} \def\b{\mathbf{b}} \def\z{\mathbf{z}}
 \def\GG{\mathcal{G}} \def\OO{\mathcal{O}} \def\PP{\mathcal{P}}
\def\EEE{\mathbb{E}} \def\KKK{\mathbb{K}} \def\UUU{\mathbb{U}}

\title[Adaptive Lipschitz-Free Conditional Gradient Methods]{Adaptive Lipschitz-Free Conditional Gradient Methods for Stochastic Composite Nonconvex Optimization}
\usepackage{times}

\coltauthor{%
 \Name{Ganzhao Yuan} \Email{yuanganzhao@foxmail.com}\\
 \addr Shenzhen University of Advanced Technology (SUAT), China
}

\begin{document}

\maketitle

\begin{abstract}


We propose ALFCG (Adaptive Lipschitz-Free Conditional Gradient), the first \textit{adaptive} projection-free framework for stochastic composite nonconvex minimization that \textit{requires neither global smoothness constants nor line search}. Unlike prior conditional gradient methods that use openloop diminishing stepsizes, conservative Lipschitz constants, or costly backtracking, ALFCG maintains a self-normalized accumulator of historical iterate differences to estimate local smoothness and minimize a quadratic surrogate model at each step. This retains the simplicity of Frank-Wolfe while adapting to unknown geometry. We study three variants. ALFCG-FS addresses finite-sum problems with a SPIDER estimator. ALFCG-MVR1 and ALFCG-MVR2 handle stochastic expectation problems by using momentum-based variance reduction with single-batch and two-batch updates, and operate under average and individual smoothness, respectively. To reach an $\epsilon$-stationary point, ALFCG-FS attains $\mathcal{O}(N+\sqrt{N}\epsilon^{-2})$ iteration complexity, while ALFCG-MVR1 and ALFCG-MVR2 achieve $\tilde{\mathcal{O}}(\sigma^2\epsilon^{-4}+\epsilon^{-2})$ and $\tilde{\mathcal{O}}(\sigma\epsilon^{-3}+\epsilon^{-2})$, where $N$ is the number of components and $\sigma$ is the noise level. In contrast to typical $\mathcal{O}(\epsilon^{-4})$ or $\mathcal{O}(\epsilon^{-3})$ rates, our bounds reduce to the optimal rate up to logarithmic factors $\tilde{\mathcal{O}}(\epsilon^{-2})$ as the noise level $\sigma \to 0$. Extensive experiments on multiclass classification over nuclear norm balls and $\ell_p$ balls show that ALFCG generally outperforms state-of-the-art conditional gradient baselines.

\end{abstract}


\section{Introduction}
This paper studies the stochastic composite nonconvex minimization problem given by:
\begin{align}\label{eq:main}
\min_{\mathbf{x} \in \mathcal{X} \subset \mathbb{R}^n}\,~F(\mathbf{x}):= f(\mathbf{x}) + h(\mathbf{x}),
\end{align}
where $\mathcal{X}$ is a compact convex set, $h(\cdot)$ is a proper, closed, convex function, and $f(\mathbf{x})$ is a differentiable, possibly nonconvex function. We consider two distinct settings for $f(\mathbf{x})$:
\begin{equation}
f(\mathbf{x}) := \begin{cases}
\frac{1}{N} \sum_{i=1}^{N} f_i(\mathbf{x}) & (\text{Finite-Sum}) \\
\mathbb{E}_{\xi \sim \mathcal{D} }[ f(\mathbf{x};\xi)] & (\text{Expectation})
\end{cases}.\nonumber
\end{equation}
Here, $f(\mathbf{x})$ represents either an empirical risk over $N$ samples or the expectation of a stochastic function $f(\mathbf{x};\xi)$ with random vector $\xi$.

We focus on projection-free settings where Euclidean projections onto $\mathcal{X}$ are computationally prohibitive, whereas linear optimization over $\mathcal{X}$ is relatively efficient. Specifically, we assume access to a Linear Minimization Oracle (LMO), which favors conditional gradient methods for solving \eqref{eq:main}.

\textbf{Motivation}. Conditional Gradient (CG), or Frank-Wolfe (FW) algorithms \cite{FrankWolfe1956, ICMLjaggi13}, have become the paradigm of choice for constrained optimization where Euclidean projections are costly (e.g., nuclear-norm constraints). The standard update is:
\begin{equation}
\mathbf{v}_t \in \arg \min_{\mathbf{x}\in\mathcal{X}} \langle \mathbf{x},\nabla f(\mathbf{x}_t) \rangle + h(\mathbf{x}),\,\quad \mathbf{x}_{t+1} = (1-\eta_t)\mathbf{x}_t + \eta_t \mathbf{v}_t,\nn
\end{equation}
typically using $\eta_t = \frac{2}{t+2}$. However, the practical performance is critically dependent on the choice of the step size $\eta_t$. Exact line-search procedures, which solve $\eta_t \in \arg \min_{\eta \in [0,1]} F(\x_t(\eta))$ with $\x_t(\eta):=(1-\eta)\mathbf{x}_t + \eta \mathbf{v}_t$, are often computationally prohibitive, particularly in stochastic settings where function values of $f(\cdot)$ are unavailable or noisy. Alternatively, Lipschitz-based methods consider a quadratic majorant:
\begin{equation*}
f(\mathbf{y}) \leq \mathcal{Q}(\mathbf{y};\mathbf{x}) := f(\mathbf{x}) + \langle \nabla f(\mathbf{x}), \mathbf{y}-\mathbf{x}\rangle + \tfrac{L}{2}\|\mathbf{y}-\mathbf{x}\|_2^2.
\end{equation*}
\noi Incorporating the convexity of $h(\cdot)$ yields a step size rule that minimizes a composite upper bound: $\eta_t \in \arg \min_{\eta \in [0,1]}  (1-\eta) h(\mathbf{x}_t) + \eta h(\mathbf{v}_t) + \mathcal{Q}(\mathbf{x}_t(\eta);\mathbf{x}_t)$. This subproblem admits a closed-form solution: $\bar{\eta}_t = \min\big( \frac{ h(\mathbf{x}_t) - h(\mathbf{v}_t) - \langle \nabla f(\x_t), \mathbf{v}_t - \mathbf{x}_t \rangle }{ L \|\mathbf{v}_t - \mathbf{x}_t\|_2^2 }, 1 \big)$. However, the Lipschitz constant $L$ is often unknown, conservative, or locally variable, leading to suboptimal convergence.

\textbf{Contributions}. To address these limitations, we propose the Adaptive Lipschitz-Free Conditional Gradient (ALFCG). Our contributions are threefold:

\begin{enumerate}[label=\textbf{(\alph*)}, leftmargin=20pt, itemsep=1pt, topsep=1pt, parsep=0pt, partopsep=0pt]

\item \textbf{Lipschitz-Free \& Model-Based Design}. We propose ALFCG, an adaptive projection-free framework that eliminates the need for line search or prior knowledge of global smoothness constants. ALFCG dynamically estimates the local Lipschitz constant $L_t$ using a self-normalized accumulator of past iterate differences. We develop three variants: ALFCG-FS (using SPIDER for finite-sum problems), and ALFCG-MVR1 and ALFCG-MVR2 (using single- and two-batch momentum-based variance reduction, respectively, for stochastic expectation problems). Notably, the algorithm operates independently of the global constant $L$, which appears solely in the theoretical analysis.

\item \textbf{Rigorous Theoretical Guarantees}. We prove that ALFCG-FS, ALFCG-MVR1, and ALFCG-MVR2 achieve gradient complexities of $\mathcal{O}(N + \sqrt{N}\epsilon^{-2})$, $\tilde{\mathcal{O}}(\sigma^2\epsilon^{-4} + \epsilon^{-2} )$, and $\tilde{\mathcal{O}}( \sigma\epsilon^{-3} + \epsilon^{-2})$, respectively. To the best of our knowledge, these are the first adaptive, Lipschitz-free CG methods to achieve optimal iteration complexity for this class of problems, matching the lower bounds in \cite{arjevani2023lower}. A notable feature of our work is the unified nature of the convergence analysis. We demonstrate that if the noise parameter is chosen such that $\beta \propto \sigma^2$, our result bridges the gap between stochastic and deterministic optimization: the complexity smoothly reduces to the optimal rate up to logarithmic factors $\tilde{\mathcal{O}}(\epsilon^{-2})$ as the noise level $\sigma \to 0$. This contrasts with prior art where convergence bounds often retain suboptimal dependencies even when the noise is negligible.

\item \textbf{Empirical Superiority}. Extensive experiments on multiclass classification constrained by nuclear norm balls and $\ell_p$ balls demonstrate that ALFCG generally outperforms state-of-the-art adaptive FW baselines.

\end{enumerate}

\textbf{Organization}. The remainder of this paper is structured as follows: Section \ref{sect:related} reviews related literature, and Section \ref{sect:prel} establishes the necessary assumptions and technical preliminaries. The proposed ALFCG framework and its variants are detailed in Section \ref{sect:proposed}, followed by a rigorous iteration complexity analysis in Section \ref{sect:iter}. Section \ref{sect:conc} concludes the paper. Motivating applications, detailed proofs, and experimental results are deferred to the Appendix.

\section{Related Work}
\label{sect:related}

The development of projection-free optimization methods has been extensive, driven by their efficiency in handling complex constraints where linear minimization is significantly cheaper than Euclidean projection \cite{ICMLjaggi13,lacoste2013block}. We categorize the related literature into deterministic, finite-sum, and expectation settings.

\textbf{Deterministic Frank-Wolfe Algorithms}. The foundational Frank-Wolfe (FW) algorithm \cite{FrankWolfe1956,MMP1996}, also known as the Conditional Gradient method, has been revitalized for large-scale machine learning due to its projection-free nature \cite{ICMLjaggi13}. Early works established a convergence rate of $\mathcal{O}(\epsilon^{-2})$ for smooth nonconvex objectives, but typically required either diminishing step-size schedules (e.g., $2/(t+2)$) or knowledge of the Lipschitz constant \cite{CGBook2025}. We refer to these variants as FW-Openloop and FW-ShortStep, respectively. While subsequent improvements like FW-Sliding \cite{lan2017conditionalSIOPT} and FW-Momentum \cite{li2021heavy} incorporate acceleration and momentum strategies to boost convergence, they remained tied to global smoothness parameters. To circumvent this dependency, adaptive variants such as FW-Armijo \cite{ochs2019model,ghadimi2019conditional} and FW-ParaFree \cite{ito2023parameter} were proposed, achieving ``Lipschitz-Free" optimization via backtracking line searches. However, these methods necessitate function evaluations to verify descent conditions, which can be prohibitively costly in large-scale applications. In contrast, we introduce ALFCG-D (our finite-sum variant with $N=1$). Unlike existing line search-based methods, ALFCG-D relies solely on gradient information, avoiding expensive function value queries of $f(\cdot)$. It achieves the optimal $\mathcal{O}(\epsilon^{-2})$ rate in a fully adaptive manner without requiring global smoothness parameters.

\textbf{Finite-Sum Optimization.} In the finite-sum setting, where $f(\mathbf{x}) = \frac{1}{N}\sum_{i=1}^N f_i(\mathbf{x})$, variance reduction techniques have been successfully integrated into the FW framework to improve convergence rates. Early works such as SVFW \cite{reddi2016stochastic, hazan2016variance} and SAGA-FW \cite{defazio2014saga} adapted SVRG and SAGA estimators to conditional gradient methods. More recent advances, including SPIDER-CG \cite{yurtsever19bSpiderCG}, OneSample-FW \cite{zhang2020oneOneSample}, and SARAH-FW \cite{beznosikov2024sarah}, utilize recursive gradient estimators to achieve an improved complexity bound of $\mathcal{O}(N+\sqrt{N}\epsilon^{-2})$ for nonconvex problems. While effective, these methods largely rely on non-adaptive constant step sizes or open-loop step size rules. In contrast, our method adapts to the local geometry of the optimization problem, extending adaptive, Lipschitz-free guarantees to this regime without sacrificing optimal complexity bounds of $\mathcal{O}(N+\sqrt{N}\epsilon^{-2})$.

\textbf{Expectation Minimization.} In the expectation setting, where $f(\mathbf{x}) = \mathbb{E}_{\xi \sim \mathcal{D} }f(\mathbf{x};\xi)$, \citet{reddi2016stochastic} introduced the Stochastic Frank-Wolfe (SFW) algorithm. To handle the inherent variance of stochastic gradients, subsequent works have incorporated momentum and recursive variance reduction. Notable examples include SFW-GB \cite{negiar2020stochastic}, which employs gradient boosting, and momentum-based approaches like OneSample-STORM \cite{zhang2020oneOneSample}, OneSample-EMA \cite{zhang2020oneOneSample}, and SRFW \cite{tang2022high}. While these methods improve robustness against noise, they typically require careful tuning of step sizes and momentum parameters based on global smoothness assumptions. Our approach differs by offering a adaptive Lipschitz-free update rule that adapts to the local geometry of the objective function, removing the need for global Lipschitz constants in the stochastic setting. The proposed ALFCG-MVR1 and ALFCG-MVR2 variants operate under average smoothness and individual smoothness assumptions, respectively. We establish iteration complexity bounds of $\tilde{\mathcal{O}}(\sigma^2 \epsilon^{-4} + \epsilon^{-2})$ for ALFCG-MVR1 and $\tilde{\mathcal{O}}(\sigma \epsilon^{-3} + \epsilon^{-2})$ for ALFCG-MVR2. These results are superior to existing methods, which typically yield $\mathcal{O}(\epsilon^{-4})$ and $\mathcal{O}(\epsilon^{-3})$ rates, as our bounds explicitly decouple the noise variance from the convergence rate, allowing for faster convergence in low-noise regimes as $\sigma \to 0$.

\textbf{Adaptive Lipschitz-Free Methods}. Step size selection remains a critical bottleneck in optimization. While adaptive algorithms like Adam \cite{KingmaB14} and AdaGrad \cite{duchi2011adaptive,ward2020adagrad} effectively scale updates based on gradient history, their applicability is often impeded in constrained settings. Recent advances, such as the ``auto-conditioned'' projected gradient method in \cite{lan2024projected}, achieve optimal complexity without line searches or known Lipschitz constants. Similarly, \cite{MalitskyM20} leverages gradient and iterate differences to estimate local curvature. However, a critical gap remains: the analysis in \cite{KingmaB14,duchi2011adaptive} focuses on unconstrained minimization, while the frameworks in \cite{lan2024projected,MalitskyM20} are strictly confined to convex optimization. To address this, we propose a self-normalized accumulator of past iterate differences tailored for projection-free nonconvex constrained optimization. This mechanism enables a robust, data-driven estimation of local smoothness, facilitating the precise construction of the quadratic surrogate model at each iteration.

\textbf{Other Related Settings}. Conditional gradient methods have demonstrated remarkable versatility, extending to diverse domains such as bilevel optimization \cite{jiang2023conditional}, Riemannian manifolds \cite{weber2023riemannian}, and multi-objective problems \cite{assunccao2025generalized}. Recent works have further broadened the scope to include zeroth-order optimization \cite{huang2020accelerated, akhtar2022zeroth}, as well as fully composite \cite{vladarean2023linearization} and nonsmooth composite minimization \cite{yurtsever2018conditional,locatello2019stochastic,woodstock2025splitting}. Additionally, specialized variants have been developed for linearly constrained nonconvex problems \cite{silveti2020generalized} and saddle point escaping \cite{gidel17a}.

In summary, while existing methods have made significant strides in variance reduction or adaptivity, our ALFCG framework uniquely bridges these two directions. It provides the first unified, projection-free algorithm that is simultaneously adaptive to local geometry (Lipschitz-Free) and independent of evaluating $f(\cdot)$ ($f$-Value-Free) across deterministic, finite-sum, and expectation settings. A detailed comparison of conditional gradient methods for solving composite nonconvex problems is provided in Table \ref{tab:comparision} (Appendix \ref{app:sect:existing}).

\section{Assumptions and Preliminaries}\label{sect:prel}

\noindent$\bullet$~\textbf{Assumptions}. We adopt the following standard assumptions regarding problem structure, function regularity, and oracle capabilities.
\begin{assumption}[\textbf{Feasible Set and Solution in Both Settings}] \label{ass:set_and_solution}
The constraint set $\mathcal{X} \subset \mathbb{R}^n$ is compact and convex with a finite diameter $D$, defined as $D := \sup_{\mathbf{x}, \mathbf{y} \in \mathcal{X}} \|\mathbf{x} - \mathbf{y}\| < \infty$. Furthermore, the feasible set $\operatorname{dom}(F) \cap \mathcal{X}$ is non-empty, and the global minimum $F_* := \inf_{\mathbf{x} \in \mathcal{X}} F(\mathbf{x})$ is finite and attained at some optimal solution $\mathbf{x}_* \in \mathcal{X}$.
\end{assumption}

\begin{assumption}[\textbf{Smoothness in Finite-Sum Setting}] \label{ass:regularity:finite:sum}
We assume that each component function $f_j(\cdot)$ is $L$-smooth for all $j \in [N]$. That is, $\|\nabla f_j(\mathbf{x}) - \nabla f_j(\mathbf{y})\| \leq L \|\mathbf{x} - \mathbf{y}\|$, $\forall \mathbf{x}, \mathbf{y} \in \mathcal{X}$.
\end{assumption}

\begin{assumption}[\textbf{Function Regularity in Expectation Setting}] \label{ass:regularity:expectation}
The objective function satisfies the following conditions:

\begin{enumerate}[label=\textbf{(\alph*)}, leftmargin=20pt, itemsep=1pt, topsep=1pt, parsep=0pt, partopsep=0pt]

\item \textbf{Unbiasedness and Bounded Variance:} The stochastic gradient is an unbiased estimator of the full gradient, i.e., $\mathbb{E}_{\xi}[\nabla f(\mathbf{x}; \xi)] = \nabla f(\mathbf{x})$. Furthermore, its variance is bounded by $\sigma^2$: $\mathbb{E}_{\xi}\left[\|\nabla f(\mathbf{x}; \xi) - \nabla f(\mathbf{x})\|^2\right] \leq \sigma^2 < \infty$.

    \item \textbf{Bounded (Sub-)gradients:} There exists a constant $G > 0$ such that $\|\nabla f(\mathbf{x})\| \leq G$ and $\|\boldsymbol{\zeta}\| \leq G$ for all $\mathbf{x} \in \mathcal{X}$ and all subgradients $\boldsymbol{\zeta} \in \partial h(\mathbf{x})$.

    \item \textbf{Smoothness Structure:} The objective function satisfies one of the following smoothness conditions:

    \begin{itemize}[label=\textbf{(\alph*)}, leftmargin=20pt, itemsep=1pt, topsep=1pt, parsep=0pt, partopsep=0pt]
    \item[1.] \textbf{Average Smoothness:} The expected function $f(\mathbf{x}) = \mathbb{E}[f(\mathbf{x}; \xi)]$ is $L$-smooth.
    \item[2.] \textbf{Individual Smoothness:} The stochastic function $f(\cdot; \xi)$ is $L$-smooth almost surely.
    \end{itemize}

\end{enumerate}
\end{assumption}

\noindent$\bullet$~\textbf{Linear Minimization Oracle (LMO)}. We assume access to an oracle that, for any $\mathbf{d} \in \mathbb{R}^n$ (typically a gradient estimate), returns a solution to the subproblem: $\text{LMO}(\mathbf{d}) \in \arg\min_{\mathbf{v} \in \mathcal{X}} \langle \mathbf{d}, \mathbf{v} \rangle + h(\mathbf{v})$.

\noindent$\bullet$~\textbf{Stationarity Measure}. Since $F(\mathbf{x})$ is possibly nonconvex, we utilize the generalized Frank-Wolfe gap:

\begin{definition}[\textbf{FW Gap}]
For any $\mathbf{x} \in \mathcal{X}$, the generalized Frank-Wolfe gap is defined as: $\mathcal{G}(\mathbf{x}) := \max_{\mathbf{v} \in \mathcal{X}} \left\{ \langle \nabla f(\mathbf{x}), \mathbf{x} - \mathbf{v} \rangle + h(\mathbf{x}) - h(\mathbf{v}) \right\}$.
\end{definition}

\noi The quantity $\mathcal{G}(\mathbf{x})$ is computable using the LMO. Specifically, if $\mathbf{v}_{\mathbf{x}} = \text{LMO}(\nabla f(\mathbf{x}))$, then $\mathcal{G}(\mathbf{x}) = \langle \nabla f(\mathbf{x}), \mathbf{x} - \mathbf{v}_{\mathbf{x}} \rangle + h(\mathbf{x}) - h(\mathbf{v}_{\mathbf{x}})$.
We note two key properties of this gap: (i) \textit{Stationarity}: $\mathcal{G}(\mathbf{x}) \geq 0$ for all $\mathbf{x} \in \mathcal{X}$, and $\mathcal{G}(\mathbf{x}) = 0$ if and only if $\mathbf{x}$ is a first-order stationary point of Problem \eqref{eq:main}. (ii) \textit{Convex Bound}: If $f$ is convex, the gap upper bounds the suboptimality: $F(\mathbf{x}) - F(\mathbf{x}_*) \leq \mathcal{G}(\mathbf{x})$.

\begin{definition}[\textbf{$\epsilon$-Approximate Solution}]
A point $\mathbf{x} \in \mathcal{X}$ is called an $\epsilon$-approximate stationary point if $\mathcal{G}(\mathbf{x})\leq \epsilon$.
\end{definition}

\section{The Proposed Algorithm}
\label{sect:proposed}

In this section, we introduce the Adaptive Lipschitz-Free Conditional Gradient (ALFCG) algorithm for Problem (\ref{eq:main}). We propose three variants: ALFCG-FS for finite-sum settings, and ALFCG-MVR1/MVR2 for expectation settings. The latter utilize Momentum-based Variance Reduction (MVR) with single- and two-batch updates, respectively. Notably, when $N=1$, ALFCG-FS reduces to a special case of the deterministic setting, which we denote as ALFCG-D.

Our core innovation is a recursive accumulation strategy that estimates the smoothness parameter via past iterate differences: $L_{t} = \rho \big(1 + \sum_{i=0}^{t-1} L_{i}^2 \|\mathbf{x}_{i+1}-\mathbf{x}_{i}\|^2 \big)^{1/2}$, where $\rho>0$ is a scaling constant, and $L_i$ denotes the estimated Lipschitz constant at step $i$. Unlike methods requiring global smoothness, ALFCG adaptively updates $L_t$ based on the trajectory's geometry without prior knowledge of the global constant.

\subsection{Gradient Estimation Mechanisms}

ALFCG employs distinct estimators to balance efficiency and variance based on the problem setting.

\textbf{Finite-Sum Setting (ALFCG-FS)}. For finite-sum problems, we track the full gradient using the SPIDER estimator \cite{fang2018spider}. Given frequency $q$ and mini-batch $\xi_t$, $\g_t$ updates as:
\beq \label{eq:g:spider}
\g_t = \begin{cases}\nabla f(\x_t),&\text{if}\mod(t,q)=0,\\
\g_{t-1} + \nabla f(\x_t;\xi_t) - \nabla f(\x_{t-1};\xi_t), & \text{otherwise},\end{cases}
\eeq
where $\nabla f(\x;\xi_t)$ denotes the stochastic gradient computed over the mini-batch $\xi_t$ with $\nabla f(\mathbf{x}; \xi_t) = \frac{1}{|\xi_t|} \sum_{i \in \xi_t} \nabla f_i(\mathbf{x})$. This recursive update reduces the variance of the estimator while maintaining a low computational cost per iteration.

\textbf{Expectation Setting (ALFCG-MVR1/MVR2)}. For the general expectation setting where the full gradient is inaccessible, we apply momentum-based variance reduction (MVR) via two variants:

\begin{enumerate}[label=\textbf{(\alph*)}, leftmargin=20pt, itemsep=1pt, topsep=1pt, parsep=0pt, partopsep=0pt]

\item  \textbf{MVR1, Single-Batch Momentum}: This variant employs an Exponential Moving Average (EMA) update with an adaptive decay rate $\alpha_t$ to estimate the stochastic gradient \cite{KingmaB14}:
\begin{equation} \label{eq:g:MVR1}
\g_t = (1-\alpha_t) \g_{t-1} + \alpha_t \nabla f(\x_t;\xi_t),
\end{equation}
\noi where $\alpha_t = \big( 1 + \sum_{i=0}^{t-1} (\beta + L_{i}^2 \|\x_{i+1}-\x_{i}\|^2) \big)^{-1/2}$, where $\beta\geq 0$ is a user-defined parameter parameter.

\item \textbf{MVR2, Two-Batch Momentum}: To further suppress noise, this variant incorporates a recursive correction term, similar to STORM updates \cite{cutkosky2019momentum,levy2021stormplus}:
\begin{equation} \label{eq:g:MVR2}
\g_t = (1-\alpha_t) (\g_{t-1} - \nabla f(\x_{t-1};\xi_t)) + \nabla f(\x_t;\xi_t),
\end{equation}
\noi where $\alpha_t = \min_{j=0}^t \hat{\alpha}_j$, and $\hat{\alpha}_j:=\big(\tfrac{ 1 + \max_{i=0}^{j-1} (\beta +  L_{i}^2 \|\x_{i+1}-\x_{i}\|^2 )  } { 1 + \sum_{i=0}^{j-1} (\beta +  L_{i}^2 \|\x_{i+1}-\x_{i}\|^2 ) }\big)^{2/3}$ with a user-defined parameter $\beta \geq 0$. A key advantage of this update rule is that it guarantees $0 \leq \tfrac{1}{\alpha_{t+1}} - \tfrac{1}{\alpha_t} \leq \tfrac{2}{3}$ for all $t \geq 0$ (see Lemma \ref{lemma:alpha:2:3}), which facilitates our analysis.

\end{enumerate}

\begin{algorithm}[!t]
\caption{ALFCG: {\bf Proposed Lipschitz-Free Conditional Gradient Method for Problem (\ref{eq:main}).}}\label{alg:main}

\KwIn{$x_0 \in \mathbb{R}^n$. frequency $q>0$, mini-batch size $b>0$. $\rho > 0$, $\beta \ge 0$. \\ (Default parameters: $b=q=\sqrt{N}$, $\rho=10^{-5}, \beta=100$) }

Set $L_0=\rho$, $\x_{-1}=\x_0$ and $\g_{-1} = \zero$.

\For{$t\leftarrow 0$ \KwTo $T$}
{
\textbf{S1)} Sample a mini-batch $\xi_t$ of size $b$, and compute $L_t$ and $\g_t$ under different settings:

$\bullet$ Option ALFCG-FS (Finite-Sum Setting): Compute $L_t$ using the following update rule, and update $\g_t$ via \eqref{eq:g:spider}:
\beq
\ts L_{t} = \rho \big(1 + \sum_{i=0}^{t-1}  L_{i}^2 \|\x_{i+1}-\x_{i}\|^2 \big)^{1/2}.\nn
\eeq

$\bullet$ Option ALFCG-MVR1 (Expectation Setting under Average Smoothness Assumption): Compute $\alpha_t$ and $L_t$ using the following update rules, and obtain $\g_t$ via (\ref{eq:g:MVR1}):
\beq
 \ts \alpha_t = \big( 1 + \sum_{i=0}^{t-1} (\beta + L_{i}^2 \|\x_{i+1}-\x_{i}\|^2) \big)^{-1/2}, ~ L_{t} = \rho \big( 1 + \sum_{i=0}^{t-1} L_{i}^2 \|\x_{i+1}-\x_{i}\|^2 \big)^{1/2}\alpha_{t}^{-1/2}.\nn
\eeq

$\bullet$ Option ALFCG-MVR2 (Expectation Setting under Individual Smoothness Assumption): Compute $\alpha_t$ and $L_t$ using the following update rules, and obtain $\g_t$ via (\ref{eq:g:MVR2}):
\beq
 \ts \alpha_t = \min_{j=0}^t \big(\tfrac{ 1 + \max_{i=0}^{j-1} (\beta +  L_{i}^2 \|\x_{i+1}-\x_{i}\|^2 )  } { 1 + \sum_{i=0}^{j-1} (\beta +  L_{i}^2 \|\x_{i+1}-\x_{i}\|^2 ) }\big)^{2/3},~\ts L_{t} = \rho \big( 1 + \sum_{i=0}^{t-1} L_{i}^2 \|\x_{i+1}-\x_{i}\|^2 \big)^{1/2}\alpha_{t}^{-1/4}.\nn
\eeq

\textbf{S2)} $\v_t = \text{LMO}(\g_t) := \arg \min_{\v \in \mathcal{X}}\la \v, \g_t \ra + h(\v) $.

 \textbf{S3)} $\x_{t+1} = \x_t + \bar{\eta}_t\cdot (\v_t-\x_t)$, where $\bar{\eta}_t =  \min(\tfrac{ h(\x_t) - h(\v_t) -\la \g_t,\v_t-\x_t\ra  }{L_t\|\v_t-\x_t\|_2^2},1)$.

}
\end{algorithm}

\subsection{Adaptive Lipschitz Estimation}

A distinguishing feature of ALFCG is the adaptive curvature estimate $L_t$, which replaces static constants with trajectory-based updates. Specifically, in the finite-sum setting, we compute $L_{t} = \rho \big(1 + \sum_{i=0}^{t-1}  L_{i}^2 \|\x_{i+1}-\x_{i}\|^2 \big)^{1/2}$ with $\rho > 0$, allowing $L_t$ to adapt to local geometry while ensuring convergence. In the stochastic setting, we scale this estimate by $\alpha_t^{-p}$ (where $p\in(0,1)$) to achieve optimal iteration complexity. At each iteration $t$, the algorithm first solves the subproblem $\v_t \in \arg \min_{\v \in \mathcal{X}} \langle \g_t, \v - \x_t \rangle + h(\v)$. Analogous to the global Lipschitz-based FW-ShortStep, the step size $\bar{\eta}_t \in [0,1]$ is determined by minimizing a quadratic upper bound constructed with $L_t$, yielding the closed-form solution: $\bar{\eta}_t = \min\big( \frac{ h(\mathbf{x}_t) - h(\mathbf{v}_t) - \langle \nabla f(\x_t), \mathbf{v}_t - \mathbf{x}_t \rangle }{ L_t \|\mathbf{v}_t - \mathbf{x}_t\|_2^2 }, 1 \big)$. Note that the numerator is guaranteed to be non-negative by the optimality of $\mathbf{v}_t$.

The complete procedure is detailed in Algorithm \ref{alg:main}.

\section{Iteration Complexity}
\label{sect:iter}

This section presents the iteration complexity analysis of the proposed ALFCG algorithms.

\noi \textbf{Notation}. For any solution $\x_t$, we define the suboptimality gap as $F^+_t := F(\x_t) - F_*$, and the stationarity measure as $\mathcal{G}(\x_t) := \left\{ \langle \nabla f(\mathbf{x}_t), \mathbf{x}_t - \tilde{\v}_t \rangle + h(\mathbf{x}_t) - h(\tilde{\v}_t) \right\}$, where $\tilde{\v}_t \in \text{LMO}(\nabla f(\mathbf{x}_t))$. Additionally, we denote the gradient error by $s_t := \|\g_t - \nabla f(\x_t)\|$, and the scaled step difference as $\delta_t= L_t\|\x_{t+1}-\x_t\|$. Finally, $T \geq 0$ denotes the total number of iterations.

\subsection{Initial Theoretical Analysis} \label{sect:init}

Leveraging the $L$-smoothness of $f(\cdot)$, the convexity of $h(\cdot)$, and the optimality conditions of $\v_t$ and $\bar{\eta}_t$, we derive the following critical descent lemma. This result applies to both the finite-sum and expectation settings.

\label{sect:iter:Init}

\begin{lemma}\label{lemma:descent:ncvx} (Fundamental Descent Inequality, Proof in Appendix \ref{app:lemma:descent:ncvx}) For all $\gamma_t\in (0,1]$, we have:
\beq \label{eq:analysis:first}
\EEE[\gamma_t\GG(\x_t)   + F(\x_{t+1}) - F(\x_{t})  + \tfrac{\delta_t^2}{4 L_t}  ] ~\leq~ \EEE[\tfrac{\delta_t^2L}{2 L_t^2}  + \tfrac{2 s_t^2}{L_t} + \gamma_t^2 L_t D^2].
\eeq

\end{lemma}

\noi Based on Lemma \ref{lemma:descent:ncvx}, the parameter $\gamma_t$ is specified differently depending on the setting: (i)  ALFCG-FS:
$\gamma_t = \tfrac{L_0}{L_t \sqrt{T+1}}$; (ii) ALFCG-MVR1: $\gamma_t = \tfrac{L_0}{2L_t} \big(\tfrac{\beta^{1/4}}{1+\beta^{1/4}} (T+1)^{-1/4} + (T+1)^{-1/2} \big)$; (ii) ALFCG-MVR2: $\gamma_t = \tfrac{L_0}{2 L_t} \big(\tfrac{\beta^{1/6}}{1+\beta^{1/6}} (T+1)^{-1/3} +(T+1)^{-1/2} \big)$. Note that $\gamma_t\in (0,1]$ holds by construction.

\subsection{Analysis for ALFCG-FS}\label{sect:iter:FS}

This subsection analyzes the convergence of ALFCG-FS. Our proof hinges on the stability of the adaptive smoothness estimator $L_t$. We first bound the growth of $L_t$ by showing that the ratio of consecutive estimates is uniformly bounded.

 \begin{lemma} \label{lemma:kappa}
(Proof in Appendix \ref{app:lemma:kappa}) For all $t\geq 0$, it holds that $\tfrac{L_{t+1}}{L_{t}} \leq \kappa$, where $\kappa = 1 + \rho D$.
\end{lemma}

To build intuition and establish the fundamental properties of our adaptive mechanism, we begin by examining the \textbf{reduced deterministic setting} (i.e., the noise-free case where $s_t = 0$). This serves as a stepping stone for the full stochastic analysis presented later.

\subsubsection{Deterministic Setting (Base Case)}

We begin with a high-level overview of the proof strategy for ALFCG-D. First, we derive the fundamental descent inequality in (\ref{eq:analysis:first}) and the growth bound in Lemma \ref{lemma:kappa}. Second, by multiplying both sides of (\ref{eq:analysis:first}) by $L_t$, a telescoping sum establishes a bound on $\sum_{t=0}^T \delta_t^2$, which in turn provides an upper bound for $L_{t+1}$ as $\mathcal{O}(1) + \mathcal{O}(\max_{i=0}^t F_i^+)$. Third, a direct telescoping of Inequality (\ref{eq:analysis:first}) yields the uniform bound $F_t^+ \leq \overline{F}$. Finally, after establishing the boundedness of $\{F_t, L_t\}$, we derive the convergence rate for the aggregate weighted Frank-Wolfe gap $\sum_{t=0}^T \gamma_t \mathcal{G}(\mathbf{x}_t)$.

To formalize the strategy outlined in the roadmap, we now present Lemma \ref{lemma:bound:F:ncvx}.

\begin{lemma} \label{lemma:bound:F:ncvx}
(Proof in Appendix \ref{app:lemma:bound:F:ncvx}) We have the following results for all $t\geq 0$.

\begin{enumerate}[label=\textbf{(\alph*)}, leftmargin=20pt, itemsep=1pt, topsep=1pt, parsep=0pt, partopsep=0pt]

\item $L_{t+1} \leq L'  +   4\rho^2 \max_{i=0}^t F^+_{i}$, where $L': = \rho + 4 L \kappa +  4 \rho^3 D^2$ is a constant.

\item $F^+_t\leq \overline{F}$, where $\overline{F} = 2 (F^+_{0} + D^2\rho + \omega |\ln(\rho)|  + \omega |\ln(\lambda)|  + \omega \lambda L')$, $\lambda = \tfrac{1}{8 \omega \rho^2}$, and $\omega=\tfrac{ L \kappa^2}{\rho^2}$.

\item $L_t\leq \overline{L}$, where $\overline{L}:=  L'  + 4\rho^2 \overline{F}$.

\item $\sum_{t=0}^T \gamma_t\GG(\x_t) \leq \overline{F}$, where $\gamma_t = \tfrac{L_0}{L_t \sqrt{T+1}}$.

\end{enumerate}

\end{lemma}

Finally, we establish the iteration complexity for ALFCG-D in the following theorem.

\begin{theorem}
\label{theorem:iter:complexity:ncvx}
(Proof in Appendix \ref{app:theorem:iter:complexity:ncvx}) The average Frank-Wolfe gap satisfies: $\tfrac{1}{T+1}\sum_{t=0}^T \GG(\x_t) \leq \tfrac{C}{\sqrt{T+1}}$, where $C:=\tfrac{\overline{F} \cdot \overline{L}}{\rho}$ is a constant. Consequently, for any $\epsilon > 0$, the algorithm finds an $\epsilon$-approximate solution (i.e., $\min_{0 \le t \le T}\mathcal{G}(\mathbf{x}_t) \leq \epsilon$) within at most $T = \lceil C^2 \epsilon^{-2} - 1\rceil$ iterations.

\end{theorem}

\begin{remark}
To the best of our knowledge, ALFCG-D is the first adaptive projection-free algorithm for composite nonconvex minimization that requires neither global smoothness constants nor line search (which typically necessitates a zeroth-order oracle of $f(\cdot)$).
\end{remark}

 \subsubsection{General Finite-Sum Setting (with SPIDER)}

Building upon the deterministic analysis, we extend our results to the finite-sum setting by integrating the SPIDER estimator to control stochastic variance.

We fix the epoch length $q \geq 1$ and denote the current epoch index as $r_{t}:= \lfloor\frac{t}{q}\rfloor+1$, implying $t \in [(r_{t} - 1)q, r_{t} q - 1]$. We begin by characterizing the local stability of the estimator. The following lemma bounds the gradient estimation error by the iterate movement, a key mechanism for variance reduction.

\begin{lemma} \label{eq:spider:estimator}
(Adapted from \rm{Lemma 1 in} \cite{fang2018spider}) Let $\mathbf{g}_t$ be the SPIDER estimator. For all $t$ such that $(r_t-1)q \leq t \leq r_t q - 1$, the estimation error satisfies: $\EEE[\|\g_t- \nabla f(\x_t)\|_2^2] - \|\g_{t-1} - \nabla f(\x_{t-1})\|_2^2
\leq \tfrac{L^2}{b} \EEE [  \|\x_t - \x_{t-1}\|_2^2 ]$.
\end{lemma}

Telescoping this recursive bound over a single epoch yields a cumulative bound on the gradient estimation error, ensuring it remains controlled by the trajectory's movement.

\begin{lemma} \label{lemma:f:nablaf}
(Proof in Appendix \ref{app:lemma:f:nablaf}) For all $t$ satisfying $(r_t-1)q \leq t\leq r_t q-1$, we have: $\EEE[\|  \g_t - \nabla f(\x_t)    \|_2^2] \leq \tfrac{L^2}{b} \sum_{i= (r_{t}-1)q }^{t-1}\EEE [  \|\x_{i+1} - \x_{i}\|_2^2 ]$.
\end{lemma}

Analyzing variance reduction often involves complex nested summations. To bound these error terms alongside our adaptive estimates $L_t$, we provide the following technical lemma to reduce double summations into single sums.

\begin{lemma} \label{lemma:sum:y:two}
(Proof in Appendix \ref{app:lemma:sum:y:two})
Let $\{L_t,M_t\}_{t=0}^{\infty}$ be two nonnegative sequences. Assume $1\leq \tfrac{L_{t+1}}{L_t}\leq \kappa$, where $\kappa\geq 1$. For all $t$ with $(r_t-1)q \leq t\leq r_t q-1$ and $T\geq 0$, we have:

\begin{enumerate}[label=\textbf{(\alph*)}, leftmargin=20pt, itemsep=2pt, topsep=1pt, parsep=0pt, partopsep=0pt]

\item $\sum_{j=(r_{t}-1)q}^{t}  [\sum_{i=(r_{j}-1)q}^{j-1} M_i] \leq (q-1) \sum_{i=(r_{t}-1)q}^{t-1} M_i$.

\item $\sum_{j={(r_t-1)q}}^t  \sum_{i=(r_{j}-1)q}^{j-1}  \tfrac{M_i}{L_j}  \leq q \kappa^{q-1} \sum_{i=(r_{t}-1)q}^{t-1} \tfrac{M_i}{L_i}$.

\item $\sum_{t=0}^{T}  [\sum_{i=(r_{t}-1)q}^{t-1} M_i] \leq (q-1) \sum_{t=1}^{T} M_t$.

\end{enumerate}

\end{lemma}

Equipped with these auxiliary bounds, we establish the uniform boundedness of the adaptive parameters for ALFCG-FS. The following lemma is the stochastic counterpart to Lemma \ref{lemma:bound:F:ncvx}, ensuring that the step size estimates and surrogate objective values remain stable despite stochastic noise.

\begin{lemma} \label{lemma:bound:F:ncvx:VR}
(Proof in Appendix \ref{app:lemma:bound:F:ncvx:VR}) For all $t\geq 0$, we have the following results.

\begin{enumerate}[label=\textbf{(\alph*)}, leftmargin=20pt, itemsep=1pt, topsep=1pt, parsep=0pt, partopsep=0pt]

\item It holds $\EEE[L_{t}]\leq L' + 4\rho^2 \max_{i=0}^{t-1} F^+_i$, where $L' :=  4 \rho^3 D^2  +   4 L\kappa + 16 L^2 \tfrac{q \kappa^q}{b \rho} + \rho$.

\item $\EEE[F^+_t]\leq \overline{F}$, where $\overline{F} :=2 (F^+_0 + \rho D^2  +  \omega |\ln(\rho)| + \omega |\ln(\lambda)| +  \omega\lambda L')$, $\omega :=\tfrac{L \kappa^2}{\rho^2} + \tfrac{4 L^2}{b \rho^3} q \kappa^{q+1}$.

\item $\EEE[L_t]\leq \overline{L}$, where $\overline{L}:= L' + 4\rho^2 \overline{F}$.

\item $\EEE[\sum_{t=0}^{T-1} \gamma_t\GG(\x_t)] \leq \overline{F}$, where $\gamma_t = \tfrac{L_0}{L_t \sqrt{T+1}}$.

\end{enumerate}

\end{lemma}

Finally, by combining these uniform bounds with the variance-reduction properties of SPIDER, we characterize the total iteration complexity for ALFCG-FS.

\begin{theorem}
\label{theorem:VR}
(Proof in Appendix \ref{app:theorem:VR}) We have the following results.

\begin{enumerate}[label=\textbf{(\alph*)}, leftmargin=20pt, itemsep=1pt, topsep=1pt, parsep=0pt, partopsep=0pt]

\item The expected average Frank-Wolfe gap satisfies $\EEE[\tfrac{1}{T+1}\sum_{t=0}^T \GG(\x_t)] \leq  \tfrac{C}{\sqrt{T+1}}$, where $C:= \tfrac{\overline{F} \cdot \overline{L} }{\rho}$. Consequently, for any $\epsilon > 0$, the algorithm finds an $\epsilon$-approximate solution in expectation (i.e., $\mathbb{E}[\min_{0 \le t \le T} \mathcal{G}(\mathbf{x}_t)] \leq \epsilon$) within at most $T = \lceil C^2 \epsilon^{-2} -1 \rceil$ iterations.

\item Assume $b = q =\sqrt{N}$. The total iteration complexity in expectation required to find an $\epsilon$-approximate critical point is given by $\OO(N + \sqrt{N}\epsilon^{-2})$.

\end{enumerate}
\end{theorem}

\begin{remark}
Moving beyond fixed or diminishing step sizes, ALFCG adapts to local geometry to provide the first adaptive Lipschitz-free method in the finite-sum setting without sacrificing the optimal $\mathcal{O}(N+\sqrt{N}\epsilon^{-2})$ complexity.
\end{remark}

\subsection{Analysis for ALFCG-MVR1}
\label{sect:iter:MVR1}

This subsection analyzes the single-batch variant, ALFCG-MVR1, under the average smoothness assumption (Assumption \ref{ass:regularity:expectation}(c-1)).

We provide a high-level overview of the proof strategy for Lemma \ref{lemma:bound:MVR1}. First, we establish a uniform bound $F(\mathbf{x}) \leq \overline{F}$ leveraging the Lipschitz continuity of $F(\cdot)$ and the finite diameter of $\mathcal{X}$. Second, utilizing the exponential moving average update in (\ref{eq:g:MVR1}) and the bounded variance assumption, we derive a uniform bound for $\mathbb{E}[\|\mathbf{g}_t\|_2^2]$. Third, incorporating the fundamental descent inequality in (\ref{eq:analysis:first}), we bound the expected displacement $\mathbb{E}[\delta_t^2] \leq \bar{\delta}^2$, which subsequently controls the growth ratio $L_{t+1}/L_t$. We now formalize this roadmap in Lemma \ref{lemma:bound:MVR1}.

\begin{lemma} \label{lemma:bound:MVR1}(Proof in Appendix \ref{app:lemma:bound:MVR1}) For all $t\geq 0$, the following properties hold:

\begin{enumerate}[label=\textbf{(\alph*)}, leftmargin=20pt, itemsep=1pt, topsep=1pt, parsep=0pt, partopsep=0pt]

\item $F(\x)\leq\overline{F}$, where $\overline{F} = F(\x_*) + 2 G D$.

\item $\EEE[\|\s_t\|_2^2] \leq \smash{\bar{s}}^2$, where $\overline{s}:= 3 (\sigma+G)$.

\item $\EEE[\delta_t^2] \leq \smash{\bar{\delta}}^2$, where $\overline{\delta}:=  4 L D + 26 G  + 8 \sigma + 3 \rho D$.

\item $\tfrac{L_{t+1}}{L_{t}} \leq \kappa$, where $\kappa = 1 + \overline{\delta}$.

\end{enumerate}

\end{lemma}

Next, leveraging the EMA update rule in (\ref{eq:g:MVR1}) and the average smoothness assumption, we bound the cumulative estimation error $\mathbb{E}[\sum_{t=0}^T \|\mathbf{s}_t\|_2^2]$ via the following lemma.

\begin{lemma}
\label{lemma:MVR1:S}(Proof in Appendix \ref{app:lemma:MVR1:S}) For all $T\geq t\geq 0$, we have:

\begin{enumerate}[label=\textbf{(\alph*)}, leftmargin=20pt, itemsep=1pt, topsep=1pt, parsep=0pt, partopsep=0pt]

\item We have: $\EEE[s^2_{t}] \leq \EEE[(1-\alpha_{t}) s^2_{t-1} + B_{t}]$, where $B_t := \alpha_{t}^2 \sigma^2 + \tfrac{L^2}{\alpha_{t}}\|\x_{t}-\x_{t-1}\|_2^2$.

\item  We have: $\EEE[\sum_{t=0}^T s_t^2] \leq \big( \smash{\bar{s}}^2 + \overline{B} \big) \cdot \EEE[  \Gamma^{1/2} + \beta^{1/2} T^{1/2} ]$, where $\overline{B}:=\sigma^2 ( 1 + \frac{\ln(1 + \beta T)}{\beta} )  + \tfrac{L^2 \kappa^2}{\rho^2} \ln (1 + T \smash{\bar{\delta}}^2 )$, and $\Gamma:= 1 + \sum_{i=0}^T \delta_i^2$.

\end{enumerate}

\end{lemma}

The following lemma establishes joint bounds for $1 + \sum_{t=0}^T \mathbb{E}[\delta_t^2]$ and $\sum_{t=0}^T \mathbb{E}[s_t^2]$. Employing the case-based analysis strategy from STORM+ \cite{levy2021stormplus}, we consider two regimes: $1 + \sum_{t=0}^T \mathbb{E}[\delta_t^2] \geq 48 \sum_{t=0}^T \mathbb{E}[s_t^2]$ and its complement. In the first regime, we apply the fundamental descent inequality (\ref{eq:analysis:first}) to bound the cumulative step difference $\sum_{t=0}^T \delta_t^2$; in the second, we leverage Lemma \ref{lemma:MVR1:S}(b) to bound the cumulative error $\sum_{t=0}^T s_t^2$. This leads to the following lemma.

\begin{lemma} \label{lemma:lemma:MVR1:S:delta}(Proof in Appendix \ref{app:lemma:lemma:MVR1:S:delta}) Let $\gamma_t = \tfrac{L_0}{2L_t} \big(\tfrac{\beta^{1/4}}{1+\beta^{1/4}} (T+1)^{-1/4} + (T+1)^{-1/2} \big)$. We define $C_1:= 6 D^2 \rho^2 + 12^2 \rho^2 \overline{F}^2$, $Q_1 := 81 ( 1 + 4\rho \overline{F} + 2 D^2 \rho^2 + 4 L\kappa/\rho)^4$, $C_2 := 96 (  \smash{\bar{s}}^2 + \overline{B})$, and $Q_2 := 96^2 (  \smash{\bar{s}}^2 + \overline{B})^2$. We define $C_0:=\max(2C_1, C_2)$, and $Q_0:=\max(2Q_1, Q_2)$. For all $T\geq 0$, we have the following results.

\begin{enumerate}[label=\textbf{(\alph*)}, leftmargin=20pt, itemsep=1pt, topsep=1pt, parsep=0pt, partopsep=0pt]

\item $\ts 1+ \sum_{t=0}^T \EEE[\delta_t^2] \leq \Omega$ and $\sum_{t=0}^T s_t^2 \leq\Omega$, where $\Omega:= C_0 \max\left(\beta^{1/2},\vartheta^2\right) (T+1)^{1/2} + Q_0$.

\item $\ts \sum_{t=0}^T \EEE[L_t\gamma_t\GG(\x_t)] \leq \Omega$. In particular, if we set $\beta =\nu \sigma^2$ with $\nu=\OO(1)$, we have: $\ts \sum_{t=0}^T \EEE[L_t\gamma_t\GG(\x_t) ] \leq C_0 \cdot \nu^{1/2} \sigma \cdot (T+1)^{1/2} +Q_0$.

\end{enumerate}

\end{lemma}

Finally, we state the main complexity result for ALFCG-MVR1.
\begin{theorem}
\label{theorem:MVR1}
(Proof in Appendix \ref{app:theorem:MVR1}) If we set $\beta =\nu \sigma^2$ with $\nu=\OO(1)$, we have:
\beq
\ts \EEE[\tfrac{1}{T+1}\sum_{t=0}^T \GG(\x_t)] \leq \tilde{\mathcal{O}}\left( \nu^{1/2} (1+\sigma^{1/2})\sigma^{1/2} T^{-1/4} + T^{-1/2} \right).\nn
\eeq
\noi Consequently, for any $\epsilon > 0$, ALFCG-MVR1 finds an $\epsilon$-approximate solution in expectation (i.e., $\mathbb{E}[\min_{0 \le t \le T} \mathcal{G}(\mathbf{x}_t)] \leq \epsilon$) within at most $\tilde{\mathcal{O}}\big( \frac{\sigma^2 (1+\sigma^{1/2})^4}{\epsilon^4} + \frac{1}{\epsilon^2} \big)$ iterations.

\end{theorem}

\begin{remark} \ding{182} The logarithmic factors in Theorem \ref{theorem:MVR1} arise solely from the bound $\overline{B} \leq \tilde{\mathcal{O}}(1)$. \ding{183} The complexity established in Theorem \ref{theorem:MVR1} improves upon existing methods, which typically yield $\mathcal{O}(\epsilon^{-4})$, by explicitly decoupling the noise variance from the convergence rate. Consequently, in the noiseless limit ($\sigma \to 0$), our bound recovers the nearly optimal rate of $\tilde{\mathcal{O}}(\epsilon^{-2})$.
\end{remark}

\subsection{Analysis for ALFCG-MVR2}
\label{sect:iter:MVR2}

This subsection analyzes the two-batch variant, ALFCG-MVR2, under the individual smoothness assumption (Assumption \ref{ass:regularity:expectation}(c-2)).

\textbf{Variance Stability}. A key distinction of the MVR2 update is that, unlike MVR1, the boundedness of the estimator $\g_t$ is not inherently guaranteed by the boundedness of stochastic gradients. Consequently, we introduce a mild assumption on the variance stability.

\begin{assumption} \label{ass:bound:s}
The variance of the gradient estimator is uniformly bounded, i.e., there exists a constant $\overline{s} > 0$ such that $\mathbb{E}[\|\g_t - \nabla f(\x_t)\|_2^2] \leq \smash{\bar{s}}^2$ for all $t \geq 0$.
\end{assumption}

\begin{remark} \ding{182} Assumption \ref{ass:bound:s} is mild as Lemma \ref{lemma:lemma:MVR2:S:delta} confirms that the average variance vanishes asymptotically. Theoretical proof of this uniform bound in the adaptive Lipschitz-free setting remains an open challenge. \ding{183} Figure \ref{fig:7} in the Appendix plots the evolution of $s_t^2 := \|\g_t - \nabla f(\x_t)\|_2^2$ for ALFCG-MVR2 on both the nuclear norm and $\ell_p$-norm ball constraint problems. The results demonstrate that Assumption \ref{ass:bound:s} is empirically valid and that $\|\s_t\|_2^2$ generally decreases over iterations.
\end{remark}

To manage the specific parameter decay in MVR2, we establish the following property for the auxiliary sequence $\alpha_t$.

\begin{lemma} \label{lemma:alpha:2:3}
(Proof in Appendix \ref{app:lemma:alpha:2:3})
Let $\alpha_t := \min_{0\le k\le t}\hat\alpha_k$, where $\hat\alpha_t := \Big(\frac{1+\max_{i=0}^t u_i}{1+\sum_{i=0}^t u_i}\Big)^p$, $u_i:=\beta +  L_{i}^2 \|\x_{i+1}-\x_{i}\|^2\geq 0$, and $p = 2/3$. Then for all $t\geq 0$, the sequence satisfies $0\leq \tfrac{1}{\alpha_{t+1}} - \tfrac{1}{\alpha_t} \leq p$.
\end{lemma}

Similar to the MVR1 analysis in Lemma \ref{lemma:bound:MVR1}, we provide a high-level overview of the proof strategy for Lemma \ref{lemma:bound:MVR2}. First, we establish a uniform bound $F(\mathbf{x}) \leq \overline{F}$ leveraging the Lipschitz continuity of $F(\cdot)$ and the finite diameter of $\mathcal{X}$. Second, utilizing Assumption \ref{ass:bound:s} and the fundamental descent Inequality (\ref{eq:analysis:first}), we bound the expected displacement $\mathbb{E}[\delta_t^2] \leq \bar{\delta}^2$, which in turn controls the growth ratio $L_{t+1}/L_t$. We now formalize this roadmap in Lemma \ref{lemma:bound:MVR2}.

\begin{lemma} \label{lemma:bound:MVR2}
(Proof in Appendix \ref{app:lemma:bound:MVR2}) For all $t\geq 0$, we have the following results.

\begin{enumerate}[label=\textbf{(\alph*)}, leftmargin=20pt, itemsep=1pt, topsep=1pt, parsep=0pt, partopsep=0pt]

\item $F(\x)\leq\overline{F}$, where $\overline{F} = F(\x_*) + 2 G D$.

\item $\EEE[\delta_t^2]\leq \smash{\bar{\delta}}^2$, where $\overline{\delta}:= 8 \big( L D + C + \overline{s} + \rho D \big) $.

\item $\tfrac{L_{t+1}}{L_{t}} \leq \kappa$, where $\kappa = ( 2 + 2 \smash{\bar{\delta}}^2/\rho)^{1/2}$.

\end{enumerate}

\end{lemma}

Next, leveraging the STORM-like update rule in (\ref{eq:g:MVR2}) and the individual smoothness assumption, we derive the cumulative estimation error $\mathbb{E}[\sum_{t=0}^T \|\mathbf{s}_t\|_2^2]$ using the following lemma.

\begin{lemma} \label{lemma:MVR2:S}
(Proof in Appendix \ref{app:lemma:MVR2:S}) For all $t\geq0$, we obtain:

\begin{enumerate}[label=\textbf{(\alph*)}, leftmargin=20pt, itemsep=1pt, topsep=1pt, parsep=0pt, partopsep=0pt]

\item $\EEE[s^2_{t} ]\leq \EEE[ (1-\alpha_{t}) s^2_{t-1} + B_{t}]$, where $B_{t}:= 2 \alpha_{t}^2 \sigma^2  + 8 L^2 \|\x_{t}-\x_{t-1}\|_2^2$.

\item $\EEE[\sum_{t=0}^T s_t^2] \leq \EEE[ \dot{B} (\sigma^2 + \frac{\sigma^2}{\beta} + \sigma^2\beta^{-2/3} T^{1/3} )+ \ddot{B}  \left( 1 + \sum_{i=0}^{T} \delta_i^2 \right)^{1/3} + \ddot{B} T^{1/3} \beta^{1/3}]$, where $\dot{B}:= 18  (\beta + \smash{\bar{\delta}}^2+1)$, and $\ddot{B} := 40  L^2 \kappa^2 \rho^{-1} \ln(1 + \smash{\bar{\delta}}^2 T)$.
\end{enumerate}
\end{lemma}

Following the case-based analysis used in STORM+ \cite{levy2021stormplus}, we decouple the interaction between cumulative displacement $1 + \sum_{t=0}^T \mathbb{E}[\delta_t^2]$ and cumulative variance $\sum_{t=0}^T \mathbb{E}[s_t^2]$ by examining the regimes where $1 + \sum \mathbb{E}[\delta_t^2]$ dominates $48 \sum \mathbb{E}[s_t^2]$ and its complement. In the first case, we utilize the upper bound of $\sum_{t=0}^T \delta_t^2$ derived from the fundamental descent Inequality (\ref{eq:analysis:first}). In the second case, we leverage the cumulative variance bound $\sum_{t=0}^T s_t^2$ from Lemma \ref{lemma:MVR2:S}(b). These results are formally presented below.

\begin{lemma} \label{lemma:lemma:MVR2:S:delta} (Proof in Appendix \ref{app:lemma:lemma:MVR2:S:delta}) Let $\gamma_t = \tfrac{L_0}{2 L_t} \big(\tfrac{\beta^{1/6}}{1+\beta^{1/6}} (T+1)^{-1/3} +(T+1)^{-1/2} \big)$. We define $Y_1:=12 D^2 \rho^2  + 12^2  \overline{F}^2\rho^2$, $Z_1 := 27 ( 1 + 4  \overline{F}\rho + 4 D^2 \rho^2 + \tfrac{4L\kappa}{\rho} )^3$, $Y_2:= 96 ( \dot{B}  + \ddot{B} )$, and $Z_2:=(48 \ddot{B})^{3/2} + 96 \dot{B}  (\sigma^2+ \tfrac{\sigma^2}{\beta} )$. We let $Y_0:=\max(2Y_1, Y_2)$, and $Z_0:=\max(2Z_1, Z_2)$. For all $T\geq 0$, we have the following results.

\begin{enumerate}[label=\textbf{(\alph*)}, leftmargin=20pt, itemsep=1pt, topsep=1pt, parsep=0pt, partopsep=0pt]

\item $1 + \sum_{t=0}^T\EEE[ \delta_t^2]\leq \Omega$, $\sum_{t=0}^T \EEE[s_t^2]\leq\Omega$, where $\Omega:=Y_0 \max\left(\beta^{1/3},\sigma^2 \beta^{-2/3} \right) (T+1)^{1/3} +Z_0$.

\item $\sum_{t=0}^T\EEE[L_t \gamma_t \GG(\x_t)]\leq \Omega$. In particular, if we set $\beta =\nu \sigma^2$ with $\nu=\OO(1)$, we have: $\ts \sum_{t=0}^T \EEE[L_t\gamma_t \GG(\x_t)]  \leq  Y_0 \max( \nu^{1/3}, \nu^{-2/3} )  \sigma^{2/3} (T+1)^{1/3} +Z_0$.

\end{enumerate}

\end{lemma}

Finally, we establish the expected iteration complexity for ALFCG-MVR2.

\begin{theorem} \label{theorem:MVR2}
(Proof in Appendix \ref{app:theorem:MVR2}) If we set $\beta =\nu \sigma^2$ with $\nu=\OO(1)$, we have:
\beq
\ts \EEE[\tfrac{1}{T+1}\sum_{t=0}^T \GG(\x_t) ] \leq \tilde{\OO} \left( (\nu + \tfrac{1}{\nu}) (1+\sigma^{1/3}) \sigma^{1/3} T^{-1/3} + T^{-1/2} \right).\nn
\eeq
\noi Consequently, for any $\epsilon > 0$, ALFCG-MVR2 finds an $\epsilon$-approximate solution in expectation (i.e., $\mathbb{E}[\min_{0 \le t \le T} \mathcal{G}(\mathbf{x}_t)] \leq \epsilon$) within at most $\tilde{\mathcal{O}}\left( \frac{\sigma(1+\sigma^{1/3})^3}{\epsilon^3} + \frac{1}{\epsilon^2} \right)$ iterations.
\end{theorem}

\begin{remark} \ding{182} The logarithmic factors in Theorem \ref{theorem:MVR2} originate solely from the bound $\ddot{B}\leq \tilde{\mathcal{O}}(1)$. \ding{183} \textit{Comparison to STORM}. The complexity bound in Theorem \ref{theorem:MVR2} matches the $\mathcal{O}(T^{-1/2}+\sigma^{1/3} T^{-1/3})$ rate of STORM \cite{cutkosky2019momentum} up to logarithmic factors. Notably, while those methods are designed for unconstrained smooth stochastic optimization, our bound achieves this rate in a projection-free, constrained setting. \ding{184} \textit{Noise-Adaptivity}. Theorem \ref{theorem:MVR2} parallels the results of Theorem \ref{theorem:MVR1}. By setting $\beta \propto \sigma^2$, our bound becomes noise-adaptive, successfully recovering the nearly optimal iteration complexity of $\tilde{\mathcal{O}}(\epsilon^{-2})$ as the noise level vanishes ($\sigma \to 0$). \ding{185} \textit{Comparison of MVR1 vs. MVR2}. We observe that the bound for ALFCG-MVR1, $\tilde{\mathcal{O}}(\sigma^{1/2} T^{-1/4} + T^{-1/2})$, provides a tighter dependency on the noise parameter $\sigma$ compared to the $\tilde{\mathcal{O}}(\sigma^{1/3} T^{-1/3} + T^{-1/2})$ bound of ALFCG-MVR2 in the low-noise regime where $\sigma \to 0$.
\end{remark}

\section{Conclusions}
\label{sect:conc}

In this paper, we introduced ALFCG, an adaptive projection-free framework for stochastic nonconvex optimization. By leveraging a novel self-normalized smoothness accumulator, ALFCG effectively estimates local geometry without requiring global Lipschitz constants or expensive line searches. We instantiated this framework with three variants: ALFCG-FS for finite-sum problems, and ALFCG-MVR1/ALFCG-MVR2 for stochastic expectation problems. Theoretically, we proved that ALFCG-FS achieves a complexity of $\mathcal{O}(N + \sqrt{N}\epsilon^{-2})$, while ALFCG-MVR1 and ALFCG-MVR2 achieve $\tilde{\mathcal{O}}(\sigma^2\epsilon^{-4} + \epsilon^{-2})$ and $\tilde{\mathcal{O}}(\sigma\epsilon^{-3} + \epsilon^{-2})$, respectively. A distinguishing feature of our analysis is the unified convergence bound, which seamlessly interpolates between stochastic and deterministic regimes, recovering the nearly optimal $\tilde{\mathcal{O}}(\epsilon^{-2})$ rate as the noise level vanishes ($\sigma \to 0$). Empirically, extensive experiments on multiclass classification constrained by nuclear norm and $\ell_p$ balls demonstrate that ALFCG generally outperforms state-of-the-art conditional gradient baselines in terms of computational efficiency.

\clearpage
\normalem

\bibliographystyle{plain}
\bibliography{mybib}

\clearpage
\normalem

\noi \appendix
\noi {\huge Appendix}

\noi The appendix is structured as follows:

\noi Appendix~\ref{app:sect:notation} defines the notation used throughout the paper.

\noi Appendix~\ref{app:sect:app} describes the motivating applications.

\noi Appendix~\ref{app:sect:existing} reviews related conditional gradient methods.

\noi Appendix~\ref{app:sect:support} collects the necessary supporting lemmas.

\noi  Appendices~\ref{app:sect:iter:Init}, \ref{app:sect:iter:FS}, \ref{app:sect:iter:MVR1}, and \ref{app:sect:iter:MVR2} provide the proofs for Sections~\ref{sect:init}, \ref{sect:iter:FS}, \ref{sect:iter:MVR1}, and \ref{sect:iter:MVR2}, respectively.

\noi Appendix~\ref{app:sect:exp} details the experiment results.

\section{Notation}
\label{app:sect:notation}

Throughout this paper, Boldfaced lowercase letters denote vectors and uppercase letters denote real-valued matrices. We use the following notation in this paper.

\begin{itemize}[leftmargin=20pt,itemsep=0.2ex]

\item $[n]$: $\{1,2,...,n\}$;

\item $\|\mathbf{x}\|$: Euclidean norm: $\|\mathbf{x}\|=\|\mathbf{x}\|_2 = \sqrt{\la \mathbf{x},\mathbf{x}\ra}$;

\item $\la \mathbf{x}, \mathbf{y} \ra$: Euclidean inner product of vectors $\mathbf{x}$ and $\mathbf{y}$;

\item $\mathcal{X}$: The compact convex constraint set;

\item $D$: Diameter of the constraint set $\mathcal{X}$, i.e., $D = \sup_{\mathbf{x}, \mathbf{y} \in \mathcal{X}} \|\mathbf{x} - \mathbf{y}\|$;

\item $\mathcal{G}(\mathbf{x})$: (generalized) Frank-Wolfe gap,
$\mathcal{G}(\mathbf{x}) := \max_{\mathbf{v}\in\mathcal{X}}\{\langle \nabla f(\mathbf{x}),\mathbf{x}-\mathbf{v}\rangle + h(\mathbf{x})-h(\mathbf{v})\}$;

\item $\text{LMO}(\cdot)$: Linear Minimization Oracle, $\text{LMO}(\mathbf{d}) \in \arg\min_{\mathbf{v} \in \mathcal{X}} \langle \mathbf{d}, \mathbf{v} \rangle + h(\v)$;

\item $\sigma^2$: The variance bound of the stochastic gradient with $\mathbb{E}\|\nabla f(\mathbf{x};\xi)-\nabla f(\mathbf{x})\|^2\le \sigma^2$;

\item $G$: uniform bound on gradients/subgradients (as in Assumption~\ref{ass:regularity:expectation});

\item $\mathbf{g}_t$: gradient estimator (SPIDER / MVR1 / MVR2 depending on the variant);

\item $\mathbf{s}_t := \mathbf{g}_t-\nabla f(\mathbf{x}_t)$: gradient estimation error;
$s_t:=\|\mathbf{s}_t\|_2$;

\item $L$: (unknown) smoothness constant in assumptions;

\item $L_t$: adaptive curvature estimate used in the quadratic model;

\item $\delta_t:=L_t\|\mathbf{x}_{t+1}-\mathbf{x}_t\|$: scaled step difference;

\item $\alpha_t\in(0,1]$: momentum/VR weight in MVR1/MVR2;

\item $\kappa$: growth factor such that $L_t/L_{t-1}\le \kappa$ (when applicable);

\item $q$: epoch length in SPIDER-type estimator;

\item $b$: mini-batch size;

\item $\mathbb{E}[\cdot]$: Expectation operator with respect to the underlying random variables;

\item $\tilde{\mathcal{O}}(\cdot)$: Big-O notation that suppresses polylogarithmic factors;

\item $\|\mathbf{X}\|_{\fro}$: Frobenius norm, $\|\mathbf{X}\|_{\fro}=\sqrt{\sum_{i,j} \X_{ij}^2}$;

\item $\|\mathbf{X}\|_{*}$: nuclear norm (sum of singular values).

\end{itemize}

\section{Motivating Applications}
\label{app:sect:app}
Many machine learning problems can be framed as instances of Problem (\ref{eq:main}). A primary advantage of the Frank-Wolfe algorithm in these settings is its projection-free nature. By leveraging a Linear Minimization Oracle (LMO) instead of a metric projection, Frank-Wolfe often achieves significantly lower per-iteration costs, especially when the constraint set is complex. We highlight two representative examples below: optimization over the nuclear norm ball and the $\ell_p$ ball.

\subsection{Optimization over the Nuclear Norm Ball}

The nuclear norm constraint is fundamental in multi-class classification, multi-task learning, and matrix completion, where the objective is to recover a parameter matrix with a low-rank structure. Consider a set of i.i.d. labeled training data $\{(\hat{\mathbf{x}}_i, y_i)\}_{i=1}^N \subset \mathbb{R}^d \times \mathcal{Y}$, where $\mathcal{Y} = \{1, 2, \dots, c\}$ denotes the class labels. The model parameter is a matrix $\mathbf{W} \in \mathbb{R}^{c \times d}$. We aim to minimize the multinomial logistic loss subject to a nuclear norm constraint:
\begin{equation} \label{eq:app:nuclear}
\ts\min_{\mathbf{W} \in \mathbb{R}^{c \times d}} \frac{1}{N} \sum_{i=1}^N \mathcal{L}_i(\mathbf{W}) \quad \text{s.t.} \quad \|\mathbf{W}\|_* \leq \delta_1,
\end{equation}
where $\delta_1\geq 0$, and $\mathcal{L}_i(\mathbf{W}) = \log\big( 1 + \sum_{l \in \mathcal{Y} \setminus \{y_i\}} \exp(\mathbf{w}_l^\top \hat{\mathbf{x}}_i - \mathbf{w}_{y_i}^\top \hat{\mathbf{x}}_i) \big)$ denotes the loss for the $i$-th example. Here, $\mathbf{w}_l$ denotes the $l$-th row of $\mathbf{W}$. The nuclear norm constraint promotes a low-rank weight matrix. Crucially, while the Euclidean projection onto the nuclear norm ball requires a full Singular Value Decomposition (SVD), the LMO only requires computing the top singular pair (e.g., via the Power Method), which is computationally far cheaper for high-dimensional matrices.

\subsection{Optimization over the $\ell_p$ Ball}This framework encompasses a wide array of regularized Generalized Linear Models (GLMs), including high-dimensional regression and sparse classification. The problem is formulated as:
\begin{equation} \label{eq:app:l3}
\ts\min_{\mathbf{x} \in \mathbb{R}^d} \frac{1}{N} \sum_{i=1}^N \ell_i(\mathbf{x}) \quad \text{s.t.} \quad \|\mathbf{x}\|_p \leq \delta_2,
\end{equation}
where $\delta_2\geq 0$, and $\ell_i$ is a smooth convex loss function associated with $i$-th example (e.g., squared or logistic loss). For general $p \in (1, \infty)$ (e.g., $p=3$), the Euclidean projection onto the $\ell_p$ ball lacks a closed-form solution and is computationally expensive. In contrast, the LMO for $\ell_p$ balls is efficient and available in closed form via Holder's inequality (see e.g., \cite{ito2023parameter}), making Frank-Wolfe particularly well-suited for this class of problems.

\section{Comparison of Existing Conditional Gradient Methods} \label{app:sect:existing}

We provide a comparative analysis of existing conditional gradient methods in Table \ref{tab:1}, categorized by their optimization settings.

\begin{table*}[!h]
  \centering
  \caption{Comparison of existing conditional gradient methods for solving Problem (\ref{eq:main}). The notation $\tilde{\mathcal{O}}(\cdot)$ hides polylogarithmic factors, while $\mathcal{O}(\cdot)$ hides constants. ALF: Adaptive Lipschitz-Free analytic step size from quadratic upper model. Local Geometry-Aware: Stepsize depends on iterate differences / curvature estimate. }
\scalebox{0.8}{
\begin{tabular}{|p{6.8cm}|p{2.1cm}|p{2.5cm}|p{2.5cm}|}
\hline
\multicolumn{4}{|c|}{\textbf{Deterministic Setting}}
\\ \hline
 Reference &  stepsize $\eta_t$ & $f$-Value Free   &     Complexity  \\
\hline
FW-Openloop \cite{FrankWolfe1956} &    $\tfrac{2}{t+2}$ & \ding{51}    &     $\mathcal{O}(\epsilon^{-2})$ \\
 FW-ShortStep \cite{ICMLjaggi13} & constant & \ding{51}    &     $\mathcal{O}(\epsilon^{-2})$ \\
 FW-Sliding \cite{lan2017conditionalSIOPT} & constant & \ding{51}    &     $\mathcal{O}(\epsilon^{-2})$ \\
 FW-Momentum \cite{li2021heavy} & $\tfrac{2}{t+2}$ & \ding{51}    &     $\mathcal{O}(\epsilon^{-2})$ \\
 FW-Armijo \cite{ochs2019model} & backtracking & \ding{55}  &  $\mathcal{O}(\epsilon^{-2})$ \\
 FW-ParaFree \cite{ito2023parameter} & backtracking & \ding{55}   &  $\mathcal{O}(\epsilon^{-2})$ \\
 ALFCG-D [ours] & ALF  & \ding{51}    &     $\mathcal{O}(\epsilon^{-2})$ \\
\hline
\end{tabular}}

\scalebox{0.8}{
\begin{tabular}{|p{6.8cm}|p{2.9cm}|p{2.9cm}|p{3.1cm}|}
\hline
\multicolumn{4}{|c|}{\textbf{Finite-Sum Setting}}
\\ \hline 
 Reference & stepsize $\eta_t$ & Local Geometry-Aware   &     Complexity  \\
\hline
 SVFW \cite{reddi2016stochastic} & constant & \ding{55}    &   $\mathcal{O}(N + N^{2/3}\epsilon^{-2})$   \\
 SAGA-FW \cite{reddi2016stochastic} & constant & \ding{55}   &    $\mathcal{O}(N + N^{2/3}\epsilon^{-2})$   \\
 SPIDER-CG \cite{yurtsever19bSpiderCG} & constant & \ding{55}   &     $\mathcal{O}(N + \sqrt{N}\epsilon^{-2})$ \\
  OneSample-FW \cite{zhang2020oneOneSample}  & $\tfrac{1}{(t+1)^p},~p\in(0,1)$ & \ding{55} &    $\mathcal{O}(\epsilon^{-2})$   \\
 SARAH-FW \cite{beznosikov2024sarah} & constant & \ding{55}   &     $\mathcal{O}(N + \sqrt{N}\epsilon^{-2})$  \\
 ALFCG-FS [ours]  & ALF  & \ding{51} & $\mathcal{O}(N+\sqrt{N}\epsilon^{-2})$ \\
\hline
\end{tabular}}

\scalebox{0.8}{
\begin{tabular}{|p{6.8cm}|p{2.5cm}|p{2.9cm}|p{2.3cm}|p{2.5cm}|}
\hline
\multicolumn{5}{|c|}{\textbf{Expectation Setting}}
\\ \hline
Reference & Smoothness Assumption & stepsize $\eta_t$   & Mini-Batch   &     Complexity  \\
\hline
 SFW \cite{reddi2016stochastic} & average  & constant     & \ding{55} &    $\mathcal{O}(\epsilon^{-4})$ \\
 SFW-GB \cite{hazan2016variance} & average  &constant &  \ding{55} &     $\mathcal{O}(\epsilon^{-4})$ \\
OneSample-STORM \cite{zhang2020oneOneSample} &individual   & $\tfrac{1}{(t+1)^p},~p\in(0,1)$& \ding{51} &     $\mathcal{O}(\epsilon^{-3})$ \\
OneSample-EMA \cite{zhang2020oneOneSample} & average  & $\tfrac{1}{(t+1)^p},~p\in(0,1)$ &  \ding{51} &     $\mathcal{O}(\epsilon^{-4})$ \\
 SRFW\cite{tang2022high} & average  & $\tfrac{1}{(t+1)^p},~p\in(0,1)$ &  \ding{51} &     NA$^a$ \\
ALFCG-MVR1 [ours] & average  &  ALF &\ding{51}    & $\tilde{\mathcal{O}}\big( \sigma^2 \epsilon^{-4} + \epsilon^{-2} \big)$ \\
 ALFCG-MVR2 [ours] & individual & ALF &  \ding{51}  &  $\tilde{\mathcal{O}}\big( \sigma\epsilon^{-3} + \epsilon^{-2}\big)$ \\
\hline
\end{tabular}}
\label{tab:1}
\par
\vspace{1pt}
\scalebox{0.95}{\begin{tabular}{@{}p{\linewidth}@{}}
\footnotesize Note $a$: This paper focuses exclusively on the convex setting where $f(\mathbf{x})$ is convex. \\
\end{tabular}}
\label{tab:comparision}
\end{table*}

\section{Supporting Lemmas}\label{app:sect:support}

This section details a collection of useful lemmas, each of which is independent of context and methodology.

\begin{lemma}\label{lemm:ln:linear}For all $x,\lambda> 0$, we have:
\beq
\ln(x) ~\leq ~\lambda x + |\ln(\lambda)|.\nn
\eeq

\begin{proof} Since $\ln(\cdot)$ is concave, for any $x_0>0$, $\ln(x)\leq \ln(x_0) + \tfrac{1}{x_0}(x-x_0)$. Choosing $x_0=1/\lambda$ yields $\ln(x) \le -\ln(\lambda) + \lambda x - 1$, leading to $\ln(x) \leq \lambda x + |\ln(\lambda)|$.

\end{proof}

\end{lemma}

\begin{lemma}\label{eq:ab:sort}Let $\{A_t\}_{t=0}^{n}$ and $\{C_t\}_{t=0}^{n+1}$ be two nonnegative sequences, with $\{A_t\}_{t=0}^n$ non-decreasing. Then we have
\beq
\ts \sum_{t=0}^n A_t (C_t - C_{t+1})~\leq~ A_n (\max_{t=0}^n C_t).\nn
\eeq

\begin{proof} We have: $\ts \sum_{t=0}^n A_t (C_t - C_{t+1} ) \nn =\ts A_0 C_0 - A_n C_{n+1} +  \sum_{t=1}^n (A_t - A_{t-1}) C_t  \overset{\step{1}}{\leq} \ts   A_0 C_0 +  \sum_{t=1}^n (A_t - A_{t-1}) C_t\overset{\step{2}}{\leq} \ts  A_0 \max_{t=0}^n (C_t) + (\max_{j=0}^n C_j) \sum_{t=1}^n ( A_t - A_{t-1})  =\ts  A_n (\max_{t=0}^n C_t)$, where step \step{1} uses $A_n,C_{n+1}\geq 0$; step \step{2} uses $\{A_t\}_{t=1}^n$ is non-decreasing.

\end{proof}

\end{lemma}

\begin{lemma} \label{lemma:A:alpha} Let $\{A_t\}_{t=0}^T$ be a nonnegative sequence satisfying $\ts A_{T+1}\,\le\,\zeta+\alpha \max_{t=0}^T A_t$, where $\zeta\ge 0$ and $\alpha\in[0,1)$. Then, for all $t \ge 0$,
$$
A_t~ \leq~ \max (A_0,\tfrac{\zeta}{1-\alpha}).
$$

\begin{proof} Define $M_t:= \max_{i=0}^t A_i$, and $\overline{A}:= \max\{A_0,\tfrac{\zeta}{1-\alpha}\}$.

\noi We prove by induction that $M_t\le \overline{A}$ for all $t\geq 0$.

\paragraph{Base case.} $M_t\le \overline{A}$ holds trivially for $t=0$.

\paragraph{Inductive step.} Fix some $t$ and assume that $M_t\le \overline{A}$ holds for $t-1$. We have:
\beq
M_{t}\,=\,\max(M_{t-1},A_{t})~\overset{\step{1}}{\le}~\max(\overline{A},\zeta+\alpha \overline{A})~\overset{\step{2}}{\leq}~\overline{A}, \nn
\eeq
\noi where step \step{1} uses $M_{t-1}\leq \overline{A}$ and the condition of this lemma that $A_{T+1}\le \zeta+\alpha \max_{t=0}^T A_t$; step \step{2} uses $\zeta+\alpha \overline{A}\le \overline{A}$, which is implied by $\tfrac{\zeta}{1-\alpha}\leq \overline{A}$.

\end{proof}

\end{lemma}

\begin{lemma}
\label{lemma:sqrt:2} Assume $\{A_t\}_{t=0}^{T}$ is a nonnegative sequence. Then
\beq
\ts \sum_{t=0}^T \frac{A_t}{\sqrt{1 + \sum_{i=0}^t A_i}} ~\leq~2\sqrt{1 + \sum_{i=0}^T A_i}.\nn
\eeq

\begin{proof} This lemma extends the result of Lemma 5 in \cite{McMahanS10}.

\noi Define $S_t := 1 + \sum_{i=0}^t A_i$. We finish the proof using induction.

\paragraph{Auxiliary Inequality.} Let $h(x) :=  \tfrac{x}{\sqrt{y}} + 2 \sqrt{y-x} - 2 \sqrt{y} $, where $y> x\geq 0$. Since $\nabla h(x) = \tfrac{1}{\sqrt{y}} -\tfrac{1}{\sqrt{y-x}}\leq 0$ and $h(0)=0$, we conclude that $h(x)$ is non-increasing for all $x\geq 0$, leading to:
\beq \label{eq:sqrt:2:aux}
h(x) ~:= ~ \tfrac{x}{\sqrt{y}} + 2 \sqrt{y-x} - 2 \sqrt{y}~\leq ~0,\quad \forall y> x\geq 0.
\eeq

\paragraph{Base case.} The lemma holds trivially with $T = 0$.

\paragraph{Induction step.} Fix some $T$ and assume that the lemma holds for $T-1$. We derive:
\beq
\ts \sum_{t=0}^T {A_t}/{\sqrt{S_t}} ~\overset{\step{1}}{\leq}~ A_T / {\sqrt{S_T}}+ 2 \sqrt{S_{T-1}} ~\overset{\step{2}}{=}~ {A_T}/{\sqrt{S_T}}+ 2 \sqrt{S_T - A_T } ~\overset{\step{3}}{\leq}~ 2 \sqrt{S_T}, \nn
\eeq
\noi where step \step{1} uses the inductive hypothesis for $T-1$, namely that $\sum_{t=0}^{T-1} {A_t}/{\sqrt{S_t}} \leq 2 \sqrt{S_{T-1}}$; step \step{2} uses $S_T := 1+\sum_{t=0}^T A_t$; step \step{3} uses Inequality (\ref{eq:sqrt:2:aux}) with $y=S_T$ and $x=A_T$.

\end{proof}
\end{lemma}

\begin{lemma} \label{lemma:online:ppp} Assume $\{A_t\}_{t=0}^{T}$ is a nonnegative sequence. Then
\[
\ts \sum_{t=0}^T \tfrac{A_t}{ 1 + \sum_{i=0}^t A_i} ~\leq~\ln( 1 + \sum_{i=0}^T A_i ).
\]
\begin{proof} This lemma extends the result of Lemma 3.2 in \cite{ward2020adagrad}.

\noi Define $S_t := 1+\sum_{i=0}^t A_i$.


\paragraph{Auxiliary inequality.} Let $h(x) := x + \ln(1-x )$, where $x\in [0,1)$. Since $h(0)=0$ and $h'(x) = -\tfrac{x}{1-x}\leq 0$, we have
\beq \label{lemma:online:ppp:aux}
h(x)~:= ~x + \ln(1 - x) \leq 0,\quad\forall x \in [0,1).
\eeq

\paragraph{Base case.} For $T=0$, we have $A_0=S_0$, resulting in:
\beq
 \tfrac{A_0}{1 + A_0} - \ln(1+ A_0)~\overset{}{=}~\tfrac{A_0}{1 + A_0} + \ln(1 - \tfrac{A_0}{1 + A_0}) ~ \overset{\step{1}}{\leq}~0,\nn
\eeq
\noi where step \step{1} uses Inequality (\ref{lemma:online:ppp:aux}) that $h(x)\leq 0$ for all $x\in [0,1)$ with $x=\tfrac{A_0}{1 + A_0}$.

\paragraph{Induction step.} Now, fix some $T$ and assume that the claim holds for $T-1$. Then
\beq
\ts \sum_{t=0}^T \tfrac{A_t}{S_t} ~\overset{\step{1}}{\leq}~ \ts \tfrac{A_T}{S_T}+\ln(S_{T-1} ) ~\overset{\step{2}}{\leq}~ \ts - \ln(1 - \tfrac{A_T}{S_T}) + \ln(S_T - A_T) ~\overset{\step{3}}{=}~ \ts \ln(S_T) , \nn
\eeq
\noi where step \step{1} uses the inductive hypothesis $T-1$, namely that $\sum_{t=0}^{T-1}\tfrac{A_{t}}{S_{t}}\leq \ln(S_{T-1})$; step \step{2} uses $S_T := 1+\sum_{t=0}^T A_t$ and Inequality (\ref{lemma:online:ppp:aux}) with $x = \tfrac{A_T}{S_T}$; step \step{3} uses $\ln(a) - \ln(b) = \ln(\tfrac{a}{b})$ for all $a,b>0$.

\end{proof}
\end{lemma}

\section{Proof for Section \ref{sect:iter:Init} } \label{app:sect:iter:Init}

\subsection{Proof of Lemma \ref{lemma:descent:ncvx}}
\label{app:lemma:descent:ncvx}

\begin{proof} Let $\gamma_t \in (0,1]$ be an arbitrary constant.

\noi Define the quadratic surrogate $\PP(\eta) :=  \frac{L_t}{2} \|\v_t-\x_t\|_2^2 \eta^2 + \eta\cdot\left( \la \v_t-\x_t,\g_t\ra + h(\v_t) - h(\x_t) \right)$.

\noi For any $\x_t\in\mathcal{X}$, the generalized Frank-Wolfe gap is defined as $$\mathcal{G}(\x_t) ~:=~ \left\{ \langle \nabla f(\mathbf{x}_t), \mathbf{x}_t - \tilde{\v}_t \rangle + h(\mathbf{x}_t) - h(\tilde{\v}_t) \right\},~\text{where}~\tilde{\v}_t \in \text{LMO}(\nabla f(\x_t)).$$

\paragraph{Bounding the Descent $F(\x_{t+1}) - F(\x_{t})$.} Using the $L$-smoothness of $f(\x)$, we have:
\beq \label{eq:opt:f}
f(\x_{t+1})~\leq~ f(\x_t)+\la\x_{t+1}-\x_t,\nabla f(\x_t)\ra+ \tfrac{L}{2}\|\x_{t+1}-\x_t\|_2^2.
\eeq
\noi By the convexity of $h(\cdot)$, the update $\mathbf{x}_{t+1} = (1-\bar{\eta}_t)\mathbf{x}_t + \bar{\eta}_t \mathbf{v}_t$ satisfies:
\begin{equation} \label{eq:opt:h}
h(\mathbf{x}_{t+1}) ~\leq ~ (1-\bar{\eta}_t) h(\mathbf{x}_t) + \bar{\eta}_t h(\mathbf{v}_t).
\end{equation}
\noi Combining Inequalities (\ref{eq:opt:f}) and (\ref{eq:opt:h}) yields:
\beq \label{eq:333:bound:1}
&& \ts F(\x_{t+1}) - F(\x_{t}) \nn\\
&\leq&\ts \la\x_{t+1}-\x_t,\nabla f(\x_t)\ra + \bar{\eta}_t \left( h(\v_t) - h(\x_t)\right)  + \tfrac{L}{2}\|\x_{t+1}-\x_t\|_2^2 \nn\\
&\overset{\step{1}}{=}&\ts \la \x_{t+1}-\x_t,\nabla f(\x_t) - \g_t \ra  +  \PP(\bar{\eta}_t) +  \big (\tfrac{L}{2} - \tfrac{L_t}{2}\big)\|\x_{t+1}-\x_t\|_2^2  \nn\\
&\overset{\step{2}}{\leq}&\ts \tfrac{L}{2} \|\x_{t+1}-\x_t\|_2^2  - \tfrac{L_t}{4} \|\x_{t+1}-\x_t\|_2^2+ \tfrac{1}{L_t}\|\nabla f(\x_t) - \g_t\|_2^2 +  \PP(\gamma_t)\nn\\
&\overset{\step{3}}{=}&\ts \tfrac{L \delta_t^2}{2 L_t^2}   - \tfrac{ \delta_t^2}{4 L_t}   + \tfrac{s_t^2}{L_t}  +  \PP(\gamma_t),
\eeq
\noi where step \step{1} uses $\bar{\eta}_t  (\v_t-\x_t) =\x_{t+1}-\x_t$, and the definition of $\PP(\bar{\eta}_t)$; step \step{2} uses $\la\mathbf{a},\mathbf{b}\ra\leq \tfrac{L_t}{4}\|\mathbf{a}\|_2^2+\tfrac{1}{L_t}\|\mathbf{b}\|_2^2$ for all $\mathbf{a},\mathbf{b}\in\Rn^n$, and the optimality of $\eta_t$ that $\PP(\bar{\eta}_t)\leq \PP(\gamma_t)$ for all $\gamma_t \in (0,1]$; step \step{3} uses $\delta_t=L_t\|\x_{t+1} - \x_t\|$ and $s_t=\|\nabla f(\x_t) - \g_t\|$.

\paragraph{Key Inequality Bounding $\PP(\gamma_t)$.} We further bound $\PP(\gamma_t)$, as follows:
\beq \label{eq:333:bound:2}
\ts \PP(\gamma_t) &:=&\ts \tfrac{\gamma_t^2 L_t}{2}\|\v_t-\x_t\|_2^2 + \gamma_t \left( \la \v_t-\x_t,\g_t \ra + h(\v_t) - h(\x_t) \right)   \nn\\
&\overset{\step{1}}{\leq}&\ts \tfrac{\gamma_t^2 L_t}{2}D^2  +   \gamma_t \left( \la \tilde{\v}_t-\x_t,\g_t \ra + h(\tilde{\v}_t) - h(\x_t) \right)  \nn\\
&\overset{\step{2}}{=}&\ts  \tfrac{\gamma_t^2 L_t}{2}D^2 -\gamma_t \GG(\x_t)  + \gamma_t \la\tilde{\v}_t-\x_t  ,\g_t - \nabla f(\x_t)\ra   \nn\\
&\overset{\step{3}}{\leq}&\ts \tfrac{\gamma_t^2 L_t}{2}D^2 - \gamma_t \GG(\x_t) + \gamma_t  D  \|\g_t - \nabla f(\x_t)\|  \nn\\
&\overset{\step{4}}{\leq}&\ts \tfrac{\gamma_t^2 L_t}{2}D^2 - \gamma_t \GG(\x_t) + \tfrac{s_t^2}{2 L_t} +\tfrac{\gamma_t^2 L_t}{2}D^2,
\eeq
\noi where step \step{1} uses $\|\v_t-\x_t\|\leq D$ as $\v_t,\x_t\in\mathcal{X}$, and the optimality of $\v_t \in \arg \min_{\v \in \mathcal{X}}\la \v-\x_t, \g_t \ra + h(\v)$ that
\beq
\la \v_t,\g_t\ra + h(\v_t)~\leq~ \la \tilde{\v}_t,\g_t\ra + h(\tilde{\v}_t);\nn
\eeq
\noi step \step{2} uses the definition of $\GG(\x_t)$; step \step{3} uses $\la\mathbf{a},\mathbf{b}\ra\leq\|\mathbf{a}\|\|\mathbf{b}\|$ and $\|\tilde{\v}_t-\x_t\|\leq D$; step \step{4} uses $s_t=\|\g_t-\nabla f(\x_t)\|$ and $ab\leq \tfrac{1}{2}a^2+\tfrac{1}{2}b^2$ for all $a,b\in\Rn$ with $a=\tfrac{s_t}{\sqrt{L_t}}$ and $b=\gamma_t D\sqrt{L_t}$.

Finally, combining Inequalities (\ref{eq:333:bound:1}) and (\ref{eq:333:bound:2}) results in:
\beq
\ts \gamma_t\GG(\x_t)   + F(\x_{t+1})  - F(\x_{t}) + \tfrac{ \delta_t^2}{4 L_t}~ \leq~  \tfrac{L \delta_t^2}{2 L_t^2}     + \tfrac{s_t^2}{L_t}  +\gamma_t^2 L_t D^2 + \tfrac{s_t^2}{2 L_t}   ~\overset{}{\leq} ~  \tfrac{L\delta_t^2}{2 L_t^2} +  \tfrac{2 s_t^2}{L_t}  + \gamma_t^2 L_t D^2. \nn
\eeq

\end{proof}

\section{Proof for Section \ref{sect:iter:FS} }
\label{app:sect:iter:FS}

\subsection{Proof of Lemma \ref{lemma:kappa}}
\label{app:lemma:kappa}

\begin{proof} We now focus on ALFCG-D and ALFCG-FS.

From the definition of $L_{t+1}^2 = \rho^2 \left( 1 + \sum_{i=0}^{t} L_i^2 \|\mathbf{x}_{i+1}-\mathbf{x}_i\|_2^2 \right)$, we isolate the last term of the summation for all $t\geq 0$:
\beq
\ts L_{t+1}^2 ~= ~\underbrace{\ts \rho^2 \big( 1 + \sum_{i=0}^{t-1} L_i^2 \|\mathbf{x}_{i+1}-\mathbf{x}_i\|_2^2 \big)}_{L_t^2} + \rho^2 L_t^2 \|\mathbf{x}_{t+1}-\mathbf{x}_t\|_2^2.\nn
\eeq
\noi This simplifies to the recursive relation: $L_{t+1}^2 = L_t^2 \left( 1 + \rho^2 \|\mathbf{x}_{t+1}-\mathbf{x}_t\|_2^2 \right)$. Dividing both sides by $L_t^2$ yields the squared ratio:
\beq
\tfrac{L_{t+1}^2}{L_t^2} ~=~ 1 + \rho^2 \|\mathbf{x}_{t+1}-\mathbf{x}_t\|_2^2.\nn
\eeq
\noi Taking the square root, we obtain the following bound for all $t \geq 0$:
\beq
\tfrac{L_{t+1}}{L_t} ~=~ \sqrt{1 + \rho^2 \|\mathbf{x}_{t+1}-\mathbf{x}_t\|_2^2} ~\leq ~\sqrt{1 + \rho^2 D^2}\leq 1 + \rho D,\nn
\eeq
\noi where $D$ is the diameter of the constraint set $\mathcal{X}$.

\end{proof}

\subsection{Proof of Lemma \ref{lemma:bound:F:ncvx}}
\label{app:lemma:bound:F:ncvx}

\begin{proof} We now focus on ALFCG-D.

\noi We let $L_{t} = \rho \big( 1 + \sum_{i=0}^{t-1} \delta_i^2 \big)^{1/2}$, where $\delta_i := L_i \|\x_{i+1}-\x_i\|$.

\noi For all $0\leq t\leq T$, we let $\gamma_t = \tfrac{L_0}{L_t \sqrt{T+1}}$. We define $F_t^+:= F(\x_t) - F_*$.

\paragraph{Part (a).} We have from Inequality (\ref{eq:analysis:first}):
\beq
\EEE[\gamma_t\GG(\x_t) + F^+_{t+1} - F^+_t +\tfrac{\delta_t^2}{4 L_t} ]   ~~\leq~~  \EEE[\tfrac{\delta_t^2 L}{2 L_t^2}    +  \tfrac{2 s_t^2}{ L_t} + \gamma_t^2 L_t D^2] ~~\overset{\step{1}}{=}~~\EEE[\tfrac{\delta_t^2 L}{2 L_t^2}  + \tfrac{L_0^2 D^2 }{L_t (T+1)} ],\label{eq:to:be:view:0}
\eeq
\noi where step \step{1} uses $\s_t=\zero$ for the deterministic setting, and $\gamma_t = \tfrac{L_0}{L_{t} \sqrt{T+1}}$. Dropping the non-negative term $\gamma_t\GG(\x_t)$ and multiplying both sides by $L_t$ leads to:
\beq \label{eq:to:tel:D}
\ts 0 ~~\leq~~ \EEE[ L_t ( F^+_{t} - F^+_{t+1}) + \tfrac{L_0^2 D^2}{T+1} + \ts \tfrac{L \delta_t^2}{2 L_{t}}  - \tfrac{\delta_t^2 }{4}].
\eeq
\noi By telescoping Inequality (\ref{eq:to:tel:D}) over $t$ from $0$ to $T$ and dividing both sides by $L_{T+1}$, we obtain:
\beq
\ts 0 & \leq& \ts \tfrac{1}{L_{T+1}}\sum_{t=0}^T [ \tfrac{ L_0^2 D^2 }{ T+1} + L_t (F^+_{t} - F^+_{t+1}) + \tfrac{ L \delta_t^2}{2 L_{t}}  - \tfrac{ \delta_t^2}{4}  ] \nn\\
&\overset{\step{1}}{\leq}& \ts \rho D^2 +  \max_{t=0}^T F^+_{t} + \tfrac{L \kappa}{2 L_{T+1}}  \sum_{t=0}^T \tfrac{\delta_t^2}{ \rho \sqrt{1 + \sum_{i=0}^t \delta_i^2 }} - \tfrac{1}{4 L_{T+1}} \cdot \sum_{t=0}^T \delta_t^2\nn\\
&\overset{\step{2}}{\leq}& \ts \rho D^2 +   \max_{t=0}^T F^+_{t} + \tfrac{L \kappa}{2 L_{T+1}} \cdot \frac{2}{\rho} \sqrt{ 1 + \sum_{t=0}^T \delta_t^2 }    - \tfrac{1}{4\rho^2 L_{T+1}} (  L_{T+1}^2 - \rho^2 )\nn\\
&\overset{\step{3}}{\leq}& \ts \rho D^2 +   \max_{t=0}^T F^+_{t} + \tfrac{ L \kappa}{\rho^2} \ts - \tfrac{L_{T+1} }{4\rho^2}  + \tfrac{1}{4\rho},\nn
\eeq
\noi where step \step{1} uses $\tfrac{L_{t+1}}{L_t}\leq \kappa$, Lemma \ref{eq:ab:sort}, and $L_{t+1}^2 = \rho^2   + \rho^2 \sum_{i=0}^t \delta_i^2$; step \step{2} uses Lemma \ref{lemma:sqrt:2}; step \step{3} uses $\tfrac{L_0}{L_{T+1}}\leq 1$. Multiplying both sides by $4\rho^2$ leads to:
\beq \label{eq:LT1:upperbound:D}
L_{T+1} ~\leq ~ \left(  4\rho^2 \max_{t=0}^T F^+_{t}\right)+\underbrace{\ts \rho + 4 L \kappa +  4 \rho^3 D^2 }_{:= L'}.
\eeq
The upper bound for $L_{T+1}$ is established in Inequality (\ref{eq:LT1:upperbound:D}), but it depends on the unknown variable $\max_{t=0}^T F_t^+$.

\paragraph{Part (b).} We now prove that $(\max_{t=0}^{T} F_t^+)$ is always bounded above by a universal constant $\overline{F}$. In view of Inequality (\ref{eq:to:be:view:0}), we have:
\beq
\gamma_t\GG(\x_t) + F^+_{t+1} ~\leq~ \ts F^+_t  + \tfrac{D^2 \rho L_0 }{ L_t (T+1)} +       \tfrac{L \delta_t^2}{2 L_t^2} ~\leq ~ F^+_t  + \tfrac{D^2 \rho }{ T+1} +  \tfrac{ L\kappa^2 \delta_t^2}{2 L_{t+1}^2}.\nn
\eeq
\noi Summing this inequality over $t$ from $0$ to $T$ yield:
\beq
&& \ts F^+_{T+1} - F^+_{0} - \rho D^2 + \sum_{t=0}^T \gamma_t\GG(\x_t) ~~\leq~~ \ts       \tfrac{L\kappa^2}{2} \sum_{t=0}^T \tfrac{\delta_t^2}{L_{t+1}^2} \nn\\
&\overset{\step{1}}{\leq}& \ts \tfrac{L \kappa^2}{2 \rho^2 } \cdot \sum_{t=0}^T \tfrac{\delta_t^2}{1 +  \sum_{i=0}^t \delta_i^2} ~~\overset{\step{2}}{\leq}~~ \ts  \tfrac{L \kappa^2}{2\rho^2} \cdot  \ln( 1 + \sum_{t=0}^T \delta_t^2 )  ~~\overset{\step{3}}{=} ~~\tfrac{L \kappa^2}{2\rho^2} \cdot  \ln( \tfrac{L_{T+1}^2}{\rho^2})  \nn\\
& = & \ts \tfrac{L \kappa^2}{\rho^2} \cdot  \ln( \tfrac{L_{T+1}}{\rho})  ~~\leq ~~  \tfrac{ L \kappa^2}{\rho^2} (\ln (L_{T+1}) + |\ln(\rho)|),  \nn
\eeq
\noi where step \step{1} uses the definition of $L_{t+1}$; step \step{2} uses Lemma (\ref{lemma:online:ppp}) with $A_t = \delta_t^2$; step \step{3} uses $L_{t}^2 = \rho^2 \big( 1 + \sum_{i=0}^{t-1} \delta_i^2 \big)$.

\noi We define $\omega:=\tfrac{ L \kappa^2}{\rho^2}$. For any $\lambda>0$, we derive:
\beq \label{eq:F:GG:Det}
\ts F^+_{T+1} + \sum_{t=0}^T \gamma_t\GG(\x_t) &\leq&  \ts F^+_{0} + D^2 \rho + \omega |\ln(\rho)|  + \omega \ln (L_{T+1}) \nn\\
&\overset{\step{1}}{\leq}& \ts  F^+_{0} +  D^2\rho     + \omega  |\ln(\rho)|  + \omega |\ln(\lambda)|  + \omega \lambda L_{T+1}  \nn\\
&\overset{\step{2}}{\leq}& \ts F^+_{0} + D^2\rho    + \omega  |\ln(\rho)|  + \omega |\ln(\lambda)|  + \omega \lambda \big( L' + 4\rho^2 \max_{t=0}^T F_t^+ \big) \nn\\
&\overset{\step{3}}{\leq}& \ts \underbrace{\ts F^+_{0} +  D^2\rho    + \omega |\ln(\rho)|  + \omega |\ln(\lambda)|  + \omega \lambda L' }_{:= \zeta}  + \tfrac{1}{2} \max_{t=0}^T F_t^+ ,
\eeq
\noi where step \step{1} uses $\ln(x) \leq |\ln(\lambda)|+\lambda x $ with $x=L_{T+1}$ for all $\lambda>0$; step \step{2} uses $L_{T+1}\leq L' + 4 \rho^2 \max_{t=0}^T F_t^+$; step \step{3} uses the choice $\lambda = \tfrac{1}{8 \omega \rho^2}$. Dropping the nonnegative term $\sum_{t=0}^T \gamma_t\GG(\x_t)$ and apply Lemma \ref{lemma:A:alpha} with $\alpha=\tfrac{1}{2}$ and $A_t=F_{t}^+$, we have:
\beq
F_{t}^+~ \leq~ \max(F_{0}^+,2 \zeta) ~=~ 2 \zeta ~:= ~\overline{F}.\nn
\eeq

\paragraph{Part (c).} We have from Inequality (\ref{eq:LT1:upperbound:D}):
\beq
L_{T+1} ~\leq~ \ts L' +  4 \rho^2 \max_{t=0}^T F_t^+  ~\leq ~ \ts  L'  + 4 \rho^2 \overline{F} ~:=~\overline{L}.\nn
\eeq

\paragraph{Part (d).} We have from Inequality (\ref{eq:F:GG:Det}):
\beq
\ts \sum_{t=0}^T \gamma_t\GG(\x_t) ~\leq~ \ts  \zeta + \tfrac{1}{2}\overline{F}  ~=  ~ \ts  \overline{F}.\nn
\eeq

\end{proof}

\subsection{Proof of Theorem \ref{theorem:iter:complexity:ncvx}}
\label{app:theorem:iter:complexity:ncvx}

\begin{proof} Define $C:=\tfrac{\overline{F} \cdot \overline{L} }{\rho}$ and set $L_0=\rho$.

\noi Fix an integer $T\geq 0$. For all $t = 0,1,\ldots,T$, let $\gamma_t := \tfrac{L_0}{L_t \sqrt{T+1}}$.

Using Lemma \ref{lemma:bound:F:ncvx:VR}(d) and the fact that $\gamma_t\geq \gamma_T$ for all $t\leq T$, we obtain:
\beq
\ts \sum_{t=0}^{T} \GG(\x_t)~ \leq ~ \tfrac{1}{\gamma_{T}} \sum_{t=0}^T \gamma_t \GG(\x_t) ~\leq~ \tfrac{1}{\gamma_{T}} \overline{F} ~ = ~ \overline{F} \cdot \tfrac{L_{T} \sqrt{T+1}}{L_0}~ = C \sqrt{T+1}. \nn
\eeq
\noi This leads to:
\beq
\ts \tfrac{1}{T+1} \sum_{t=0}^{T} \GG(\x_t) ~\leq~ \tfrac{C}{\sqrt{T+1}}.\nn
\eeq
\noi In other words, for any $\epsilon > 0$, the algorithm finds an $\epsilon$-approximate solution (i.e., $\min_{0 \le t \le T}\mathcal{G}(\mathbf{x}_t) \leq \epsilon$) within at most $T = \lceil C^2 \epsilon^{-2} - 1\rceil$ iterations.

\end{proof}

 \subsection{Proof of Lemma \ref{lemma:f:nablaf}}
\label{app:lemma:f:nablaf}

\begin{proof} Using Lemma \ref{eq:spider:estimator}, we have:
\beq
\ts \EEE [\|\g_t- \nabla f(\x_t)\|_2^2] - \|\g_{t-1} - \nabla f(\x_{t-1})\|_2^2 ~\leq~ \ts  \tfrac{L^2}{b}\EEE [ \|\x_t - \x_{t-1}\|_2^2 ].\nn
\eeq
Telescoping the above inequality over $t$ from $(r_{t}-1)q+1$ to $t$, where $t\leq r_{t} q -1$, we obtain:
\beq \label{eq:approximation:vf}
\ts \EEE[\|  \g_t - \nabla f(\x_t)    \|_2^2] &\overset{}{\leq} &  \ts \left(\tfrac{L^2}{b} \sum_{i= (r_{t}-1)q+1 }^{t}\EEE[ \|\x_{i} - \x_{i-1}\|_2^2 ] \right) + \EEE[\|\g_{ (r_t-1)q} - \nabla f (\x_{  (r_t-1)q })\|_2^2]    \nn\\
&\overset{\step{1}}{=} &  \ts \left( \tfrac{L^2}{b} \sum_{i= (r_{t}-1)q }^{t-1} \ts \EEE[ \|\x_{i+1} - \x_{i}\|_2^2 ] \right) + 0,
\eeq
\noi where step \ding{172} uses $\g_{j} = \nabla f (\x_{j})$ when $j$ is a multiple of $q$. Notably, Inequality (\ref{eq:approximation:vf}) holds for every $t$ of the form $t = (r_{t}-1)q$, since at these points we have $\g_t=\nabla f(\x_t)$.

\end{proof}

\subsection{Proof of Lemma \ref{lemma:sum:y:two}}
\label{app:lemma:sum:y:two}

\begin{proof} Let $\{L_t,M_t\}_{t=0}^{\infty}$ be two nonnegative sequences. Assume $1\leq \tfrac{L_{t+1}}{L_t}\leq \kappa$, where $\kappa\geq 1$.

\paragraph{The first auxiliary inequality.} For any integer $t\geq 0$, we obtain:
\beq \label{eq:length:tt}
t -  (r_t -1) q   ~\overset{\step{1}}{=}  ~   t -  ( \lfloor\tfrac{t}{q}\rfloor+1  -1) q  ~=  ~ t -  \lfloor\tfrac{t}{q}\rfloor  q ~ \overset{\step{2}}{\leq} ~  q - 1,
\eeq
\noi where step \step{1} uses $r_t := \lfloor\tfrac{t}{q}\rfloor+1$; step \step{2} uses $t -  \lfloor\tfrac{t}{q}\rfloor  q\leq q-1$ for all integer $t\geq 0$ and $q\geq 1$.

\paragraph{The second auxiliary inequality.} For any $t$ with $t\geq (r_t-1)q$, we have the following results:
\beq \label{eq:vv:rt1:q:rt1q}
\tfrac{ L_t }{ L_{(r_t-1)q} }  ~=  ~ \tfrac{ L_{(r_t-1)q+1}}{ L_{ (r_t-1)q }} \cdot \tfrac{ L_{(r_t-1)q+2} }{ L_{ (r_t-1)q+1 }}   \ldots \cdot \tfrac{ L_t}{L_{t-1}}~~\overset{\step{1}}{\leq}  ~~ \kappa^{q-1},
\eeq
\noi where step \step{1} uses the fact that the product length is at most $\left( t - (r_{t}-1)q  \right)$ and Inequality (\ref{eq:length:tt}).

\paragraph{Part (a).} For all $t$ with $(r_t-1)q \leq t \leq r_t q-1$, we have:
\beq\label{eq:vy:vy:2}
\ts \sum_{j=(r_{t}-1)q}^{t}  \sum_{i=(r_{j}-1)q}^{j-1} M_i &\overset{\step{1}}{=} & \ts \sum_{i=(r_{t}-1)q}^{t-1} (t-i) M_i \nn\\
&\overset{\step{2}}{\leq} & \ts \left( [t-1]- [(r_t-1)q] + 1\right) \cdot \sum_{i=(r_{t}-1)q}^{t-1}  M_i  \nn\\
&\overset{\step{3}}{\leq} & \ts (q-1)\sum_{i=(r_{t}-1)q}^{t-1}M_i ,   \nn
\eeq
\noi where step \step{1} uses basic reduction; step \step{2} uses $i\geq (r_{t}-1)q$; step \step{3} uses Inequality (\ref{eq:length:tt}).

\paragraph{Part (b).} For all $t$ with $(r_t-1)q \leq t \leq r_t q-1$, we have:
\beq \label{eq:vy:vy}
\ts \sum_{j={(r_t-1)q}}^t \left( \tfrac{1}{L_j} \sum_{i=(r_{j}-1)q}^{j-1} M_i  \right) &\overset{\step{1}}{\leq} & \ts \sum_{j={(r_t-1)q}}^t  \left( \tfrac{1}{L_j}  \sum_{i=(r_{j}-1)q}^{t-1}  M_i \right) \nn\\
&\overset{\step{2}}{=} & \ts \sum_{j={(r_t-1)q}}^t \left( \tfrac{1}{L_j} \cdot  \sum_{i=(r_{t}-1)q}^{t-1}  M_i \right) \nn\\
&\overset{\step{3}}{=} & \ts \sum_{i=(r_{t}-1)q}^{t-1}  \left(  M_i \cdot   \sum_{j={(r_t-1)q}}^t \tfrac{1}{L_j}  \right) \nn\\
&\overset{\step{4}}{\leq} & \ts \sum_{i=(r_{t}-1)q}^{t-1}  \left(  M_i \cdot   \sum_{j={(r_t-1)q}}^t \tfrac{\kappa^{q-1}}{L_t}  \right) \nn\\
&\overset{\step{5}}{\leq} & \ts \ts q \kappa^{q-1} \cdot \sum_{i=(r_{t}-1)q}^{t-1} \tfrac{M_i}{L_i},\nn
\eeq
\noi where step \step{1} uses $j\leq t$ for all $j\in [ (r_t-1)q,t]$; step \step{2} uses $r_j=r_t$ for all $j\in [ (r_t-1)q,t]$ with $t\in [(r_t-1)q , r_t q-1]$; step \step{3} uses the fact that $\sum_{j=\underline{j}}^{\overline{j}} (\mathbf{a}_j \sum_{i=\underline{i}}^{ \overline{i}} \mathbf{b}_i) = \sum_{i=\underline{i}}^{ \overline{i}} ( \mathbf{b}_i \sum_{j=\underline{j}}^{  \overline{j}} \mathbf{a}_j )$ for all $\underline{i}\leq\overline{i}$ and $\underline{j}\leq \overline{j}$; step \step{4} uses $L_j\leq L_t$ as $j\leq t$; step \step{5} uses $t- (r_t-1)q \leq q$; step \step{6} uses Inequality (\ref{eq:vv:rt1:q:rt1q}); step \step{7} uses $i\geq (r_t-1)q$.

\paragraph{Part (c).} We have the following results:
\beq\label{eq:vy:vy:2}
\ts \sum_{t=0}^{T}  [\sum_{i=(r_{t}-1)q}^{t-1} M_i]  ~ \overset{\step{1}}{\leq}~ \ts \left( t - (r_{t}-1)q   \right) \sum_{t=0}^{T} M_t\nn~\overset{\step{2}}{\leq}~\ts (q-1) \sum_{t=0}^{T} M_t ,\nn
\eeq
\noi step \step{1} uses the fact that the length of the summation is $\left( t - (r_{t}-1)q\right)$; step \step{2} uses Inequality (\ref{eq:length:tt}).

\end{proof}

\subsection{Proof of Lemma \ref{lemma:bound:F:ncvx:VR}} \label{app:lemma:bound:F:ncvx:VR}

\begin{proof} We define $F^+_t := F(\x_t) - F_*$.

\noi Fix an integer $T\geq 0$. For all $t = 0,1,\ldots,T$, let $\gamma_t := \tfrac{L_0}{L_t \sqrt{T+1}}$.

\noi Let $L_{t} = \rho \big(1 + \sum_{i=0}^{t-1}\delta_i^2\big)^{1/2}$, where $\delta_t = L_t \|\x_{t+1}-\x_t\|$.

\noi For all $t$ with $(r_t-1)q \leq t\leq r_t q-1$, we have from Inequality (\ref{eq:analysis:first}):
\beq \label{eq:hhhh:VR}
\ts \EEE[\gamma_t\GG(\x_t)  + F^+_{t+1} - F^+_t + \tfrac{\delta_t^2}{4 L_t} - \tfrac{\delta_t^2 L}{2 L_t^2} ] &\leq& \EEE[   \tfrac{2\|\s_t\|_2^2}{ L_t}  +  \gamma_t^2 L_t D^2  ] \nn\\
&\overset{\step{1}}{\leq}& \ts \EEE[ \tfrac{2 L^2}{b L_t} \sum_{i=(r_t-1)q}^{t-1} \|\x_{i+1}-\x_{i}\|_2^2  +  \tfrac{ \rho^2 D^2}{L_t (T+1)} ],~~~~~
\eeq
\noi where step \step{1} uses Lemma \ref{lemma:f:nablaf} and $\gamma_t = \tfrac{L_0}{L_t \sqrt{T+1}}$.

\paragraph{Part (a).} Dropping the term $\gamma_t\GG(\x_t)$ and multiplying both sides of Inequality (\ref{eq:hhhh:VR}) by $L_t$ yields:
\beq
0  &\leq& \ts \EEE[ L_t (F^+_t - F^+_{t+1} )  +  \tfrac{\delta_t^2 L }{2 L_{t}} -  \tfrac{\delta_t^2}{ 4 } +  \tfrac{2 L^2}{b} \sum_{i=(r_t-1)q}^{t-1} \|\x_{i+1}-\x_{i}\|_2^2  +  \tfrac{ L_0^2 D^2}{T+1}].\nn
\eeq
\noi Telescoping this inequality over $t$ from $(r_t-1)q$ to $t$ with $t\leq r_t q-1$, we have:
\beq
0 &\leq&\ts \sum_{j=(r_t-1)q}^{ {  t}}  \EEE[ L_j (F^+_j - F^+_{j+1} )  + \tfrac{L}{2} \tfrac{\delta_j^2}{L_{j}} -  \tfrac{\delta_j^2}{4}  +\tfrac{2 L^2}{b} {  \sum_{i= (r_{j}-1)q }^{j-1} \|\x_{i+1}-\x_{i}\|_2^2} + \tfrac{\rho^2 D^2 }{T+1} ]  \nn\\
&\overset{\step{1}}{\leq}  & \ts \sum_{j=(r_t-1)q}^{ {  r_tq-1} } \ts  \underbrace{\ts \EEE[ L_j (F^+_j - F^+_{j+1} )  +\tfrac{L}{2} \tfrac{\delta_j^2}{L_{j}} -  \tfrac{\delta_j^2}{ 4} +\tfrac{2 L^2}{b} {  q \kappa^{q-1} \|\x_{j+1}-\x_j\|_2^2 } + \tfrac{ \rho^2 D^2 }{T+1} ] }_{ \UUU_j },
\eeq
\noi where step \step{1} uses Lemma \ref{lemma:sum:y:two} that $\sum_{j=(r_t-1)q}^t \sum_{i=(r_j-1)q}^{j-1} M_i\leq q \kappa^{q-1}\sum_{j=(r_t-1)q}^{t-1} M_j$, where $M_i=\|\x_{i+1}-\x_i\|_2^2$. We further derive the following results:
\beq
&& \ts r_t = 1,\, 0 ~\leq~ \sum_{j=0}^{q-1}  \UUU_j ;\nn\\
&&\ts r_t = 2,\, 0 ~\leq~ \sum_{j=q}^{2q-1} \UUU_j ; \nn\\
&&\ts r_t = 3,\, 0 ~\leq ~\sum_{j=2q}^{3q-1} \UUU_j;  \nn\\
&& \ts \ldots\nn\\
&& \ts r_t = s,\, 0 ~\leq~  \sum_{j=sq}^{sq-1} \UUU_j . \nn
\eeq
\noi Assume that $T=sq$, where $s\geq 0$ is an integer. Summing all these inequalities together yields:
\beq \label{eq:sum:ineq:quad:stoc:1}
0&\leq& \ts  \sum_{t=0}^{T-1} \UUU_t \nn\\
&\overset{\step{1}}{=}  &  \ts \EEE[\sum_{t=0}^{T-1} L_j (F^+_j - F^+_{j+1} )  +  \tfrac{L}{2} \tfrac{\delta_j^2}{L_{j}} -  \tfrac{\delta_j^2}{4} + \tfrac{2 L^2}{b} q \kappa^{q-1} \|\x_{t+1} - \x_t\|_2^2 + \tfrac{ \rho^2 D^2 }{T+1} ] \nn\\
&\overset{\step{2}}{\leq}  & \ts \EEE[ \tfrac{T \rho^2 D^2}{T+1}  + L_{T} \max_{t=0}^{T-1}F_t^+ +  \tfrac{L\kappa}{2}  \sum_{t=0}^{T-1} \tfrac{\delta_t^2}{L_{t+1}}     -   \tfrac{1}{4} \sum_{t=0}^{T-1} \delta_t^2 + \tfrac{2L^2}{b} q \kappa^{q-1} \tfrac{\delta_t^2}{L_t^2} ] \nn\\
&\overset{\step{3}}{\leq}  & \ts \EEE[  \rho^2 D^2  + L_{T} \max_{t=0}^{T-1}F_t^+ +  \left( \tfrac{L \kappa}{2} + \tfrac{2 L^2}{b} \tfrac{q \kappa^q}{L_0} \right) \sum_{t=0}^{T-1} \tfrac{\delta_t^2}{L_{t+1}}  - \tfrac{1}{4} \sum_{t=0}^{T-1} \delta_t^2 ] \nn\\
&\overset{\step{4}}{\leq}  & \ts \EEE[ \rho^2 D^2 + L_{T} \max_{t=0}^{T-1}F_t^+ +   \ts \left( \tfrac{L \kappa}{2} + \tfrac{2 L^2}{b} \tfrac{q \kappa^q}{L_0} \right)  \cdot \tfrac{2}{\rho^2}L_{T}  -  \tfrac{L_T^2}{4\rho^2}  + \tfrac{1}{4}  ],
\eeq
\noi where step \step{1} uses the definition of $\UUU_t$; step \step{2} uses Lemma \ref{eq:ab:sort} that $\sum_{t=0}^{T-1} L_j(F_j^+-F_{j+1}^+) \leq L_{T-1}\max_{t=0}^{T-1}F_t^+\leq L_{T}\max_{t=0}^{T-1}F_t^+$; step \step{3} uses $\tfrac{L_{t+1}}{L_t}\leq \kappa$, and $\tfrac{1}{L_t}\leq \tfrac{1}{L_0}$; step \step{4} uses
\beq
&&\ts \sum_{t=0}^{T-1} \tfrac{\delta_t^2}{L_{t+1}} ~=~ \tfrac{1}{\rho}\sum_{t=0}^{T-1} \tfrac{\delta_t^2}{ \sqrt{1 + \sum_{j=0}^t \delta_j^2}} ~\leq~ \tfrac{2}{\rho }\sqrt{ 1 + \sum_{t=0}^{T-1} \delta_t^2}~=~ \tfrac{2}{\rho^2} L_{T},\\
&&\ts \tfrac{1}{4}\sum_{t=0}^{T-1}\delta_t^2 ~=~ \tfrac{L_T^2}{4 \rho^2} - \tfrac{1}{4}.  \nn
\eeq

\noi Multiplying both sides by $\tfrac{4\rho^2}{L_T}$ yields:
\beq \label{eq:upper:Lt:FS}
\EEE[L_T] &\leq& \EEE[ 4\rho^2 \max_{t=0}^{T-1}F_t^+ + \rho^2 D^2 \cdot \tfrac{4\rho^2}{L_T}  +   8 \left( \tfrac{L \kappa}{2} + \tfrac{2 L^2}{b} \tfrac{q \kappa^q}{L_0} \right) + \tfrac{\rho^2}{L_T}  ] \nn\\
&\overset{\step{1}}{\leq}  &  \EEE[4\rho^2 \max_{t=0}^{T-1}F_t^+ + \underbrace{ 4 \rho^3 D^2  +    8 \left( \tfrac{L \kappa}{2} + \tfrac{2 L^2}{b} \tfrac{q \kappa^q}{\rho} \right)  + \rho }_{:= L'}],
\eeq
\noi where step \step{1} uses $\tfrac{1}{L_T}\leq \tfrac{1}{L_0} = \tfrac{1}{\rho}$.

The upper bound for $L_{T}$ is established in Inequality (\ref{eq:upper:Lt:FS}), but it depends on the unknown variable $\max_{t=0}^{T-1} F_t^+$.

\paragraph{Part (b).} We now prove that $(\max_{t=0}^{T-1} F_t^+)$ is always bounded above by a universal constant $\overline{F}$. Dropping the non-negative term $\tfrac{\delta_t^2}{4 L_t}$, and telescoping Inequality (\ref{eq:hhhh:VR}) over $t$ from $(r_t-1)q$ to $t$ with $t\leq r_t q-1$, we have:
\beq
0 &\leq & \ts \sum_{j=(r_t-1)q}^t \EEE[ - \gamma_j\GG(\x_j)  + \tfrac{ \rho^2 D^2}{ L_j (T+1)} + F^+_j - F^+_{j+1} + \tfrac{L}{2} \tfrac{\delta_j^2}{L_j^2} + \tfrac{2 L^2}{b} \tfrac{1}{L_j} \sum_{i=(r_j-1)q}^{j-1} \|\x_{i+1}-\x_{i}\|_2^2    ] \nn\\
&\overset{\step{1}}{\leq}& \ts \sum_{j=(r_t-1)q}^{r_t q - 1} \underbrace{\ts \EEE[  - \gamma_j\GG(\x_j)  + \tfrac{ \rho^2 D^2 }{ L_j  (T+1)} + F^+_j - F^+_{j+1} + \tfrac{L}{2} \tfrac{\delta_j^2}{L_j^2} + \tfrac{2 L^2}{b} q \kappa^{q-1} \tfrac{1}{L_j} \|\x_{j+1} - \x_j\|_2^2    }_{\KKK_j} ],\nn
\eeq
\noi where step \step{1} uses Lemma \ref{lemma:sum:y:two}(b) that $\sum_{j=(r_t-1)q}^t \tfrac{1}{L_j} \sum_{i=(r_j-1)q}^{j-1}  M_i\leq q \kappa^{q-1} \sum_{j=(r_t-1)q}^{t-1}  \tfrac{M_j}{L_j}$, where $M_j=\|\x_{j+1} - \x_j\|_2^2$. We further derive the following results:
\beq
&& \ts r_t = 1,\, 0 ~\leq~ \EEE[\sum_{j=0}^{q-1}  \KKK_j ] \nn\\
&&\ts r_t = 2,\, 0 ~\leq~ \EEE[\sum_{j=q}^{2q-1} \KKK_j ]  \nn\\
&&\ts r_t = 3,\, 0 ~\leq~ \EEE[\sum_{j=2q}^{3q-1} \KKK_j]  \nn\\
&& \ts \ldots\nn\\
&& \ts r_t = s,\, 0 ~\leq~ \EEE[ \sum_{j=sq}^{sq-1} \KKK_j ]. \nn
\eeq
\noi Assume that $T=sq$, where $s\geq 0$ is an integer. Summing all these inequalities together yields:
\beq \label{eq:sum:ineq:quad:2:stoc}
0 &\leq& \ts \EEE[\sum_{t=0}^{T-1} \KKK_t ] \nn\\
&\overset{\step{1}}{=} & \ts \EEE[ \sum_{t=0}^{T-1}  F^+_t - F^+_{t+1} - \gamma_t\GG(\x_t)  + \tfrac{ \rho^2 D^2 }{ L_t  (T+1)}  + \tfrac{L}{2} \tfrac{\delta_t^2}{L_t^2} + \tfrac{2 L^2}{b} q \kappa^{q-1} \tfrac{1}{L_t} \|\x_{t+1} - \x_t\|_2^2      ] \nn\\
&\overset{\step{2}}{\leq} & \ts \EEE[ F^+_0 - F^+_{T} - \sum_{t=0}^{T-1} \gamma_t\GG(\x_t)  +  \tfrac{\rho^2 D^2}{ L_0} \tfrac{T}{T+1}    + \big( \tfrac{L \kappa^2}{2} +  \tfrac{2 L^2}{b \rho} q \kappa^{q-1} \kappa^2 \big) \sum_{t=0}^{T-1} \tfrac{\delta_t^2}{L_{t+1}^2}    ] \nn\\
&\overset{\step{3}}{\leq} & \ts\EEE[ F^+_0 - F^+_{T} -\sum_{t=0}^{T-1} \gamma_t\GG(\x_t)  +  \rho D^2    +  \big( \tfrac{L\kappa^2}{2} +  \tfrac{2 L^2}{b \rho} q \kappa^{q+1} \big)  \tfrac{2}{\rho^2} \big(\ln(L_T) + |\ln(\rho)|\big) ] ,\nn
\eeq
\noi where step \step{1} uses the definition of $\KKK^t$; step \step{2} uses $\rho = L_0 \leq L_t$, and $\tfrac{L_{t+1}}{L_t}\leq\kappa$; step \step{3} uses $L^2_{t+1} = \rho^2 \big(1 + \sum_{i=0}^{t}\delta_i^2\big)$ and Lemma \ref{lemma:online:ppp}, leading to:
\beq
\ts \rho^2 \sum_{t=0}^{T-1} \tfrac{\delta_t^2}{ L_{t+1}^2 } ~= ~ \sum_{t=0}^{T-1} \tfrac{\delta_t^2}{ 1 + \sum_{i=0}^{t} \delta_i^2  } ~ \leq ~ \ln ( 1 + \sum_{t=0}^{T-1} \delta_t^2 )  ~ = ~ 2 \ln ( \tfrac{L_T}{\rho} )  ~ \leq ~ 2\ln (L_T) +  2|\ln (\rho)|. \nn
\eeq
We define $\omega :=\big( \tfrac{L\kappa^2}{2} +  \tfrac{2 L^2}{b \rho} q \kappa^{q+1} \big)  \tfrac{2}{\rho^2}$. For any $\lambda>0$, we derive:
\beq \label{eq:FSUM:GG}
&& \ts \EEE[ F_T^+ + \sum_{t=0}^{T-1} \gamma_t\GG(\x_t)] \nn\\
& \leq  & \ts  \EEE[ F^+_0 + \rho D^2  +  \omega |\ln(\rho)|   +    \omega \ln(L_T) ] \nn\\
&\overset{\step{1}}{\leq} & \ts  \ts \EEE[ F^+_0 +  \rho D^2  +  \omega |\ln(\rho)|    + \omega |\ln(\lambda)| +    \omega \lambda L_T ]\nn\\
&\overset{}{\leq} & \ts \EEE[ F^+_0 + \rho D^2  +  \omega |\ln(\rho)| + \omega |\ln(\lambda)| +  \omega\lambda \big( L' + 4 \rho^2 \max_{t=0}^{T-1} F_t^+ \big)]  \nn\\
&\overset{\step{2}}{\leq} & \ts \EEE[\underbrace{ \ts F^+_0 + \rho D^2  +  \omega |\ln(\rho)| + \omega |\ln(\lambda)| +  \omega\lambda L'}_{:= \zeta} + \tfrac{1}{2} \max_{t=0}^{T-1} F_t^+] ,
\eeq
\noi where step \step{1} uses $\ln(x)\leq |\ln(\lambda)|+\lambda x $, where $x = L_T$; step \step{2} uses Inequality (\ref{eq:upper:Lt:FS}) that $L_T \leq L' + 4\rho^2 \max_{t=0}^{T-1}F_t^+$; step \step{3} uses the choice $\lambda  = \tfrac{1}{8 \rho^2 \omega}$. Dropping the non-negative term $\sum_{t=0}^{T-1} \gamma_t\GG(\x_t)$ and applying Lemma \ref{lemma:A:alpha} with $\alpha=\tfrac{1}{2}$ and $A_t=F_{t}^+$, we have:
\beq
\EEE[F_{t}^+] ~ \leq ~ \max(F_{0}^+,2\zeta) ~ = ~ 2\zeta ~ := ~\overline{F}.\nn
\eeq

\paragraph{Part (c).} We have from Inequality (\ref{eq:upper:Lt:FS}):
\beq
\EEE[ L_{T+1}] ~ \leq ~\ts  L' +  4 \rho^2 \max_{t=0}^T F_t^+  ~ \leq ~ \ts L' + 4 \rho^2 \overline{F} ~ := ~ \overline{L}.\nn
\eeq

\paragraph{Part (d).} We have from Inequality (\ref{eq:FSUM:GG}):
\beq
\ts \EEE[ \sum_{t=0}^{T-1} \gamma_t\GG(\x_t)] ~ \leq ~ \zeta + \frac{1}{2} \overline{F} ~ = ~ \overline{F}.\nn
\eeq

\end{proof}

\subsection{Proof of Theorem \ref{theorem:VR}}\label{app:theorem:VR}

\begin{proof} Define $C:=\tfrac{\overline{F} \cdot \overline{L} }{\rho}$ and set $L_0=\rho$.

\noi Fix an integer $T\geq 0$. For all $t = 0,1,\ldots,T$, let $\gamma_t := \tfrac{L_0}{L_t \sqrt{T+1}}$.

\paragraph{Part (a).} Using Lemma \ref{lemma:bound:F:ncvx:VR}(d), we derive:
\beq
\ts \sum_{t=0}^{T} \GG(\x_t) ~ \leq ~ \tfrac{\overline{F}}{\gamma_{T-1}} ~ = ~ \overline{F} \cdot \tfrac{L_{T} \sqrt{T+1}}{L_0}   ~ \leq ~ C \sqrt{T+1}. \nn
\eeq
\noi This leads to:
\beq
\ts \tfrac{1}{T+1} \sum_{t=0}^{T} \GG(\x_t) ~\leq ~ \tfrac{C}{\sqrt{T+1}}.\nn
\eeq

\noi Consequently, for any $\epsilon > 0$, the algorithm finds an $\epsilon$-approximate solution in expectation (i.e., $\mathbb{E}[\min_{0 \le t \le T} \mathcal{G}(\mathbf{x}_t)] \leq \epsilon$) within at most $T = \lceil C^2 \epsilon^{-2}-1\rceil$ iterations.

\paragraph{Part (b).} Let $b$ denote the mini-batch size, and $q$ the period parameter of ALFCG-FS. Suppose the algorithm converges in $T=\lceil C^2 \epsilon^{-2}-1\rceil$ iterations. In each iteration $t$, if $t$ is a multiple of $q$ (i.e., $t \equiv 0 \pmod q$), the full-batch gradient $\nabla f(\x^t)$ is computed with a cost of $\mathcal{O}(N)$, which occurs $\lceil\tfrac{T}{q}\rceil$ times throughout the process. Otherwise, a mini-batch gradient is computed with a cost of $b$, occurring $(T-\lceil\tfrac{T}{q}\rceil)$ times. Consequently, the total complexity is:
\beq
\ts \text{Total Complexity} &=& \ts N\cdot \lceil\tfrac{T}{q}\rceil  + b  \cdot (T-\lceil\tfrac{T}{q}\rceil) \nn\\
&\leq& \ts   N \cdot \frac{T + q}{q} +  b\cdot T \nn\\
& \overset{\step{1}}{\leq } & \ts   N \cdot \frac{T + \sqrt{N}}{\sqrt{N}} + \sqrt{N} \cdot T ~=~ N + 2 \sqrt{N} \cdot T \nn\\
& \overset{\step{2}}{=} & \ts  N  +   2\sqrt{N} \cdot (\tfrac{C^2}{\epsilon^2}-1) ~= ~\OO(N + \sqrt{N}\epsilon^{-2}),\nn
\eeq
\noi where step \step{1} uses the choice that $q = b=\sqrt{N}$; step \step{2} uses $T=\tfrac{C^2}{\epsilon^2} -1$.

\end{proof}

\section{Proof of Section \ref{sect:iter:MVR1}}
\label{app:sect:iter:MVR1}

\subsection{Proof of Lemma \ref{lemma:bound:MVR1}}
\label{app:lemma:bound:MVR1}

\begin{proof} Let $L_{t} = \rho \big(1 + \sum_{i=0}^{t-1} \delta_i^2\big)^{1/2} \alpha_{t}^{-1/2}$ with $L_0=\rho$.

\noi Let $\alpha_t = \big( 1 + \sum_{i=0}^{t-1} (\beta + L_{i}^2 \|\x_{i+1}-\x_{i}\|^2) \big)^{-1/2}$.

\noi Define $s_t := \|\g_t - \nabla f(\x_t)\|$, and $\overline{s} := 3 (\sigma+G)$.

\paragraph{Part (a).} We have the following results:
\beq
F(\x) - F(\x_*) ~\overset{\step{1}}{\leq} ~ 2G\| \x - \x_*\| ~\leq~ 2G D, \nn
\eeq
\noi where step \step{1} uses the Lipschitz continuity of $f(\cdot)$ and $h(\cdot)$. We conclude that $F(\x)\leq F(\x_*) +2 GD:=\overline{F}$.

\paragraph{Part (b).} Using the update rule for the stochastic gradient estimator $\g_t=(1-\alpha_t)\g_{t-1}+\alpha_t\nabla f(\x_t;\xi_t)$, we derive the following result:
\beq
\EEE[\|\g_t\|_2^2] &{\leq}& \ts \EEE[\max \left( \|\g_{-1}\|_2^2, \|\nabla f(\x_t;\xi_t) \|_2^2 \right) ] \nn\\
&\overset{\step{1}}{\leq }& \ts \max_{i=0}^t \EEE[\|\nabla f(\x_i;\xi_i) \|_2^2] ~\overset{\step{2}}{\leq} ~2\sigma^2 + \max_{i=0}^t \EEE[2\|\nabla f(\x_i) \|_2^2]~\overset{\step{3}}{\leq}~2\sigma^2 + 2G^2,\nn
\eeq
\noi where step \step{1} uses the properties of the maximum function and the choice $\g_{-1}=\zero$; step \step{2} uses the triangle inequality and Assumption \ref{ass:regularity:expectation}(a) that $\EEE[\|\nabla f(\x;\xi) -\nabla f(\x)\|_2^2]\leq \sigma^2$; step \step{3} uses the $G$-Lipschitz continuity of $f(\cdot)$. We further derive:
\beq \label{eq:bound:on:s2}
\EEE[\|\s_t\|_2^2] ~ \leq ~ 2 \EEE[\|\g_t\|_2^2 + 2 \EEE[\|\nabla f(\x_t)\|_2^2] ~ \leq ~ 4 (\sigma^2 + G^2) + 2 G^2 ~\leq ~ \smash{\bar{s}}^2.
\eeq

\paragraph{Part (c).} We have from Inequality (\ref{eq:analysis:first}):
\beq
\EEE[\gamma_t\GG(\x_t)   + F^+_{t+1} - F^+_t +\tfrac{\delta_t^2}{4 L_t}  ] ~ \leq ~ \EEE[\tfrac{\delta_t^2 L}{2 L_t^2}    +  \tfrac{2 s_t^2}{L_t} + \gamma_t^2 L_t D^2].\nn
\eeq
Dropping the non-negative term $\gamma_t\GG(\x_t)$ and multiplying both sides by $4 L_t$ yields:
\beq
\EEE[\delta_t^2] &\leq& \EEE[  \tfrac{ 2 L \delta_t^2}{L_t}        +  8s_t^2 + 4 \gamma_t^2 L_t^2 D^2  +  4 L_t(F^+_t  - F^+_{t+1} )  ] \nn\\
& \overset{\step{1}}{\leq} & \EEE[ 2 L \delta_t D       + 8 ( 6 G^2 + 4 \sigma^2 )   + 4 \gamma_t^2 L_t^2 D^2  +  8 L_t G\|\x_t - \x_{t+1}\|  ] \nn\\
& \overset{\step{2}}{\leq} & \EEE[2 L \delta_t D +   ( 48 G^2 + 32 \sigma^2 + 4\rho^2 D^2 )  +    8 G \delta_t ] \nn\\
& \overset{}{\leq} & 2 \EEE[ \max\left( 2 L \delta_t D + 8 G \delta_t,   48 G^2 + 32 \sigma^2   + 4 \rho^2 D^2 \right) ] \nn\\
& \overset{}{\leq} & \max\left( ( 4 L D + 16 G)^2,    2 \cdot (48 G^2 + 32 \sigma^2 + 4 \rho^2 D^2) \right)  \nn\\
& \overset{}{\leq} & \left(  4 L D + 16 G  +  10 G + 8 \sigma + 3 \rho D \right)^2 := \smash{\bar{\delta}}^2, \nn
\eeq
\noi where step \step{1} uses $\delta_t\leq L_t D$, and the fact that $F(\cdot)$ is $2G$-Lipschitz continuous; step \step{2} uses $\gamma_t:=\tfrac{L_0}{2L_t} (\vartheta T^{-1/4} + T^{-1/2}) \leq \tfrac{L_0}{L_t}=\tfrac{\rho}{L_t}$, and $\delta_t=L_t \|\x_t - \x_{t+1}\|$.

\paragraph{Part (d).} Using the fact that $L_{t} = \rho \big(1 + \sum_{i=0}^{t-1} \delta_i^2\big)^{1/2} \alpha_{t}^{-1/2}$ with $L_0=\rho$ and $\alpha_t = \big( 1 + \sum_{i=0}^{t-1} (\beta + L_{i}^2 \|\x_{i+1}-\x_{i}\|^2) \big)^{-1/2}$, we derive the following inequality for all $t\geq 0$:
\beq
\tfrac{L_{t+1}^2}{L_t^2} ~ = ~  \tfrac{ \left(1 + \sum_{i=0}^{t} \delta_i^2\right)  }{  \left(1 + \sum_{i=0}^{t-1} \delta_i^2\right) } \cdot \tfrac{ \big( 1 + \sum_{i=0}^{t} (\beta + \delta_i^2) \big)^{1/2} }{ \big( 1 + \sum_{i=0}^{t-1} (\beta + \delta_{i}^2 ) \big)^{1/2} } ~ \overset{\step{1}}{\leq} ~ \big(1 + \delta_t\big)\cdot \sqrt{1 + \delta_t} ~\leq ~ \big(1+\delta_t\big)^{2}, \nn
\eeq
\noi where step \step{1} uses $\tfrac{1+a}{1+b}\leq 1 + \tfrac{a}{b}$ for all $a\geq 0$ and $b>0$. Taking the square root of both sides yields: $\tfrac{L_{t+1}}{L_t} \leq 1 + \overline{\delta} :=\kappa$.

\end{proof}

\subsection{Proof of Lemma \ref{lemma:MVR1:S}}
\label{app:lemma:MVR1:S}

\begin{proof} We let $L_{t} = \rho \big( 1 + \sum_{i=0}^{t-1} L_{i}^2 \|\x_{i+1}-\x_{i}\|^2 \big)^{1/2}\alpha_{t}^{-1/2}$, where $L_0=\rho$.

\noi We let $\alpha_t = \big( 1 + \sum_{i=0}^{t-1} (\beta + \delta_{i}^2) \big)^{-1/2}$, and $\Gamma:= 1 + \sum_{i=0}^T \delta_i^2$.

\paragraph{Bounding the term $\tfrac{\|\x_t-\x_{t+1}\|_2^2}{\alpha_{t+1}^3}$.} We derive:
\beq \label{eq:bound:xx:alpha3}
\ts && \ts \tfrac{\|\x_t-\x_{t+1}\|_2^2}{\alpha_{t+1}^3} ~=~ \ts \tfrac{\delta_t^2}{L_t^2\alpha_{t+1}^3} ~\leq ~ \tfrac{\kappa^2\delta_t^2}{L_{t+1}^2\alpha_{t+1}^3} ~\leq ~\tfrac{\kappa^2\delta_t^2 }{  \rho^2 \big( 1 + \sum_{i=0}^{t}\delta_i^2\big) \alpha_{t+1}^2}  \nn\\
&=& \ts \tfrac{\kappa^2\delta_t^2 }{  \rho^2 \big( 1 + \sum_{i=0}^{t}\delta_i^2\big) } \cdot \big( 1 + \sum_{i=0}^{t} (\beta + \delta_i^2 ) \big) ~\leq~ \tfrac{\kappa^2 \smash{\bar{\delta}}^2}{\rho^2} \left(1 + \beta (t+1) \right).
\eeq

\paragraph{Part (a).} We derive the following results:
\beq\label{eq:eez:2}
\s_{t} &:= & \g_{t} - \nabla f(\x_{t}) \nn\\
&\overset{\step{1}}{=}&  (1-\alpha_{t}) \g_{t-1} + \alpha_{t} \nabla f(\x_{t};\xi_{t})  - \nabla f(\x_{t}) \nn\\
&\overset{}{=}&  (1-\alpha_{t}) (\s_{t-1} - \nabla f(\x_{t-1}) ) + \alpha_{t} \nabla f(\x_{t};\xi_{t})  + (1-\alpha_{t})\nabla f(\x_{t-1})  - \nabla f(\x_{t}) \nn\\
&\overset{\step{2}}{=}& \underbrace{ \ts (1-\alpha_{t}) \left( \s_{t-1} + \nabla f(\x_{t-1})  - \nabla f(\x_{t}) \right) }_{\triangleq \z_{t} } + \alpha_{t} \left(\nabla f(\x_{t};\xi_{t}) - \nabla f(\x_{t})\right),~~~~~~~~~~~
\eeq
\noi where step \step{1} uses $\g_{t} = (1-\alpha_{t}) \g_{t-1} + \alpha_{t} \nabla f(\x_{t};\xi_t)$; step \step{2} uses $\s_t \triangleq \g_t - \nabla f(\x_t)$. This further leads to the following inequalities:
\beq \label{eq:z:exp:2}
\EEE[\|\z_{t}\|_{2}^2] &\triangleq &  \|  (1-\alpha_{t}) \s_{t-1} + (1 - \alpha_{t})(\nabla f(\x_{t-1})  - \nabla f(\x_{t})) \|_2^2\nn\\
&\overset{\step{1}}{\leq}& (1+\alpha_{t})(1-\alpha_{t})^2 \| \s_{t-1}\|_2^2  + (1+\tfrac{1}{\alpha_{t}})(1-\alpha_{t})^2 \|  \nabla f(\x_{t-1})  - \nabla f(\x_{t}) \|_2^2 \nn\\
&\overset{\step{2}}{\leq}& (1-\alpha_{t}) \| \s_{t-1}\|_2^2  + (1+\tfrac{1}{\alpha_{t}})(1-\alpha_{t})^2 \|  \nabla f(\x_{t-1})  - \nabla f(\x_{t}) \|_2^2 \nn\\
&\overset{\step{3}}{\leq}& (1-\alpha_{t}) \| \s_{t-1}\|_2^2  + (1+\tfrac{1}{\alpha_{t}})(1-\alpha_{t})^2 L^2 \|\x_{t-1}  - \x_{t} \|_2^2 \nn\\
&\overset{\step{4}}{\leq}& (1-\alpha_{t}) \| \s_{t-1}\|_2^2  +  \tfrac{1}{\alpha_{t}} L^2 \|\x_{t-1}  - \x_{t} \|_2^2,
\eeq
\noi where step \step{1} uses $\|\mathbf{a}+\mathbf{b}\|_2^2\leq (1+\gamma)\|\mathbf{a}\|_2^2+(1+\tfrac{1}{\gamma})\|\mathbf{b}\|_2^2$ for all $\gamma=\alpha_{t}$; step \step{2} uses $(1+\alpha_{t})(1-\alpha_{t})\leq 1$; step \step{3} uses $\|\nabla f(\x_{t})  - \nabla f(\x_{t-1})\|\leq L\|\x_t - \x_{t}\|$; step \step{4} uses the following inequalities:
\beq
(1+\tfrac{1}{\alpha_{t}})(1-\alpha_{t})^2 ~= ~\tfrac{1}{\alpha_{t}} (\alpha_{t}+1)(1-\alpha_{t})^2~ =~ \tfrac{1}{\alpha_{t}} (1-\alpha_{t}^2)(1-\alpha_{t}) ~\leq~ \tfrac{1}{\alpha_{t}}.\nn
\eeq
\noi Taking the square of Inequality (\ref{eq:eez:2}) and then taking the expectation gives:
\beq \label{eq:ss:bb}
\EEE[\|\s_{t}\|_{2}^2] &=& \EEE[\| \alpha_{t} (\nabla f(\x_{t};\xi_{t}) - \nabla f(\x_{t})) \|_2^2 + \EEE[ \|\z_{t} \|_2^2] \nn\\
&\overset{\step{1}}{\leq}& \alpha_{t}^2 \sigma^2 + \EEE[ \|\z_{t} \|_2^2] \nn\\
&\overset{\step{2}}{\leq}&\EEE[ \underbrace{\ts \alpha_{t}^2 \sigma^2 + \tfrac{L^2}{\alpha_{t}} \|\x_{t}  - \x_{t-1} \|_2^2}_{:= B_{t}} ] + (1-\alpha_{t}) \EEE[\| \s^{t}\|_2^2],
\eeq
\noi where step \step{1} uses $\EEE[\|\nabla f(\x_{t};\xi_t) - \nabla f(\x_t)\|_2^2 = \sigma^2$; step \step{2} uses Inequalities (\ref{eq:z:exp:2}).

\paragraph{Part (b).} We now bound the term $\EEE[\sum_{t=0}^T B_t]$. We have:
\beq
\ts \EEE[\sum_{t=0}^T B_t]  & = & \ts  \EEE[ \sigma^2 \sum_{t=0}^T \alpha_t^2 +  L^2 \sum_{t=0}^T \frac{1}{\alpha_t}  \|\x_{t} - \x_{t-1}\|_2^2  ] \nn\\
&\overset{\step{1}}{\leq}& \ts \sigma^2 ( 1 + \frac{\ln(1 + \beta T)}{\beta} ) + \EEE[L^2\sum_{t=0}^T \frac{1}{\alpha_t}  \|\x_{t} - \x_{t-1}\|_2^2]  \nn\\
&\overset{\step{2}}{\leq}& \ts \sigma^2 ( 1 + \frac{\ln(1 + \beta T)}{\beta} ) + \EEE[ L^2 \kappa^2 \sum_{t=1}^T \frac{\delta_{t-1}^2}{ \alpha_t L_{t}^2} ]   \nn\\
&\overset{\step{3}}{=}& \ts  \sigma^2 ( 1 + \frac{\ln(1 + \beta T)}{\beta} ) + \EEE[ \tfrac{L^2 \kappa^2}{\rho^2} \sum_{t=0}^{T-1} \frac{\delta_{t}^2}{  1 + \sum_{i=0}^{t} \delta_{i}^2 }   ] \nn\\
&\overset{\step{4}}{\leq}& \ts \sigma^2 ( 1 + \frac{\ln(1 + \beta T)}{\beta} )  +  \EEE[\tfrac{L^2 \kappa^2}{\rho^2} \ln (1 + \sum_{i=0}^{t} \delta_{i}^2 )  ]   \nn\\
&\overset{\step{5}}{\leq}& \ts  \sigma^2 ( 1 + \frac{\ln(1 + \beta T)}{\beta} )  + \tfrac{L^2 \kappa^2}{\rho^2} \ln (1 + T \smash{\bar{\delta}}^2 ) := \overline{B}, \label{eq:B:const}
\eeq
\noi where step \step{1} uses the fact that:
\beq
\ts \sum_{t=0}^T \alpha_t^2 ~= ~\sum_{t=0}^T \tfrac{1}{1 + \sum_{i=0}^{t-1} (\beta + \delta_{i}^2)} ~ \leq ~\sum_{t=0}^T \tfrac{1}{1 + t \beta }  ~\leq~ 1 + \frac{\ln(1 + \beta T)}{\beta};\nn
\eeq
\noi step \step{2} uses $\delta_t := L_t \|\x_{t+1}-\x_t\|$, $\tfrac{L_{t+1}}{L_t}\leq \kappa$, and $\delta_{-1}=0$; step \step{3} uses $L_{t+1} = \rho \big( 1 + \sum_{i=0}^{t} \delta_{i}^2 \big)^{1/2}\alpha_{t+1}^{-1/2}$; step \step{4} uses Lemma \ref{lemma:online:ppp}; step \step{5} uses $\EEE[\delta_t^2] \leq \smash{\bar{\delta}}^2$.

For all $t\geq 0$, we derive the following results from Inequality (\ref{eq:ss:bb}):
\beq
\EEE[s_t^2] \leq \EEE[ B_t + (1-\alpha_t) s_{t-1}^2] ~~\Leftrightarrow~~  \EEE[\alpha_{t} s_t^2 ]\leq  \EEE[B_t + (1-\alpha_t) ( s_{t-1}^2 - s_{t}^2) ]. \nn
\eeq
\noi Summing this inequality over $t$ from $0$ to $T$ yields:
\beq
\ts \EEE[ \sum_{t=0}^T \alpha_t s_t^2] &\leq & \ts \EEE[ \sum_{t=0}^T (1-\alpha_t) ( s_{t-1}^2 - s_{t}^2) + \sum_{t=0}^T B_t] \nn\\
&\overset{\step{1}}{\leq}& \ts \EEE[ \big( \max_{t=0}^{T-1} s_{t-1}^2 \big) +  \sum_{t=0}^T B_t  ] \nn\\
&\overset{\step{2}}{\leq}& \ts \smash{\bar{s}}^2 + \overline{B},  \nn
\eeq
\noi where step \step{1} uses Lemma \ref{eq:ab:sort}; step \step{2} uses Lemma \ref{lemma:bound:MVR1}(b) that $\EEE[\|\s_t\|_2^2] \leq \smash{\bar{s}}^2$, and Inequality (\ref{eq:B:const}). This results in
\beq
\ts \EEE[ \sum_{t=0}^T s_t^2] &\leq& \ts \big( \smash{\bar{s}}^2 + \overline{B} \big) \EEE[\tfrac{1}{\alpha_T} ] ~~\leq ~~ \ts  \big( \smash{\bar{s}}^2 + \overline{B} \big)  \EEE[\sqrt{ 1 + \sum_{i=0}^{T-1} (\beta + \delta_t^2 ) } ]\nn\\
&\leq &  \ts \big( \smash{\bar{s}}^2 + \overline{B} \big) \EEE[\big( \Gamma^{1/2} + \beta^{1/2} T^{1/2} \big)]. \nn
\eeq

\end{proof}

\subsection{Proof of Lemma \ref{lemma:lemma:MVR1:S:delta}}
\label{app:lemma:lemma:MVR1:S:delta}

\begin{proof} We focus on ALFCG-MVR1.

\noi Let $L_{t} = \rho \big(1 + \sum_{i=0}^{t-1} \delta_i^2 \big)^{1/2} \alpha_{t}^{-1/2}$ and $L_0=\rho$.

\noi Define $\gamma_t = \tfrac{L_0}{2L_t} \left( \vartheta (T+1)^{-1/4} + (T+1)^{-1/2} \right)$, where $\vartheta:=\tfrac{\beta^{1/4}}{1+\beta^{1/4}}$.

\noi Define $\Gamma := 1 + \sum_{t=0}^T\delta_t^2$, and $\alpha_t := \big(1 + \sum_{i=0}^{t-1} (\beta + \delta_i^2)\big)^{-1/2}$.

\noi For any solution $\x_t$, we define $F^+_t := F(\x_t) - F_*$.

\paragraph{Key Inequality Bounding $L_T$.} We derive:
\beq \label{eq:MVR1:step1}
L_T~:=~\ts \rho \big( 1 + \sum_{i=0}^{T-1} \delta_i^2 \big)^{1/2}\alpha_{T}^{-1/2} ~\overset{\step{1}}{\leq}~\ts \rho \Gamma^{1/2} \cdot \left( \Gamma + T \beta \right)^{1/4} ~\leq~\ts   \rho \left( \Gamma^{3/4} +  T^{1/4}\beta^{1/4} \Gamma^{1/2}\right),
\eeq
\noi where step \step{1} uses $\Gamma := 1 + \sum_{t=0}^T\delta_t^2$ and the fact that $\alpha_T^{-1/2}\leq \big(1 + \sum_{i=0}^{T-1}(\beta + \delta_i^2)\big)^{1/4} \leq \left(\Gamma + T \beta\right)^{1/4}$.

\paragraph{Key Inequality Bounding $\sum_{t=0}^T \tfrac{\delta_t^2}{L_t}$.} We derive:
\beq \label{eq:MVR1:step2}
\ts \sum_{t=0}^T \tfrac{\delta_t^2}{L_t}&\overset{\step{1}}{\leq}&\ts \kappa \sum_{t=0}^T \tfrac{\delta_t^2}{L_{t+1}}  ~=~  \tfrac{\kappa}{\rho} \sum_{t=0}^T \tfrac{\delta_t^2 \alpha_{t+1}^{1/2} }{ \big( 1 + \sum_{i=0}^{t} \delta_i^2 \big)^{1/2} }  \nn\\
& \leq & \ts \frac{\kappa}{\rho} \sum_{t=0}^T  \tfrac{\delta_t^2}{ \big( 1 + \sum_{i=0}^{t} \delta_i^2 \big)^{1/2} }  ~ \overset{\step{2}}{\leq} ~  \frac{2 \kappa }{\rho } \left( 1 + \sum_{i=0}^{T} \delta_i^2  \right)^{1/2} ~=~ \frac{2 \kappa }{\rho } \Gamma^{1/2},
\eeq
\noi where step \step{1} uses Lemma \ref{app:lemma:bound:MVR1}(d) that $\tfrac{L_{t+1}}{L_{t}} \leq \kappa$ for all $t\geq 0$; step \step{2} uses Lemma \ref{lemma:sqrt:2}.

\paragraph{Part (a).} We have from Lemma \ref{lemma:descent:ncvx}:
\beq
&& \ts \EEE[\gamma_t\GG(\x_t)   + F^+_{t+1} - F^+_t   ]  \nn\\
&\leq& \EEE[\tfrac{\delta_t^2 L}{2 L_t^2} -   \tfrac{\delta_t^2}{4 L_t}   +  \tfrac{2 s_t^2}{L_t} + D^2 L_t \gamma_t^2  ] \nn\\
&=& \EEE [\tfrac{\delta_t^2 L}{2 L_t^2} -   \tfrac{\delta_t^2}{4 L_t}   +  \tfrac{2 s_t^2}{L_t} +   D^2 L_t \tfrac{L_0^2 }{4 L_t^2}  \big( \vartheta (T+1)^{-1/4} + (T+1)^{-1/2} \big)^2  ] \nn\\
&\overset{\step{1}}{\leq} & \EEE [\tfrac{\delta_t^2 L}{2 L_t^2} -   \tfrac{\delta_t^2}{4 L_t}   +  \tfrac{2 s_t^2}{ L_t} +  \tfrac{D^2 \rho^2}{2 L_t} \big(\vartheta^{2} (T+1)^{-1/2}  + (T+1)^{-1} \big)  ], \nn
\eeq
\noi where step \step{1} uses $(a+b)^2\leq 2a^2+2b^2$ for all $a,b\geq 0$. Multiplying both sides by $4 L_t$ yields:
\beq
\ts \EEE[  \delta_t^2 +  4 L_t\gamma_t\GG(\x_t)- 8 s_t^2  ]  ~~\leq~~  \ts  \EEE[ 4 L_t( F^+_t - F^+_{t+1} ) + \tfrac{2 \delta_t^2 L }{L_t}       +  2 D^2 \rho^2 \left(\vartheta^{2} (T+1)^{-1/2} + (T+1)^{-1}\right)   ].  \nn
\eeq
Summing this inequality over $t$ from 0 to $T$ yields:
\beq \label{eq:MVR1:G:d:s}
&&\ts \EEE[\sum_{t=0}^T \delta_t^2 + 4 L_t\gamma_t\GG(\x_t)   -  8 s_t^2] \nn\\
& \leq& \ts  \EEE [ \sum_{t=0}^T 4 L_t( F^+_t -  F^+_{t+1}  ) + \tfrac{2 L \delta_t^2}{L_t}        +  2 D^2 \rho^2 \big(\vartheta^{2} (T+1)^{-1/2}+ (T+1)^{-1}\big)  ]\nn\\
&\overset{\step{1}}{\leq} & \ts \EEE[ 4  \overline{F} L_T + 2 D^2 \rho^2 \left(\vartheta^{2} (T+1)^{1/2}+1\right) +  2 L \cdot \sum_{t=0}^T \tfrac{\delta_t^2}{ L_t } ] \nn\\
&\overset{\step{2}}{\leq} & \ts \EEE[ 4  \overline{F}\rho \left( \Gamma^{3/4} +  T^{1/4}\beta^{1/4} \Gamma^{1/2}\right) + 2D^2\rho^2 \left(\vartheta^{2} (T+1)^{1/2} +  1\right) + \frac{4 L \kappa }{\rho } \Gamma^{1/2}],
\eeq
\noi where step \step{1} uses Lemma \ref{eq:ab:sort}; step \step{2} uses Inequalities (\ref{eq:MVR1:step1}) and (\ref{eq:MVR1:step2}). We further have from Inequality (\ref{eq:MVR1:G:d:s}):
\beq \label{eq:two:case:MVR1:pre}
&& \ts \EEE[ 1 + \big(\sum_{t=0}^T \delta_t^2\big) + \big(\sum_{t=0}^T 4 L_t\gamma_t\GG(\x_t) \big) ] \nn\\
&\leq & \ts \EEE[ 1 + \big(\sum_{t=0}^T 8 s_t^2 \big)  + 4 \rho \overline{F} \left( \Gamma^{3/4} +  T^{1/4}\beta^{1/4} \Gamma^{1/2}\right) + 2D^2\rho^2 \left(\vartheta^{2} (T+1)^{1/2} +  1\right) + \tfrac{4L \kappa }{\rho } \Gamma^{1/2}  ] \nn\\
&\overset{\step{1}}{\leq} & \EEE[ \ts\underbrace{\ts\big( 1 + 4 \rho  \overline{F}+ 2 D^2 \rho^2 + \tfrac{4L \kappa }{\rho } \big) }_{:=Q_0}\Gamma^{3/4} + \big(\sum_{t=0}^T 8 s_t^2 \big)  +  4  \rho\overline{F} T^{1/4}\beta^{1/4} \Gamma^{1/2} + 2D^2\rho^2 \vartheta^{2} (T+1)^{1/2}   ]  \nn\\
&\overset{\step{2}}{\leq} & \ts \EEE[ 3 \max\left( Q_0 \Gamma^{3/4} ,~ \big(\sum_{t=0}^T 8 s_t^2\big) + 2D^2\rho^2 \vartheta^{2} (T+1)^{1/2} ,~  4  \rho\overline{F} T^{1/4}\beta^{1/4} \Gamma^{1/2}  \right) ] ,
\eeq
where step \step{1} uses $\Gamma:=1 + \sum_{t=0}^T\delta_t^2 \geq 1$; step \step{2} uses $a+b+c\leq 3 \max(a,b,c)$ for all $a,b,c\geq 0$. Dropping the term $\big(\sum_{t=0}^T 4 L_t\gamma_t\GG(\x_t) \big)$, we have:
\beq \label{eq:two:case:MVR1}
\EEE[\Gamma] &\leq & \ts \EEE[ 3 \max\left( Q_0 \Gamma^{3/4} ,~ \big(\sum_{t=0}^T 8 s_t^2\big) + 2D^2\rho^2 \vartheta^{2} (T+1)^{1/2},~  4  \rho\overline{F} T^{1/4}\beta^{1/4} \Gamma^{1/2} \right) ] \nn\\
&\overset{\step{1}}{\leq} & \ts   \EEE[\max \left( (3 Q_0)^4 ,~ \big(24 \sum_{t=0}^T s_t^2\big) + 6 D^2\rho^2 \vartheta^{2} (T+1)^{1/2},~  (12  \rho\overline{F} T^{1/4}\beta^{1/4})^2 \right)] \nn\\
&\overset{\step{2}}{\leq} & \ts \EEE[\underbrace{ 81 Q_0^4}_{:=Q_1} + \big(24 \sum_{t=0}^T s_t^2\big) + \underbrace{\ts \big( 6 D^2\rho^2  + 12^2 \rho^2  \overline{F}^2\big) }_{:=C_1} \cdot \max\left( \vartheta^{2},\beta^{1/2} \right)\cdot (T+1)^{1/2}], ~~
\eeq
\noi where step \step{1} uses the fact that $x\leq a x^p$ implies $x \leq a^{1/(1-p)}$ for all $x\geq 0$ and $p\in [0,1)$; step \step{2} uses $\max(a,b,c)\leq a+b+c$.

\noi We discuss two cases for Inequality (\ref{eq:two:case:MVR1}). (i) $24 \sum_{t=0}^T s_t^2\leq \tfrac{1}{2} (1 + \sum_{t=0}^T \delta_t^2)$. We have:
\beq
\ts \EEE[1 + \sum_{t=0}^T\delta_t^2 ] & \leq & \ts \EEE[ Q_1 + \tfrac{1}{2}(1 + \sum_{t=1}^T \delta_t^2)  + C_1 \max(\vartheta^2,\beta^{1/2}) (T+1)^{1/2}  ] \nn\\
&\leq& 2 Q_1 + 2 C_1 \max(\vartheta^2,\beta^{1/2}) (T+1)^{1/2} \nn
\eeq

\noi (ii) We now discuss the case $24 \sum_{t=0}^T s_t^2 > \tfrac{1}{2} (1 + \sum_{t=0}^T \delta_t^2)$. We derive:
\beq
\ts \EEE[ 1 + \sum_{t=0}^T \delta_t^2] &\leq & \ts \EEE[ 48 \sum_{t=0}^T s_t^2] \nn\\
&\leq & 48 \big( \smash{\bar{s}}^2 + \overline{B} \big) \cdot \EEE[ \Gamma^{1/2} + \beta^{1/2} T^{1/2}  ]  \nn\\
&\leq & 96 \big( \smash{\bar{s}}^2 + \overline{B} \big) \cdot \EEE[ \max\left(   \Gamma^{1/2}  , \beta^{1/2} T^{1/2}  \right) ] \nn\\
&\leq &   \max\left(   96^2 \big(  \smash{\bar{s}}^2 + \overline{B} \big)^2 , 96 \big( \smash{\bar{s}}^2 + \overline{B} \big) \cdot \beta^{1/2} T^{1/2}  \right) \nn\\
&\leq &    \underbrace{ 96^2 \big(  \smash{\bar{s}}^2 + \overline{B} \big)^2}_{:=Q_2} + \underbrace{96 \big( \smash{\bar{s}}^2 + \overline{B} \big)}_{:=C_2} \cdot \beta^{1/2} T^{1/2}.  \nn
\eeq

\noi Combining these two cases, we obtain:
\beq
\ts \EEE[1 + \sum_{t=0}^T \delta_t^2] &\leq & \max(2Q_1, Q_2) + \max(2C_1, C_2) \cdot \max(\vartheta^2,\beta^{1/2}) \cdot (T+1)^{1/2} ~~:=~~ \Omega.~~~~ ~~~~ \nn \\
\ts \EEE[\sum_{t=0}^T s_t^2] &\leq&  \ts \max\left( \tfrac{1}{48} (1+\sum_{t=0}^T \delta_t^2), \tfrac{1}{48} ( Q_2 + C_2\beta^{1/2} T^{1/2}) \right) ~~\leq ~~\tfrac{1}{48} \Omega. \nn
\eeq

\paragraph{Part (b).} We have from Inequality (\ref{eq:two:case:MVR1:pre}):
\beq \label{eq:MVR1:sss:3}
&& \ts \EEE[ \sum_{t=0}^T 4  L_t\gamma_t\GG(\x_t) ] \nn\\
&\leq & \ts  \EEE[ \max\left( 3 Q_0 \Gamma^{3/4} ,~ 24 \big(\sum_{t=0}^T s_t^2\big) + 6 D^2\rho^2 \vartheta^{2} (T+1)^{1/2} ,~  12  \rho\overline{F} T^{1/4}\beta^{1/4} \Gamma^{1/2}  \right)  -  \Gamma ] \nn\\
&\overset{\step{1}}{\leq} & \EEE[ \Omega + \tfrac{1}{2}\Omega + \tfrac{1}{2} \Omega + 12  \rho\overline{F} T^{1/4}\beta^{1/4} \Gamma^{1/2}  - \Gamma] \nn\\
&\overset{\step{2}}{\leq} & \Omega + \tfrac{1}{2}\Omega + \tfrac{1}{2} \Omega + \tfrac{1}{4} (12  \rho\overline{F} T^{1/4}\beta^{1/4})^2 \nn\\
&\overset{\step{3}}{\leq} & \Omega + \tfrac{1}{2}\Omega + \tfrac{1}{2} \Omega + \tfrac{1}{8}\Omega  \leq 4 \Omega, \nn
\eeq
\noi where step \step{1} uses $3 Q_0 =  Q_1^{1/4} \leq (2 Q_1)^{1/4} \leq \Omega^{1/4}$, and $6 D^2\rho^2 \vartheta^{2} (T+1)^{1/2} \leq C_1 \max(\vartheta^{2} ,\beta^{1/2}) (T+1)^{1/2} \leq \tfrac{1}{2}\Omega$; step \step{2} uses $-x^2+ax\leq \tfrac{1}{4}a^2$ for all $a = 12  \rho\overline{F} T^{1/4}$ and $x = \Gamma^{1/2}$; step \step{3} uses $\tfrac{1}{4} (12  \rho\overline{F} T^{1/4}\beta^{1/4})^2 \leq \tfrac{1}{4}\cdot C_1 T^{1/2}\beta^{1/2} \leq \tfrac{1}{8}\Omega$. This leads to:
\beq
\ts \sum_{t=0}^T L_t\gamma_t\GG(\x_t)  ~\leq~  \Omega.\nn
\eeq
In particular, if we set $\beta =\nu \sigma^2$ with $\nu=\OO(1)$, and $\vartheta = \tfrac{ \beta^{1/4}}{1 + \beta^{1/4}}$, we have:
\beq
\ts \EEE[\sum_{t=0}^T L_t\gamma_t\GG(\x_t)] &\leq&   \max(2Q_1, Q_2) +  \max(2C_1, C_2) \cdot \max( \vartheta^2,\beta^{1/2} ) \cdot (T+1)^{1/2} \nn\\
&=&  \max(2Q_1, Q_2) +  \max(2C_1, C_2) \cdot \nu^{1/2} \sigma \cdot (T+1)^{1/2}.\nn
\eeq
\end{proof}

\subsection{Proof of Theorem \ref{theorem:MVR1}}
\label{app:theorem:MVR1}

\begin{proof} We let $\gamma_t = \tfrac{L_0}{2L_t} \left( \vartheta (T+1)^{-1/4} + (T+1)^{-1/2} \right)$, where $\vartheta:=\tfrac{\beta^{1/4}}{1+\beta^{1/4}}$ and $\beta = \nu \sigma^2$.

We have from Lemma \ref{lemma:lemma:MVR1:S:delta}(b):
\beq
&& \ts \EEE[\sum_{t=0}^T L_t\gamma_t\GG(\x_t)] \leq \tilde{\OO}(\nu^{1/2} \sigma T^{1/2} + 1)\nn\\
&\overset{\step{1}}{\Rightarrow} & \ts \EEE[ \tfrac{L_0}{2} \left( \vartheta (T+1)^{-1/4} + (T+1)^{-1/2} \right) \sum_{t=0}^T  \GG(\x_t) ] \leq \tilde{\OO}(\nu^{1/2} \sigma T^{1/2} + 1)\nn\\
&\Rightarrow& \ts \EEE[\tfrac{1}{T+1} \sum_{t=0}^T  \GG(\x_t)] \leq \tilde{\OO} \left( \tfrac{\nu^{1/2} \sigma T^{-1/2}+T^{-1}}{  \vartheta T^{-1/4} + T^{-1/2}  } \right) \nn\\
&\overset{\step{2}}{\Rightarrow} & \ts \EEE[\tfrac{1}{T+1} \sum_{t=0}^T  \GG(\x_t)] \leq \tilde{\OO} \left( \tfrac{\nu^{1/2} \sigma }{\sigma^{1/2}} (1+\sigma^{1/2}) T^{-1/4} + T^{-1/2} \right), \nn
\eeq
\noi where step \step{1} uses the choice of $\gamma_t$; step \step{2} uses $\tfrac{a+b}{c+d}\leq \frac{a}{c} + \tfrac{b}{d}$ for all $a,b\geq 0$ and $c,d>0$ and the choice $\vartheta \leq \tfrac{\sigma^{1/2}}{1+\sigma^{1/2}}$.

\end{proof}

\section{Proof for Section \ref{sect:iter:MVR2} }\label{app:sect:iter:MVR2}

\subsection{Proof of Lemma \ref{lemma:alpha:2:3}}
\label{app:lemma:alpha:2:3}

\begin{proof} Assume $u_i\ge 0$, and $p=\frac{2}{3}$. Define $\hat\alpha_t := \Big(\frac{1+\max_{i=0}^t u_i}{1+\sum_{i=0}^t u_i}\Big)^p$, and $\alpha_t := \min_{0\le k\le t}\hat\alpha_k$.

\noi Define $S_t := \sum_{i=0}^t u_i$, $M_t := \max_{i=0}^t u_i$, and $r_t := \tfrac{1+S_t}{1+M_t}\ge 1$. Then $\hat\alpha_t=\Big(\tfrac{1+M_t}{1+S_t}\Big)^p=r_t^{-p}$, leading to $\tfrac{1}{\hat\alpha_t}=r_t^{p}$. Since $\alpha_t=\min_{i=0}^t\hat\alpha_i$, we have $\frac{1}{\alpha_t}=\max_{i=0}^t\tfrac{1}{\hat\alpha_i}
=\max_{i=0}^t r_i^p$.

\paragraph{The lower bound of $\frac{1}{\alpha_{t+1}}-\frac{1}{\alpha_t}$.} Because $\alpha_{t+1}=\min(\alpha_t,\hat\alpha_{t+1})\le \alpha_t$, we get $0\le \frac{1}{\alpha_{t+1}}-\frac{1}{\alpha_t}$.

\paragraph{Reduce the upper bound to a bound on $r_t$.} We have:
\beq \label{eq:rraa}
\tfrac{1}{\alpha_{t+1}}-\tfrac{1}{\alpha_t}  ~\overset{\step{1}}{=} ~\max\Big(0, r_{t+1}^p-\tfrac{1}{\alpha_t}\Big) ~\overset{\step{2}}{\le}~\max\big(0, r_{t+1}^p-r_t^p\big),
\eeq
\noi where step \step{1} uses $\frac{1}{\alpha_{t+1}}=\max \big(\frac{1}{\alpha_t},r_{t+1}^p\big)$; step \step{2} uses $\frac{1}{\alpha_t}=\max_{k\le t} r_k^p \ge r_t^p$. Inequality (\ref{eq:rraa}) implies that if $r_{t+1}\le r_t$, Inequality (\ref{eq:rraa}) becomes $\frac{1}{\alpha_{t+1}}-\frac{1}{\alpha_t}\leq 0$, and the claim holds. Hence it suffices to prove $r_{t+1}^p-r_t^p\le p$ for all $t$ and $r_{t+1}>r_t$.

\paragraph{Show that $0\leq r_{t+1}-r_t\leq 1$ whenever $r_{t+1}>r_t$.} Consider two cases.

$\quad$ \textit{Case A: $u_{t+1}\le M_t$}. Then $M_{t+1}=M_t$, $S_{t+1}=S_t+u_{t+1}$, and $r_{t+1}-r_t=\frac{1+S_t+u_{t+1}}{1+M_t}-\frac{1+S_t}{1+M_t}
=\frac{u_{t+1}}{1+M_t}<1$.

$\quad$ \textit{Case B: $u_{t+1}>M_t$.} Then $M_{t+1}=u_{t+1}$, $S_{t+1}=S_t+u_{t+1}$, and $r_{t+1}= \frac{1+S_t+u_{t+1}}{1+u_{t+1}}=1+\frac{S_t}{1+u_{t+1}}$, $r_t=\frac{1+S_t}{1+M_t}=1+\frac{S_t-M_t}{1+M_t}$. Since $u_{t+1}>M_t$ implies $\frac{1}{1+u_{t+1}}<\frac{1}{1+M_t}$, we have $r_{t+1}-r_t
=\frac{S_t}{1+u_{t+1}}-\frac{S_t-M_t}{1+M_t}
\le \frac{S_t}{1+M_t}-\frac{S_t-M_t}{1+M_t}
=\frac{M_t}{1+M_t}<1$. Thus, whenever $r_{t+1}>r_t$, we indeed have $0<r_{t+1}-r_t<1$.

\paragraph{Apply the mean value theorem to $x^p$.} Let $\phi(x)=x^p$ with $p\in(0,1]$. By the mean value theorem, there exists $\xi\in(r_t,r_{t+1})$ such that $r_{t+1}^p-r_t^p=\phi'(\xi)(r_{t+1}-r_t)=p \xi^{p-1}(r_{t+1}-r_t)$. Because $r_t\ge 1$ and hence $\xi\ge 1$, and since $p-1\le 0$, we have $\xi^{p-1}\le 1$. Therefore, $r_{t+1}^p-r_t^p \le p,(r_{t+1}-r_t) \le p\cdot 1 = p$. This yields $r_{t+1}^p-r_t^p\le p$, and consequently $\frac{1}{\alpha_{t+1}}-\frac{1}{\alpha_t}\le p$.

Combining the lower and upper bounds, we have $0\le \frac{1}{\alpha_{t+1}}-\frac{1}{\alpha_t}\le p=\frac{2}{3}$.

\end{proof}

\subsection{Proof of Lemma \ref{lemma:bound:MVR2}}
\label{app:lemma:bound:MVR2}

\begin{proof} Define $s_t^2 := \|\g_t - \nabla f(\x_t)\|_2^2$.

\paragraph{Part (a).} We have the following results:
\beq
F(\x) - F(\x_*) ~\overset{\step{1}}{\leq} ~ 2 G\| \x - \x_*\| ~\leq~ 2G D, \nn
\eeq
\noi where step \step{1} uses the fact that $F(\cdot)$ is $2G$-Lipschitz continuous. We conclude that $F(\x)\leq F(\x_*)+2GD:=\overline{F}$.

\paragraph{Part (b).} We have from Inequality (\ref{eq:analysis:first}):
\beq
\EEE[\gamma_t\GG(\x_t)   + F^+_{t+1} - F^+_t +\tfrac{\delta_t^2}{4 L_t}  ]  ~\leq~  \EEE[\tfrac{\delta_t^2 L}{2 L_t^2}    +  \tfrac{2 s_t^2}{L_t} + \gamma_t^2 L_t D^2].\nn
\eeq
Dropping the non-negative term $\gamma_t\GG(\x_t)$ and multiplying both sides by $4 L_t$ yields:
\beq
\EEE[\delta_t^2] &\leq& \EEE[  \tfrac{ 2 L \delta_t^2}{L_t}        +  8s_t^2 + 4 \gamma_t^2 L_t^2 D^2  +  4 L_t(F^+_t  - F^+_{t+1} )  ] \nn\\
& \overset{\step{1}}{\leq} & \EEE[ 2 L \delta_t D       + 8 \smash{\bar{s}}^2   + 4 \gamma_t^2 L_t^2 D^2  +  4 L_t C\|\x_t - \x_{t+1}\|  ] \nn\\
& \overset{\step{2}}{\leq} & \EEE[2 L \delta_t D +  8 \smash{\bar{s}}^2 + 4 \tfrac{L_0^2}{L_t^2}  L_t^2 D^2  +    4 L_t C\|\x_t - \x_{t+1}\| ] \nn\\
& \overset{\step{3}}{\leq} & \EEE[2 L \delta_t D +  8 \smash{\bar{s}}^2 + 4 \rho^2 D^2  +    4 C \delta_t ] \nn\\
& \overset{}{\leq} & 2 \EEE[ \max\left( 2 L \delta_t D + 4 C \delta_t,  8 \smash{\bar{s}}^2  + 4 \rho^2 D^2 \right) ] \nn\\
& \overset{}{\leq} & \max\left( ( 4 L D + 8 C)^2,    16 \smash{\bar{s}}^2 + 8 \rho^2 D^2 \right)\nn\\
& \overset{}{\leq} & 64 \big( L D + C + \overline{s} + \rho D \big)^2 := \smash{\bar{\delta}}^2\nn\\
\eeq
\noi where step \step{1} uses $\delta_t\leq L_t D$, Assumption \ref{ass:bound:s} that $\EEE[\|\s_t\|_2^2]\leq \smash{\bar{s}}^2$, and the fact that $F(\cdot)$ is $2G$-Lipschitz continuous; step \step{2} uses $\gamma_t:=\tfrac{L_0}{2L_t} (\vartheta T^{-1/3} + T^{-1/2}) \leq \tfrac{L_0}{L_t}$; step \step{3} uses $\delta_t = L_t \|\x_t - \x_{t+1}\|$.

\paragraph{Part (c).} For all $t\geq 0$, we derive:
\beq
\tfrac{L_{t+1}^2}{L_t^2} ~=~  \tfrac{ (L_0^2 + \rho \sum_{i=0}^{t-1} \delta_i^2) + \rho \delta_t^2   }{  (L_0^2 + \rho \sum_{i=0}^{t-1} \delta_i^2 ) } \cdot \tfrac{  \alpha_{t+1}^{-1/2} }{  \alpha_{t}^{-1/2}  } ~ \leq ~  \left( 1 +  \tfrac{\rho\smash{\bar{\delta}}^2}{L_0^2} \right) \cdot \sqrt{ \tfrac{  \alpha_{t} }{  \alpha_{t+1}  } } ~ \overset{\step{1}}{ \leq } ~  \left( 1 + \tfrac{\smash{\bar{\delta}}^2}{\rho} \right) \cdot \sqrt{  1 + \tfrac{2}{3} }~\leq ~\left( 1 + \tfrac{\smash{\bar{\delta}}^2}{\rho} \right) \cdot 2, \nn
\eeq
\noi where step \step{1} uses Lemma \ref{lemma:alpha:2:3} that $\frac{1}{\alpha_{t+1}} - \frac{1}{\alpha_{t}} \leq p = \tfrac{2}{3}$, leading to $\frac{\alpha_{t}}{\alpha_{t+1}} \leq 1 + \tfrac{2}{3}$.

Taking the square root of both sides yields: $\tfrac{L_{t+1}}{L_t} \leq \sqrt{ 2(1 + \smash{\bar{\delta}}^2/\rho)}:=\kappa$.

\end{proof}

\subsection{Proof of Lemma \ref{lemma:MVR2:S}}
\label{app:lemma:MVR2:S}

\begin{proof} Define $\s_t := \g_t - \nabla f(\x_t)$ and $\boldsymbol{\epsilon}_t = \nabla f(\mathbf{x}_t; \xi_t) - \nabla f(\mathbf{x}_t)$.

\noi Define $\dot{B}:= 18  (1+\beta + \smash{\bar{\delta}}^2)$, and $\ddot{B} := 40  L^2 \kappa^2 \rho^{-1} \ln(1 +  \smash{\bar{\delta}}^2 T)$.

\paragraph{Part (a).} We derive the following equalities:
\beq \label{eq:eez}
\s_t &:= & \g_t - \nabla f(\x_t) \nn\\
&\overset{\step{1}}{=}&  (1-\alpha_t) (\g_{t-1}  -  \nabla f(\x_{t-1};\xi_t) ) +  \boldsymbol{\epsilon}_t \nn\\
&\overset{\step{2}}{=}& (1-\alpha_t) (\g_{t-1} - \nabla f(\x_{t-1})) + (1-\alpha_t) ( \nabla f(\x_{t-1})   -   \nabla f(\x_{t-1};\xi_t) ) +  \boldsymbol{\epsilon}_t    \nn\\
&\overset{\step{3}}{=}&  (1-\alpha_t)\s_{t-1} + \underbrace{\ts (1-\alpha_t) ( \nabla f(\x_{t-1})   -   \nabla f(\x_{t-1};\xi_t) ) +  \boldsymbol{\epsilon}_t }_{:= \z^t},
\eeq
\noi where step \step{1} uses the update rule $\g_t=(1-\alpha_t) (\g_{t-1} - \nabla f(\x_{t-1};\xi_t)) +  \nabla f(\x_t;\xi_t)$; step \step{2} uses $- (1-\alpha_t)\nabla f(\x_{t-1})+(1-\alpha_t)\nabla f(\x_{t-1})=0$; step \step{3} uses $\s_t := \g_t - \nabla f(\x_t)$. This further leads to the following inequalities:
\beq \label{eq:z:exp:1}
&& \EEE[ \| \z^t \|_2^2] \nn\\
&\overset{\step{1}}{=}&  \EEE[ \|  (1-\alpha_t) \left( \nabla f(\x_{t-1})   -   \nabla f(\x_{t-1};\xi_t) \right) + \boldsymbol{\epsilon}_t \|_2^2]\nn\\
&\overset{}{=}& \EEE[ \| \alpha_t \boldsymbol{\epsilon}_t + (1-\alpha_t) \left(  \boldsymbol{\epsilon}_t +  \nabla f(\x_{t-1})   -   \nabla f(\x_{t-1};\xi_t)  \right) \|_2^2]\nn\\
&\overset{\step{2}}{\leq}& 2 \alpha_t^2 \EEE[ \| \boldsymbol{\epsilon}_t \|_2^2 ]  + 2 (1-\alpha_t)^2 \EEE[ \| \boldsymbol{\epsilon}_t  + \nabla f(\x_{t-1})   -  \nabla f(\x_{t-1};\xi_t) \|_2^2]\nn\\
&\overset{\step{3}}{\leq}& 2 \alpha_t^2 \sigma^2  + 4 \EEE[ \| \nabla f(\x_t;\xi_t) -  \nabla f(\x_{t-1};\xi_t) \|_2^2 ] + 4 \EEE[ \|\nabla f(\x_t) - \nabla f(\x_{t-1})    \|_2^2]\nn\\
&\overset{\step{4}}{\leq}& 2 \alpha_t^2 \sigma^2  + (4+4) L^2 \EEE[ \| \x_t -  \x_{t-1} \|_2^2 ],
\eeq
\noi where step \step{1} uses the definition of $\z^t$; step \step{2} uses $\|\a+\b\|_2^2\leq 2\|\a\|_2^2+2\|\b\|_2^2$ for all $\a,\b\in\Rn^n$; step \step{3} uses $\EEE[ \| \nabla f(\x_t;\xi_t) -  \nabla f(\x_t)\|_2^2 ]\leq \sigma^2$, $(1-\alpha_t)^2\leq 1$, and $\|\a+\b\|_2^2\leq 2\|\a\|_2^2+2\|\b\|_2^2$ for all $\a,\b\in\Rn^n$; step \step{4} uses $\|\nabla f(\x_{t-1};\xi_t) - \nabla f(\x_t;\xi_t)\| \leq L\|\x_t-\x_{t-1}\|$, and $\|\nabla f(\x_{t-1}) - \nabla f(\x_t)\|\leq L\|\x_t-\x_{t-1}\|$.

Taking the square of Equality (\ref{eq:eez}) and then taking the expectation gives:
\beq \label{eq:ssb}
\EEE[\|\s_t\|_2^2] &=& \EEE[\| (1-\alpha_t) \s_{t-1}\|_2^2 + \EEE[ \| \z^t \|_2^2] \nn\\
&\overset{\step{1}}{\leq}& (1-\alpha_t) \EEE[\|\s_{t-1}\|_2^2 + \underbrace{2 \alpha_t^2 \sigma^2   + 8 L \EEE[ \| \x_t -  \x_{t-1} \|_2^2 ]}_{:= B_t} ,
\eeq
\noi where step \step{1} uses $(1-\alpha_t)^2 \leq 1-\alpha_t$ for all $\alpha_t\in (0,1)$, and Inequality (\ref{eq:z:exp:1}).

\paragraph{Part (b).} For all $t\geq 0$, we derive from Inequality (\ref{eq:ssb}):
\beq
s_{t}^2 \leq (1 - \alpha_{t}) s_{t-1}^2 + B_{t} &\Rightarrow& \tfrac{1}{\alpha_{t}} s_{t}^2 \leq \tfrac{1}{\alpha_{t}} s_{t-1}^2 - s_{t-1}^2 + \tfrac{1}{\alpha_{t}} B_{t} \nn\\
&\Rightarrow& s_{t-1}^2  \leq \tfrac{1}{\alpha_{t}} (s_{t-1}^2 -s_{t}^2)  + \tfrac{1}{\alpha_{t}} B_{t} \nn\\
&\Rightarrow& s_{t}^2 \leq (\tfrac{1}{\alpha_{t}}  - 1 )(s_{t-1}^2 -s_{t}^2)  + \tfrac{B_{t}}{\alpha_{t}} .\label{eq:key:1:bbbb}
\eeq

\noi We derive the following results from Inequality (\ref{eq:key:1:bbbb}):
\beq
\ts \EEE[\sum_{t=0}^T s_{t}^2 ]&\leq &  \ts \EEE[ \left( \sum_{t=0}^T \frac{B_{t}}{\alpha_{t}}  \right) + \left( \sum_{t=0}^T  (\frac{1}{\alpha_{t}}  - 1 )(s_{t-1}^2 -s_{t}^2) \right)] \nn\\
&\overset{\step{1}}{=}&  \ts \EEE[\left( \sum_{t=0}^T \frac{B_{t}}{\alpha_{t}}  \right) +  \left(\tfrac{1}{\alpha_0} - 1\right) s_{-1}^2  +   \left( 1 - \frac{1}{\alpha_{T+1}} \right) s_T^2  + \sum_{t=0}^T \left( \frac{1}{\alpha_{t+1}} - \frac{1}{\alpha_{t}}\right) \cdot s_{t}^2 ] \nn\\
&\overset{\step{2}}{\leq}&  \ts\EEE[ \left( \sum_{t=0}^T \frac{B_{t}}{\alpha_{t}}  \right)  + p \sum_{t=0}^T s_{t}^2 ]    \nn\\
&\leq & \ts \EEE[\tfrac{1}{1-p} \sum_{t=0}^T \frac{B_{t}}{\alpha_{t}}  ~=~ 3 \sum_{t=1}^T \frac{ B_{t} }{\alpha_{t}} ] ,   \nn
\eeq
\noi where step \step{1} uses the following key identity:
\beq \label{eq:core:id}
\ts \sum_{t=0}^T \left(\frac{1}{\alpha_t} - 1\right) \left( s_{t-1}^2 - s_{t}^2 \right)  - \sum_{t=0}^T \left( \frac{1}{\alpha_{t+1}} - \frac{1}{\alpha_{t}}\right)  s_{t}^2  ~ =   ~  \left( 1 - \tfrac{1}{\alpha_{T+1}} \right) s_T^2 + (\tfrac{1}{a_0}-1)s_{-1}^2;\nn
\eeq
\noi step \step{2} uses $\alpha_1=1$ and $\alpha_t\in(0,1]$; step \step{3} uses $\frac{1}{\alpha_{t+1}} - \frac{1}{\alpha_t}\leq p$.

\noi We define $\Gamma := 1 + \sum_{t=0}^{T}\delta_t^2$. Noticing $B_t:= 2 \alpha_t^2\sigma^2 + 8 L^2 \|\x_t-\x_{t-1}\|_2^2$. We have:
\beq
&&\ts \EEE[3 \sum_{t=0}^T \frac{B_t}{\alpha_t}] \nn\\
& = & \ts \EEE[ 6 \sigma^2 \sum_{t=1}^T \alpha_t + 24 L^2 \sum_{t=0}^T \frac{1}{\alpha_t}  \|\x_{t} - \x_{t-1}\|_2^2  ] \nn\\
&\overset{\step{1}}{\leq}& \ts \EEE[ 6 \sigma^2 (1+\beta + \smash{\bar{\delta}}^2)\sum_{t=0}^T \tfrac{1}{ \left( 1 + t\beta \right)^{2/3} } + 24 L^2 \sum_{t=0}^T \frac{ \delta_{t-1}^2   }{ \alpha_t L_{t-1}^2   }  ] \nn\\
&\overset{\step{2}}{\leq}& \ts \EEE[ 6 \sigma^2 (1+\beta + \smash{\bar{\delta}}^2)  (1+ \frac{3}{\beta}+ 3  \beta^{-2/3} T^{1/3} )+ 24 L^2 \kappa^2 \sum_{t=0}^{T} \frac{ \delta_{t-1}^2  }{ \alpha_{t} L_{t}^2   }  ] \nn\\
&\overset{\step{3}}{\leq}& \ts \EEE[ 18  (1+\beta + \smash{\bar{\delta}}^2) \cdot \sigma^2 (1+ \frac{1}{\beta}+   \beta^{-2/3} T^{1/3} )+ 24 L^2 \kappa^2 \tfrac{1}{\rho } \sum_{t=-1}^{T-1} \frac{ \delta_{t}^2 \alpha_t^{1/2} }{ \alpha_{t+1}  (1 + \sum_{i=0}^t \delta_i^2 ) } ]  \nn\\
&\overset{\step{4}}{\leq}& \ts \EEE[ \dot{B} \sigma^2 (1+ \frac{1}{\beta}+   \beta^{-2/3} T^{1/3} )+ 24 L^2 \kappa^2 \tfrac{5}{3} \tfrac{1}{\rho}\sum_{t=0}^{T} \frac{ \delta_{t}^2  }{ \alpha_{t}^{1/2}  (1 + \sum_{i=0}^t \delta_i^2 ) }  ] \nn\\
&\overset{\step{5}}{\leq}& \ts  \EEE[ \dot{B} \sigma^2 (1+ \frac{1}{\beta}+   \beta^{-2/3} T^{1/3} )+ 24 L^2 \kappa^2 \tfrac{5}{3} \rho^{-1}    \ln (1 + \sum_{t=0}^T \delta_t^2)  \alpha_T^{-1/2} ]\nn\\
&\overset{\step{6}}{\leq}& \ts \EEE[ \dot{B} \sigma^2 (1+ \frac{1}{\beta}+  \beta^{-2/3} T^{1/3} )+ 40  L^2 \kappa^2 \rho^{-1} \ln(1+\smash{\bar{\delta}}^2 T) \cdot \left( \Gamma + T \beta  \right)^{1/3}] \nn\\
&\overset{}{\leq }& \ts \EEE[ \dot{B} \sigma^2 (1+ \frac{1}{\beta}+  \beta^{-2/3} T^{1/3} )+ \ddot{B} \big(\Gamma^{1/3} + T^{1/3} \beta^{1/3}\big)], \nn
\eeq
\noi where step \step{1} uses the definition of $\alpha_t$ and $\delta_t$; step \step{2} uses $\sum_{t=0}^{T} \frac{1}{(1+t\beta)^{2/3}} \le 1 + \frac{3}{\beta} \big( (1+T\beta)^{\frac{1}{3}} - 1 \big) \le 1 + \frac{3}{\beta} \big(  1+T^{1/3}\beta^{1/3} \big)$ with $p=2/3$; step \step{2} uses $L_{t+1}/L_t\leq \kappa$; step \step{3} uses the definition of $L_t$; step \step{4} uses Lemma \ref{lemma:alpha:2:3} that $\tfrac{\alpha_{t+1}}{\alpha_t}\leq 1 + p =\tfrac{5}{3}$; step \step{5} uses Lemma \ref{lemma:online:ppp}; step \step{6} uses the definition of $\alpha_t$.

\end{proof}

\subsection{Proof of Lemma \ref{lemma:lemma:MVR2:S:delta}}
\label{app:lemma:lemma:MVR2:S:delta}

\begin{proof} We let $L_{t} = \rho \left(1 + \sum_{i=0}^{t-1} \delta_i^2\right)^{1/2} \alpha_{t}^{-1/4}$ and $L_0=\rho$.

\noi We let $\gamma_t = \tfrac{L_0}{2 L_t} \left( \vartheta (T+1)^{-1/3} + (T+1)^{-1/2} \right) \in (0,1]$.

\noi We define $\Gamma := 1 + \sum_{t=0}^T\delta_t^2$.

\paragraph{Key Inequality Bounding $L_T$.} We derive:
\beq \label{eq:step1:MVR2}
L_T~\overset{}{:=}~\ts \rho \big( 1 + \sum_{i=0}^{T-1} \delta_i^2 \big)^{1/2}\alpha_{T}^{-1/4} ~\overset{\step{1}}{\leq}~\ts \rho \Gamma^{1/2} \cdot \left( \Gamma + T \beta \right)^{1/6} ~\leq~\ts   \rho \left( \Gamma^{2/3} +  T^{1/6}\beta^{1/6} \Gamma^{1/2}\right),
\eeq
\noi where step \step{1} uses $\Gamma := 1 + \sum_{t=0}^T\delta_t^2$ and $\alpha_T^{-1/4}\leq \big(1 + \sum_{i=0}^{T-1}(\beta + \delta_i^2)\big)^{1/6} \leq \left(\Gamma + T \beta\right)$.

\paragraph{Key Inequality Bounding $\sum_{t=0}^T \tfrac{\delta_t^2}{L_t}$.} We derive:
\beq \label{eq:step2:MVR2}
\ts \sum_{t=0}^T \tfrac{\delta_t^2}{L_t}&\overset{\step{1}}{\leq}&\ts \kappa \sum_{t=0}^T \tfrac{\delta_t^2}{L_{t+1}}  ~=~  \tfrac{\kappa}{\rho} \sum_{t=0}^T \tfrac{\delta_t^2 \alpha_{t+1}^{1/4} }{ \big( 1 + \sum_{i=0}^{t} \delta_i^2 \big)^{1/2} }  \nn\\
& \leq & \ts \frac{\kappa}{\rho} \sum_{t=0}^T  \tfrac{\delta_t^2}{ \big( 1 + \sum_{i=0}^{t} \delta_i^2 \big)^{1/2} }  ~ \overset{\step{2}}{\leq} ~  \frac{2 \kappa }{\rho } \left( 1 + \sum_{i=0}^{T} \delta_i^2  \right)^{1/2} ~=~ \frac{2 \kappa }{\rho } \Gamma^{1/2},
\eeq
\noi where step \step{1} uses Lemma \ref{lemma:bound:MVR2}(c) that $\tfrac{L_{t+1}}{L_{t}} \leq \kappa$ for all $t\geq 0$; step \step{2} uses Lemma \ref{lemma:sqrt:2}.

\paragraph{Part (a).} We have from Lemma \ref{lemma:descent:ncvx}:
\beq
&& \ts \EEE[\gamma_t\GG(\x_t)   + F^+_{t+1} - F^+_t   ] \nn\\
&\leq& \EEE[\tfrac{\delta_t^2 L}{2 L_t^2} -   \tfrac{\delta_t^2}{4 L_t}   +  \tfrac{2 s_t^2}{L_t} + D^2 L_t \gamma_t^2  ] \nn\\
&=& \EEE \left[\tfrac{\delta_t^2 L}{2 L_t^2} -   \tfrac{\delta_t^2}{4 L_t}   +  \tfrac{2 s_t^2}{L_t} +   D^2 L_t \tfrac{L_0^2 }{4 L_t^2}  \left( \vartheta (T+1)^{-1/3} + (T+1)^{-1/2} \right)^2  \right] \nn\\
&\overset{\step{1}}{\leq}& \EEE\left[\tfrac{\delta_t^2 L}{2 L_t^2} -   \tfrac{\delta_t^2}{4 L_t}   +  \tfrac{2 s_t^2}{ L_t} +  \tfrac{D^2 \rho^2 }{2 L_t} \left(\vartheta^{2} (T+1)^{-2/3}  + (T+1)^{-1} \right)  \right], \nn
\eeq
\noi where step \step{1} uses $(a+b)^2\leq 2a^2+2b^2$. Multiplying both sides by $4 L_t$ yields:
\beq
\ts \EEE[\delta_t^2 + 4 \gamma_t\GG(\x_t)  - 8 s_t^2 ] &\leq& \ts  \EEE[ 4 L_t( F^+_t - F^+_{t+1} ) + \tfrac{2 \delta_t^2 L }{L_t}       +  2 D^2 \rho^2 \left(\vartheta^{2} (T+1)^{-2/3} + (T+1)^{-1}\right)   ].  \nn
\eeq
Summing this inequality over $t$ from 0 to $T$ yields:
\beq  \label{eq:MVR2:G:d:s}
&& \ts \EEE[ \sum_{t=0}^T \delta_t^2 + 4 \gamma_t L_t \GG(\x_t) - 8 s_t^2 ]  \nn\\
& \leq& \ts \EEE [  \sum_{t=0}^T  4 L_t( F^+_t -  F^+_{t+1}  ) + \tfrac{2 L \delta_t^2}{L_t}        +  2 D^2L_0^2 \left(\vartheta^{2} (T+1)^{-2/3}+ (T+1)^{-1}\right)    ]\nn\\
&\overset{\step{1}}{\leq}& \ts \EEE\left[  4  \overline{F} L_T + 4 D^2 \rho \left(\vartheta^{2} (T+1)^{1/3}+1\right) +  2 L \cdot \sum_{t=0}^T \tfrac{\delta_t^2}{ L_t } \right] \nn\\
&\overset{\step{2}}{\leq}& \ts \EEE\left[  4  \overline{F}\rho \left( \Gamma^{2/3} +  T^{1/6}\beta^{1/6} \Gamma^{1/2}\right) + 4D^2 \rho^2 \left(\vartheta^{2} (T+1)^{1/3} +  1\right) + 2 L \cdot \frac{2 \kappa }{\rho } \Gamma^{1/2} \right], \nn
\eeq
\noi where step \step{1} uses Lemma \ref{eq:ab:sort}; step \step{2} uses Inequalities (\ref{eq:step1:MVR2},\ref{eq:step2:MVR2}). We further derive from Inequality (\ref{eq:MVR2:G:d:s}):
\beq  \label{eq:two:case:MVR2:pre}
&& \ts \EEE [  1 + \sum_{t=0}^T \delta_t^2 + \sum_{t=0}^T 4 L_t \gamma_t \GG(\x_t) ]\nn\\
&{\leq}&\ts \EEE [ 1 + \big(8\sum_{t=0}^T s_t^2 \big) + 4  \overline{F}\rho \left( \Gamma^{2/3} +  T^{1/6}\beta^{1/6} \Gamma^{1/2}\right) + 4D^2 \rho^2 \left(\vartheta^{2} (T+1)^{1/3} +  1\right) + 2 L \cdot \frac{2 \kappa }{\rho } \Gamma^{1/2}] \nn\\
&\overset{\step{1}}{\leq}& \EEE [ \ts \underbrace{ \ts \big( 1 + 4  \overline{F}\rho + 4 D^2 \rho^2 + \tfrac{4L\kappa}{\rho} \big)}_{:=P_0} \Gamma^{2/3} + \big(8 \sum_{t=0}^T s_t^2 \big) + 4  \overline{F}\rho T^{1/6}\beta^{1/6} \Gamma^{1/2} + 4D^2 \rho^2 \vartheta^{2} (T+1)^{1/3}  ] \nn\\
&\overset{\step{2}}{\leq}& \ts \EEE [3 \max \left( P_0 \Gamma^{2/3} ,~ \big(8 \sum_{t=0}^T s_t^2 \big) + 4D^2 \rho^2 \vartheta^{2} (T+1)^{1/3} ,~4  \overline{F}\rho T^{1/6}\beta^{1/6} \Gamma^{1/2} \right)],
\eeq
\noi where step \step{1} uses $\Gamma := 1+\sum_{t=0}^T \delta_t^2\geq 1$; step \step{2} uses $a+b+c\leq 3 \max(a,b,c)$ for all $a,b,c\geq 0$.

Dropping the term $\sum_{t=0}^T 4 L_t \gamma_t \GG(\x_t)$ yields:
\beq \label{eq:two:case:MVR2}
\EEE [\Gamma] &\leq &  \ts \EEE [ 3 \max \left( P_0 \Gamma^{2/3} ,~ \big(8 \sum_{t=0}^T s_t^2 \big) + 4D^2 \rho^2 \vartheta^{2} (T+1)^{1/3} ,~4  \overline{F}\rho T^{1/6}\beta^{1/6} \Gamma^{1/2} \right)] \nn\\
&\overset{\step{1}}{\leq}&  \ts \EEE [ \max \left( (3 P_0)^3,~ \big(24 \sum_{t=0}^T s_t^2 \big) + 12 D^2 \rho^2 \vartheta^{2} (T+1)^{1/3} ,~(12  \overline{F}\rho T^{1/6}\beta^{1/6} )^2 \right)]\nn\\
&\overset{\step{2}}{\leq}&  \ts \EEE [ \underbrace{27 P_0^3}_{:=Z_1} + \big(24 \sum_{t=0}^T s_t^2 \big) + \underbrace{ \big ( 12 D^2 \rho^2  + 12^2  \overline{F}^2\rho^2   \big) }_{:=Y_1} \cdot \max(\vartheta^{2},\beta^{1/3}) (T+1)^{1/3}],
\eeq
\noi where step \step{1} uses the fact that $x\leq ax^p$ implies $x\leq a^{1/(1-p)}$ for all $x\geq 0$ and $p\in[0,1)$; step \step{2} uses $\max(a,b,c)\leq a+b+c$ for all $a,b,c\geq 0$.

We discuss two cases for Inequality (\ref{eq:two:case:MVR2}). (i) $24 \sum_{t=0}^T s_t^2 \leq \tfrac{1}{2} (1 + \sum_{t=0}^T \delta_t^2)$. We have:
\beq
\ts \EEE [ 1 + \sum_{t=0}^T\delta_t^2 ] & \leq & \ts Z_1 + \tfrac{1}{2}(1 + \sum_{t=1}^T \delta_t^2)  + Y_1 \max(\beta^{1/3},\vartheta^2) (T+1)^{1/3}  \nn\\
&\leq& 2 Z_1 + 2 Y_1 \max(\beta^{1/3},\vartheta^2) (T+1)^{1/3} .
\eeq
\noi (ii) We now discuss the case $24 \sum_{t=0}^T s_t^2 >\tfrac{1}{2} (1 + \sum_{t=0}^T \delta_t^2)$. We derive:
\beq
\ts \EEE [ 1 + \sum_{t=0}^T \delta_t^2]&\leq & \EEE [ \ts 48 \sum_{t=0}^T s_t^2 ] \nn\\
&\leq &   48 \dot{B} (\sigma^2 + \tfrac{\sigma^2}{\beta}+  \sigma^2 \beta^{-2/3} T^{1/3} ) + 48 \ddot{B} \Gamma^{1/3} + 48 C_3 T^{1/3} \beta^{1/3}  \nn\\
&\leq &   2 \max\left( 48 \dot{B} ( \sigma^2  + \tfrac{\sigma^2 }{\beta}+  \sigma^2  \beta^{-2/3} T^{1/3} )  + 48 \ddot{B} T^{1/3} \beta^{1/3} ,48 \ddot{B}  \Gamma^{1/3} \right)  \nn\\
&\leq &  \max\left( 96 \dot{B} (\sigma^2+ \tfrac{\sigma^2}{\beta}+  \sigma^2 \beta^{-2/3} T^{1/3} )  + 96 \ddot{B} T^{1/3} \beta^{1/3} ,(48 \ddot{B})^{3/2} \right)  \nn\\
&\leq &    96 \dot{B} (\sigma^2 + \tfrac{\sigma^2}{\beta}+ \sigma^2 \beta^{-2/3} T^{1/3} )  + 96 \ddot{B} T^{1/3} \beta^{1/3}  + (48 \ddot{B})^{3/2}  \nn\\
&\leq &  \underbrace{ (48 \ddot{B})^{3/2} + 96 \dot{B}  (\sigma^2+ \tfrac{\sigma^2}{\beta} )}_{:= Z_2} +  \underbrace{ 96 ( \dot{B}  + \ddot{B} ) }_{:= Y_2}\cdot \max(\sigma^2 \beta^{-2/3},\beta^{1/3} ) (T+1)^{1/3}, \nn
\eeq
\noi Combining these two cases, we obtain:
\beq
\ts \EEE [ 1 + \sum_{t=0}^T \delta_t^2] \leq \max(2Z_1, Z_2) + \max(2Y_1, Y_2) \cdot \max\left(\beta^{1/3},\vartheta^2,\sigma^2 \beta^{-2/3} \right) (T+1)^{1/3} := \Omega.
\eeq
\noi Furthermore, we have:
\beq
\ts \EEE [\sum_{t=0}^T s_t^2] ~\leq~ \max\left( \tfrac{1}{48} \Omega,\tfrac{1}{48} \Omega \right)  ~\leq~ \Omega.
\eeq

\paragraph{Part (b).} Inequality (\ref{eq:MVR2:G:d:s}) implies that:
\beq
&& \ts \EEE [\sum_{t=0}^T 4 L_t \gamma_t \GG(\x_t)] \nn\\
& \leq & \ts  \max \left( 3 P_0 \Gamma^{2/3} ,~ \big(24 \sum_{t=0}^T s_t^2 \big) + 12 D^2 \rho^2 \vartheta^{2} (T+1)^{1/3} ,~12 \overline{F}\rho T^{1/6}\beta^{1/6} \Gamma^{1/2} \right) - \Gamma \nn\\
&\overset{\step{1}}{\leq}& \ts \Omega^{1/3} \Gamma^{2/3} + \tfrac{1}{2}\Omega + \tfrac{1}{2}\Omega + 12 \overline{F}\rho T^{1/6}\beta^{1/6} \Gamma^{1/2}   - \Gamma \nn\\
&\overset{\step{2}}{\leq}& \ts \Omega  + \tfrac{1}{2}\Omega + \tfrac{1}{2}\Omega + \tfrac{1}{4} (12 \overline{F}\rho T^{1/6}\beta^{1/6})^2 \nn\\
&\overset{\step{3}}{\leq}& \ts \Omega  + \tfrac{1}{2}\Omega + \tfrac{1}{2}\Omega + \tfrac{1}{8}\Omega  \leq 4 \Omega,\nn
\eeq
\noi where step \step{1} uses $3P_0 \leq Z_1^{1/3} \leq (2Z_1)^{1/3} \leq \Omega^{1/3}$, $12 D^2 \rho^2 \vartheta^{2} (T+1)^{1/3}\leq Y_1  \vartheta^{2} (T+1)^{1/3} \leq \tfrac{1}{2}\Omega$; step \step{2} uses $-x^2+ax \leq \tfrac{1}{4}a^2$ for all $a=12 \overline{F}\rho T^{1/6}\beta^{1/6}$ and $x=\Gamma^{1/2}$; step \step{3} uses $\tfrac{1}{4} (12 \overline{F}\rho T^{1/6}\beta^{1/6})^2\leq \tfrac{1}{4} Y_1 T^{1/3} \beta^{1/3} \leq \tfrac{1}{4}\Omega$. Finally, we obtain: $\ts \EEE [\sum_{t=0}^T L_t \gamma_t \GG(\x_t)  ]\leq   \Omega$.

\noi In particular, if we set $\beta =\nu \sigma^2$ with $\nu=\OO(1)$ and $\vartheta = \tfrac{\sqrt[6]{\beta}}{1 + \sqrt[6]{\beta}}$, we have:
\beq
\ts\EEE [ \sum_{t=0}^T L_t \gamma_t \GG(\x_t)] &\leq & \ts  \max(2Z_1, Z_2) +  \max(2Y_1, Y_2) \cdot \max\left(\beta^{1/3},\vartheta^2,\sigma^2 \beta^{-2/3} \right) (T+1)^{1/3} \nn\\
&\overset{\step{1}}{\leq}& \ts \max(2Z_1, Z_2) +  \max(2Y_1, Y_2) \cdot \max\left( \nu^{1/3} , \nu^{-2/3} \right)\sigma^{2/3}  (T+1)^{1/3},\nn
\eeq
\noi where step \step{1} uses $\vartheta^2  \leq \beta^{1/3}$ and $\beta =\nu \sigma^2$. 

\end{proof}

\subsection{Proof of Theorem \ref{theorem:MVR2}}
\label{app:theorem:MVR2}

\begin{proof} We let $\gamma_t = \tfrac{L_0}{2L_t} \left( \vartheta (T+1)^{-1/3} + (T+1)^{-1/2} \right)$, where $\vartheta:=\tfrac{\beta^{1/6}}{1+\beta^{1/6}}$, and $\beta = \nu \sigma^2$.

We have from Lemma \ref{lemma:lemma:MVR2:S:delta}(b):
\beq
\ts \EEE [\sum_{t=0}^T L_t\gamma_t\GG(\x_t) ] ~\leq~  \max(2Z_1, Z_2) + \max(2Y_1, Y_2) \cdot \max( \nu^{1/3}, \nu^{-2/3} ) \sigma^{2/3} (T+1)^{1/3} .\nn
\eeq
\noi This further leads to the following results:
\beq
&& \ts \EEE [\sum_{t=0}^T L_t\gamma_t\GG(\x_t) ] \leq \tilde{\OO}( \max( \nu^{1/3}, \nu^{-2/3} ) \sigma^{2/3} T^{1/3}  + 1)\nn\\
&\overset{\step{1}}{\Rightarrow} & \ts \tfrac{L_0}{2} \left( \vartheta (T+1)^{-1/3} + (T+1)^{-1/2} \right) \EEE [\sum_{t=0}^T  \GG(\x_t) ] \leq \tilde{\OO}( \max( \nu^{1/3}, \nu^{-2/3} ) \sigma^{2/3} T^{1/3}  + 1)\nn\\
&\Rightarrow& \ts \EEE [\tfrac{1}{T+1} \sum_{t=0}^T  \GG(\x_t)] \leq \tilde{\OO} \left( \tfrac{ \max( \nu^{1/3}, \nu^{-2/3} ) \sigma^{2/3}   T^{-2/3}+T^{-1}}{  \vartheta T^{-1/3} + T^{-1/2}  } \right) \nn\\
&\overset{\step{2}}{\Rightarrow} & \ts\EEE [ \tfrac{1}{T+1} \sum_{t=0}^T  \GG(\x_t)] \leq \tilde{\OO} \left( \tfrac{\max( \nu^{1/3}, \nu^{-2/3} ) \sigma^{2/3}}{ \nu^{1/6}\sigma^{1/3}} (1+\sigma^{1/3}) T^{-1/3} + T^{-1/2} \right) \nn\\
&\overset{\step{3}}{\Rightarrow} & \ts \EEE [\tfrac{1}{T+1} \sum_{t=0}^T  \GG(\x_t) ] \leq \tilde{\OO} \left( \tfrac{ (\frac{1}{\nu} +\nu) \sigma^{2/3}}{ \sigma^{1/3}} (1+\sigma^{1/3}) T^{-1/3} + T^{-1/2} \right), \nn
\eeq
\noi where step \step{1} uses the choice of $\gamma_t$; step \step{2} uses $\tfrac{a+b}{c+d}\leq \frac{a}{c} + \tfrac{b}{d}$ for all $a,b\geq 0$ and $c,d>0$ and the choice $\vartheta = \tfrac{\sigma^{1/3}}{1+\sigma^{1/3}}$; step \step{3} uses $\frac{\max( \nu^{1/3}, \nu^{-2/3} )}{\nu^{1/6}} \leq \frac{1}{\nu} + \nu$ for all $\nu>0$.

\end{proof}

\section{Experiments}
\label{app:sect:exp}

This section evaluates the empirical performance of ALFCG against state-of-the-art Frank-Wolfe (FW) variants. We conduct experiments across deterministic, finite-sum, and expectation settings, focusing on multi-class classification under nuclear norm and $\ell_p$-norm ($p=3$) ball constraints.

\subsection{Datasets and Tasks}

We consider multi-class classification with smooth loss functions. To assess the efficiency of our projection-free approach, we employ two constraint sets: (i) the nuclear norm ball, and (ii) the $\ell_p$-ball with $p=3$, where standard Euclidean projections are computationally demanding.
Datasets are denoted as `\texttt{randn}-$N$-$d$' (synthetic Gaussian data) and `\texttt{mnist}-$N$-$d$' (subsampled from the MNIST dataset), where $N$ and $d$ denote the number of samples and features, respectively. For synthetic data, labels are generated via a one-hot encoding of random class assignments.

\subsection{Compared Methods} 

We compare ALFCG against a comprehensive set of baselines.
\begin{enumerate}[label=\textbf{(\alph*)}, leftmargin=20pt, itemsep=1pt, topsep=1pt, parsep=0pt, partopsep=0pt]

\item  \textbf{Deterministic Setting:} FW-OpenLoop \cite{FrankWolfe1956},  FW-ShortStep \cite{ICMLjaggi13}, FW-Momentum \cite{li2021heavy}, FW-Sliding \cite{lan2017conditionalSIOPT}, FW-Armijo \cite{ochs2019model}, and FW-ParaFree \cite{ito2023parameter}.
    \item \textbf{Finite-Sum Setting:} SVFW \cite{reddi2016stochastic}, SAGA-FW \cite{reddi2016stochastic}, SPIDER-CG \cite{yurtsever19bSpiderCG}, OneSample-FW \cite{zhang2020oneOneSample}, and SARAH-FW \cite{beznosikov2024sarah}.
    \item \textbf{Expectation Setting:} SFW \cite{reddi2016stochastic}, SFW-GB \cite{hazan2016variance}, OneSample-STORM \cite{zhang2020oneOneSample}, OneSample-EMA \cite{zhang2020oneOneSample}, and SRFW\cite{tang2022high}.
\end{enumerate}

\subsection{Parameter Settings}

We consider the following parameter setting in our experiments.

\noi $\bullet$ Deterministic Setting. \ding{182} ALFCG-D: Uses an adaptive Lipschitz-free rule with $\rho = 10^{-5}$. \ding{183} FW-Armijo: Employs backtracking line search with initial step size $1.0$, contraction $\beta = 0.5$, and descent parameter $c = 10^{-4}$. \ding{184} FW-Momentum: Uses a momentum decay of $\rho = 0.9$ and the open-loop step size $\eta_t = 2/(t+2)$. \ding{185} FW-Openloop: Uses a diminishing step-size $\eta_t = 2/(t+1)$. \ding{186} FW-ParaFree: Adapts local curvature $L$ via line search with increase/decrease factors $\tau_{inc} = 2.0$ and $\tau_{dec} = 0.5$. \ding{187} FW-ShortStep: Minimizes a quadratic model using a constant Lipschitz parameter $L=10$. \ding{188} FW-Sliding: Sets $L=10$ and performs $5$ inner iterations per outer Nesterov acceleration step.

\noi $\bullet$ Finite-Sum Setting. \ding{182} ALFCG-FS: Employs a scaling factor $\rho = 10^{-5}$ and a SPIDER estimator with epoch length and mini-batch size $q = \lfloor\sqrt{N}\rfloor$. \ding{183} FW-SVFW: Implements an epoch length $m = 50$ and mini-batch size $\lfloor\sqrt{N}\rfloor$ with $\eta_t = 2/(t+1)$. \ding{184} FW-SAGA: Uses the SAGA estimator with $\eta_t = 2/(t+1)$. \ding{185} FW-SPIDER: Uses a refresh period $q = \sqrt{N}$ and mini-batch sizes of $\sqrt{N}$ for resets and $32$ for inner updates with $\eta_t = 2/(t+1)$. \ding{186} FW-OneSample: Uses a reset period $q = 50$, a variance-reduced gradient estimator, and step size $\eta_t = 2/(t+1)$ \ding{187} FW-SARAH: Utilizes a recursive estimator with inner-loop $m = 50$, mini-batch size $32$, and $\eta_t = 2/(t+1)$.

\noi $\bullet$ Expectation Setting. \ding{182} ALFCG-MVR1 and ALFCG-MVR2: Uses a scaling factor $\rho = 10^{-5}$, noise parameter $\beta = 100$, and mini-batch size $\lfloor\sqrt{N}\rfloor$. \ding{183} SFW: Implements stochastic gradient estimation with a mini-batch size $\lfloor\sqrt{N}\rfloor$ and $\eta_t = 2/(t+2)$. \ding{184} SFW-GB: Follows an increasing mini-batch schedule $b_t=\lfloor (t+1)^{2/3} \rfloor)$ with $\eta_t = 2/(t+1)$. \ding{185} OneSample-STORM: Utilizes a recursive STORM estimator with mini-batch size $\lfloor\sqrt{N}\rfloor$, where both step size $\eta_t$ and momentum $\rho_t$ follow a $t^{-2/3}$ schedule. \ding{186} OneSample-EMA: Updates the gradient via EMA with decay $\rho_t = (t+1)^{-2/3}$ and a mini-batch size $\lfloor\sqrt{N}\rfloor$, and step size $\eta_t = 2/(t+1)$. \ding{187} SRFW: Aggregates $K=5$ gradient groups via a coordinate-wise median, using a per-group mini-batch size $\lfloor\sqrt{N}\rfloor$ and $\eta_t = 2/(t+1)$.

\subsection{Results and Discussions}
The experimental results yield the following observations:

\begin{enumerate}[label=\textbf{(\alph*)}, leftmargin=20pt, itemsep=1pt, topsep=1pt, parsep=0pt, partopsep=0pt]

\item \textbf{Deterministic Setting}. Our $f$-value-free ALFCG-D is at least comparable to, if not faster than, FW-Armijo and FW-ParaFree, both of which rely on computationally expensive backtracking line searches. Furthermore, ALFCG-D significantly outperforms standard baselines such as FW-OpenLoop, FW-ShortStep, FW-Momentum, and FW-Sliding.

\item \textbf{Finite-Sum Setting}. While ALFCG-FS occasionally performs on par with other methods in rare instances, it achieves state-of-the-art performance in the vast majority of cases. 

\item \textbf{Expectation Setting}. ALFCG-MVR1 and ALFCG-MVR2 significantly exceed existing stochastic CG benchmarks. This is attributed to their noise-adaptive design, which facilitates faster convergence as the variance vanishes. While the two variants are largely comparable, ALFCG-MVR1 often exhibits a slight performance edge in practice.

\end{enumerate}

\noi $\bigstar$ Source code and additional implementation details are provided in the \textbf{supplemental material}.




\begin{figure}[!hb]
  \centering
  \begin{tabular}{cccc}
    \includegraphics[width=0.23\linewidth,height=2.8cm]{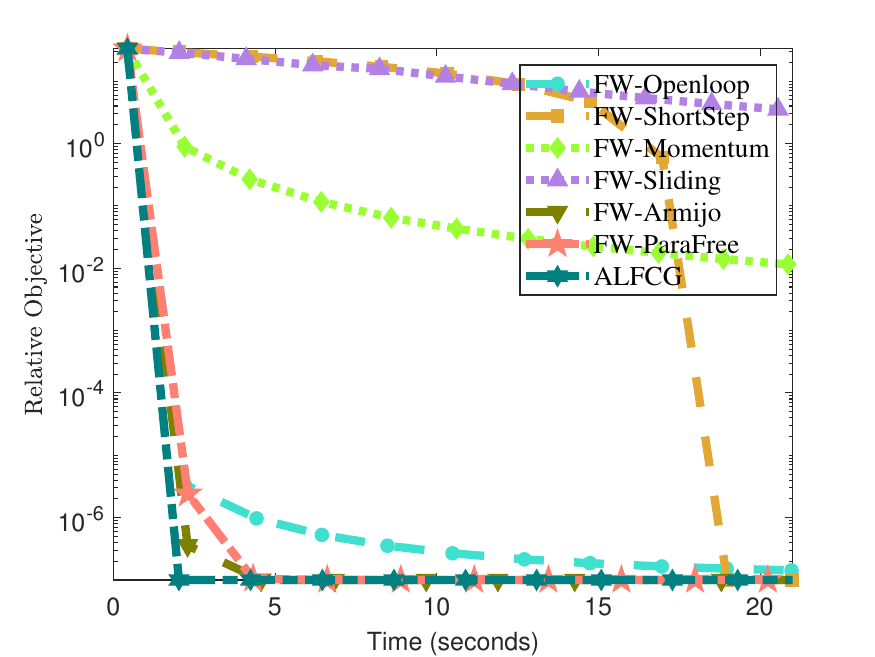} &
    \includegraphics[width=0.23\linewidth,height=2.8cm]{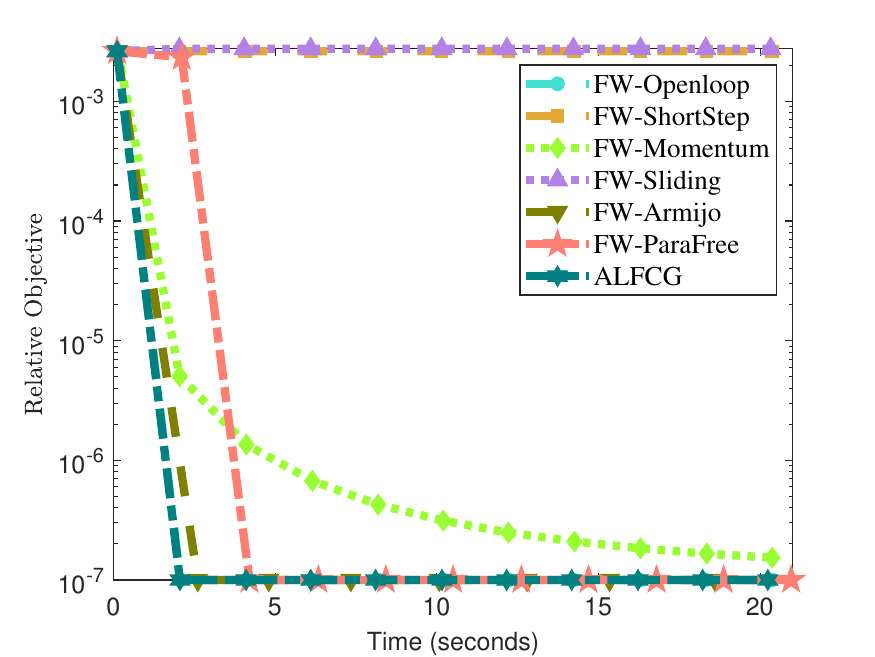} &
    \includegraphics[width=0.23\linewidth,height=2.8cm]{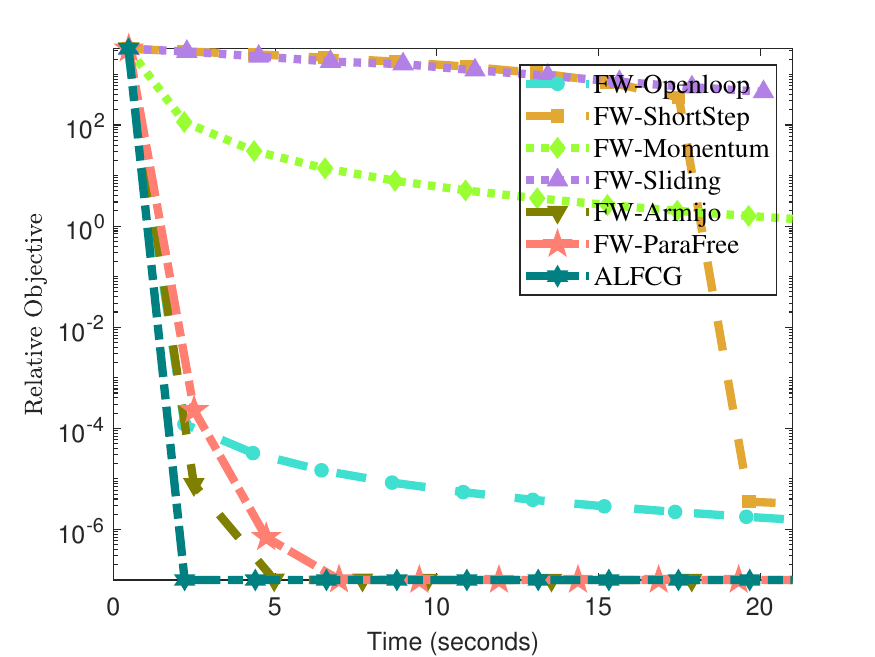} &
    \includegraphics[width=0.23\linewidth,height=2.8cm]{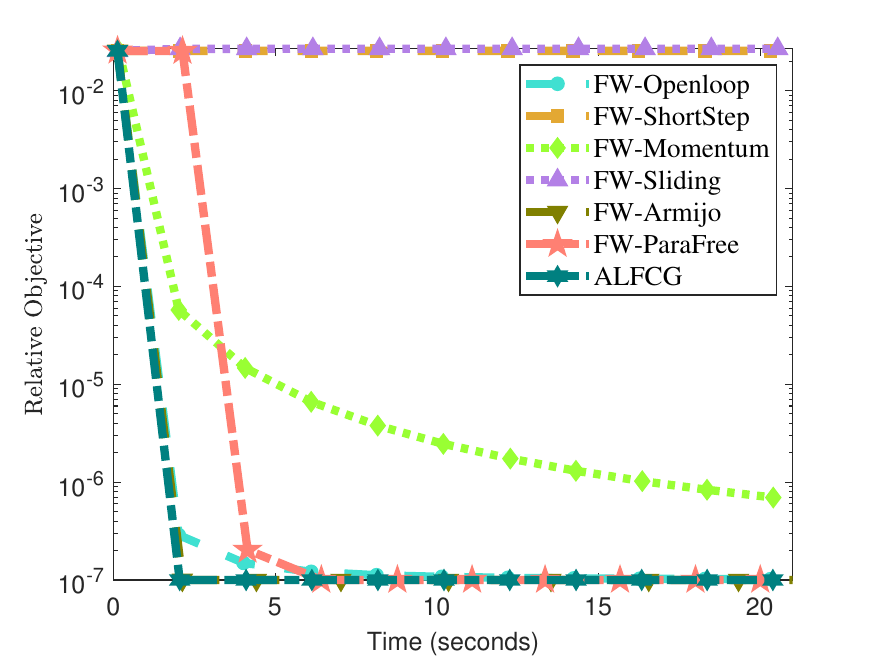} \\
    {\scriptsize (a) randn-30000-5000, $\delta_1=10$} & {\scriptsize (b) mnist-50000-768, $\delta_1=10$} & {\scriptsize (c) randn-30000-5000, $\delta_1=100$} & {\scriptsize (d) mnist-50000-768, $\delta_1=100$}
  \end{tabular}

  \centering
  \begin{tabular}{cccc}
    \includegraphics[width=0.23\linewidth,height=2.8cm]{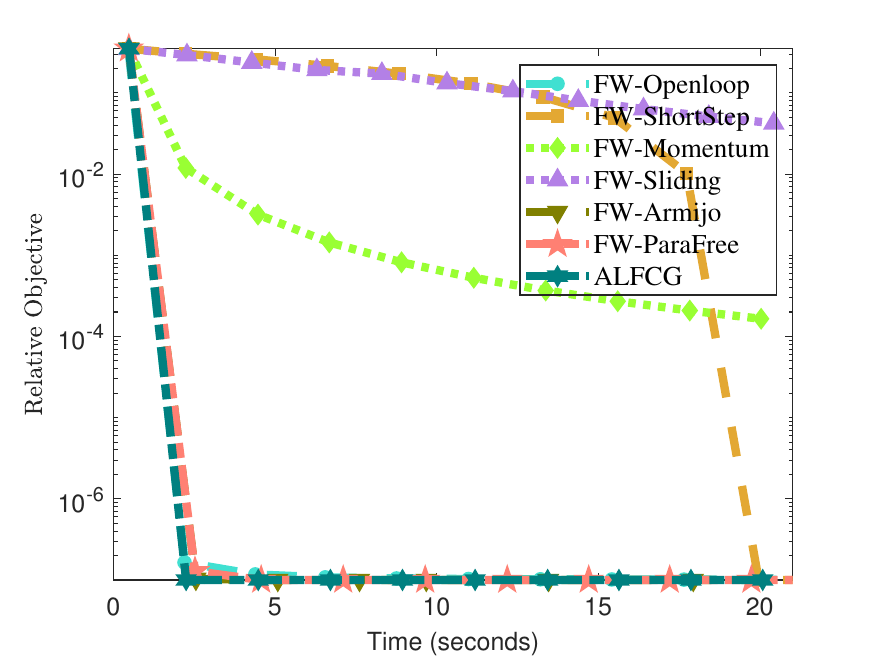} &
    \includegraphics[width=0.23\linewidth,height=2.8cm]{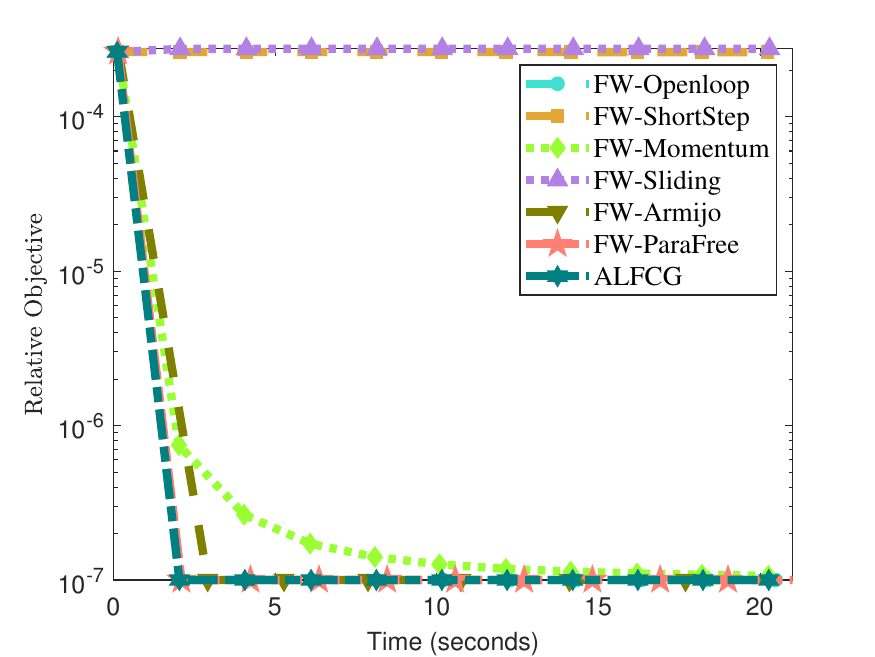} &
    \includegraphics[width=0.23\linewidth,height=2.8cm]{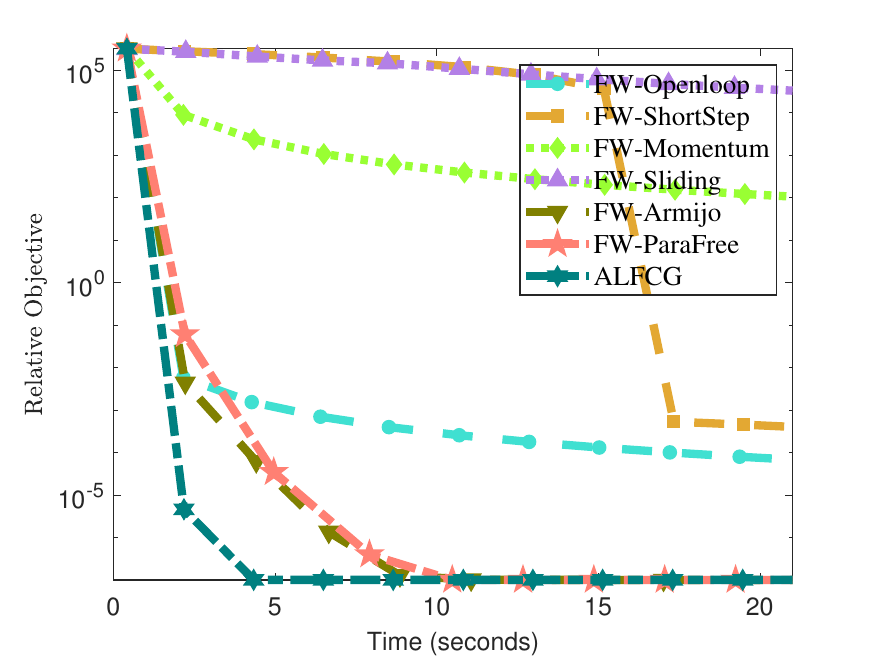} &
    \includegraphics[width=0.23\linewidth,height=2.8cm]{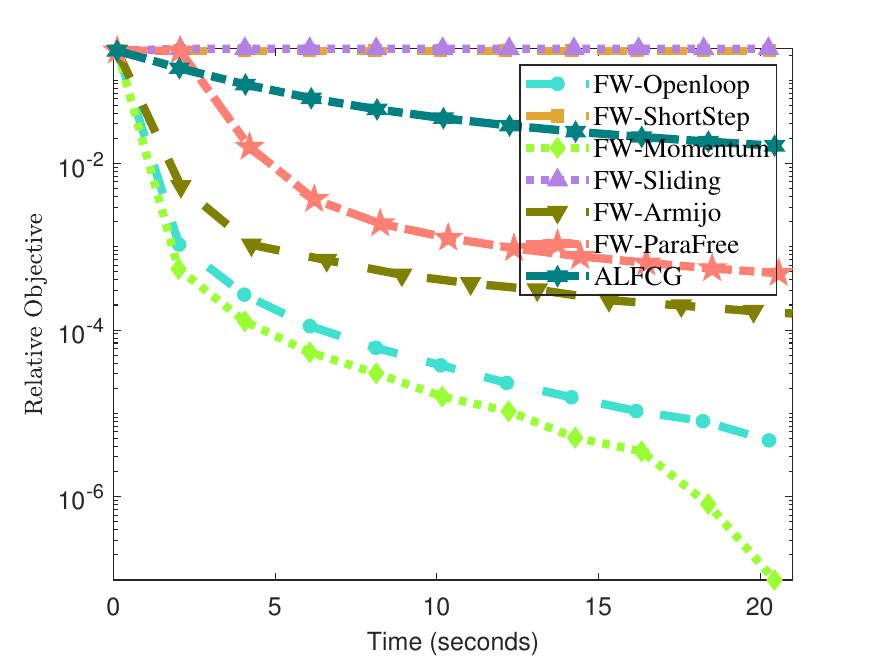} \\
    {\scriptsize (e) randn-30000-5000, $\delta_1=1$} & {\scriptsize (f) mnist-50000-768, $\delta_1=1$} & {\scriptsize (g) randn-30000-5000, $\delta_1=1000$} & {\scriptsize (h) mnist-50000-768, $\delta_1=1000$}
  \end{tabular}
    \caption{Experiment results for the nuclear norm ball constraint problem for \textbf{deterministic setting} with varying radius parameter $\delta = \{10,100,1,1000\}$.}
  \label{fig:1}
\end{figure}


\begin{figure}[t]
  \centering
  \begin{tabular}{cccc}
    \includegraphics[width=0.23\linewidth,height=2.8cm]{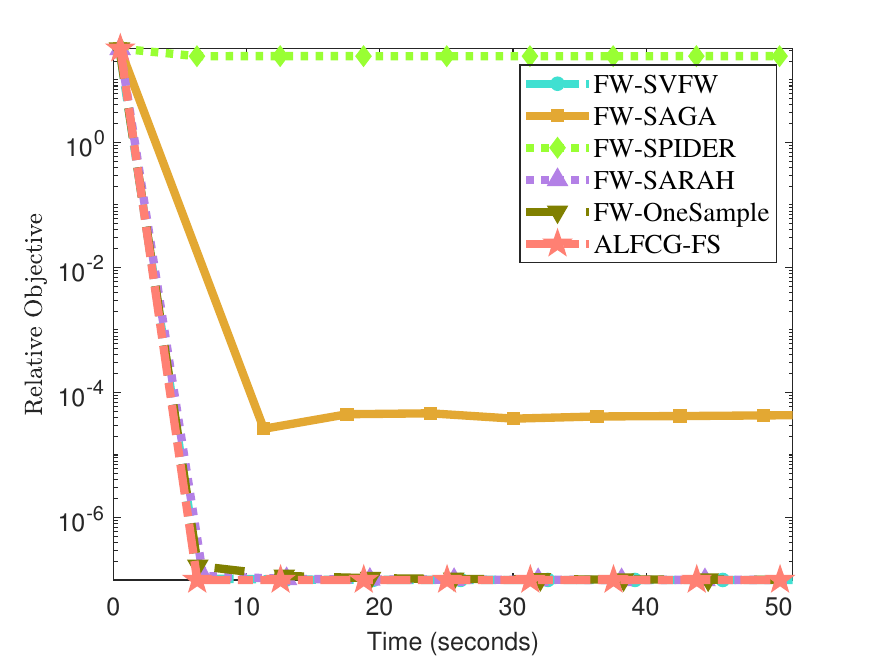} &
    \includegraphics[width=0.23\linewidth,height=2.8cm]{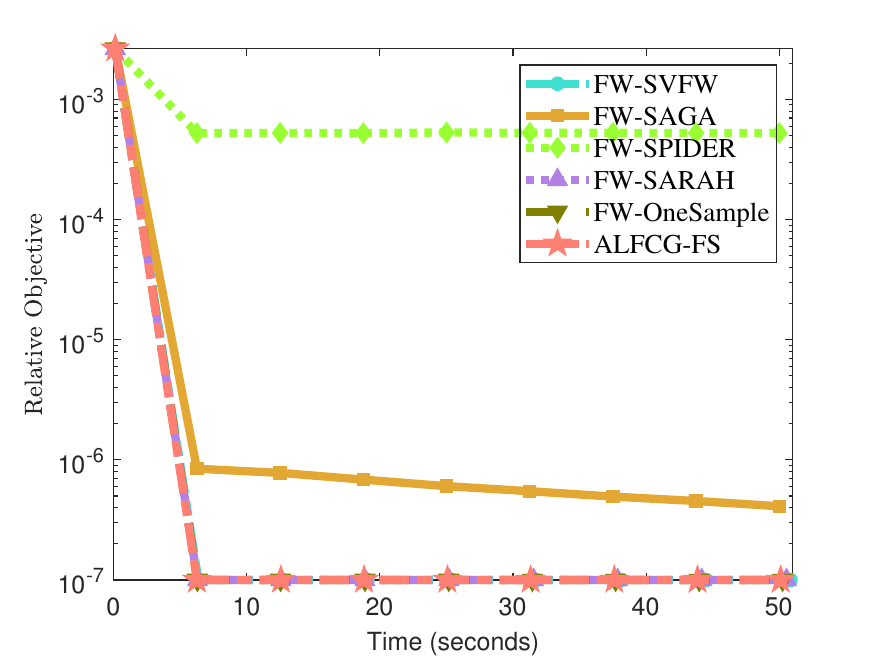} &
    \includegraphics[width=0.23\linewidth,height=2.8cm]{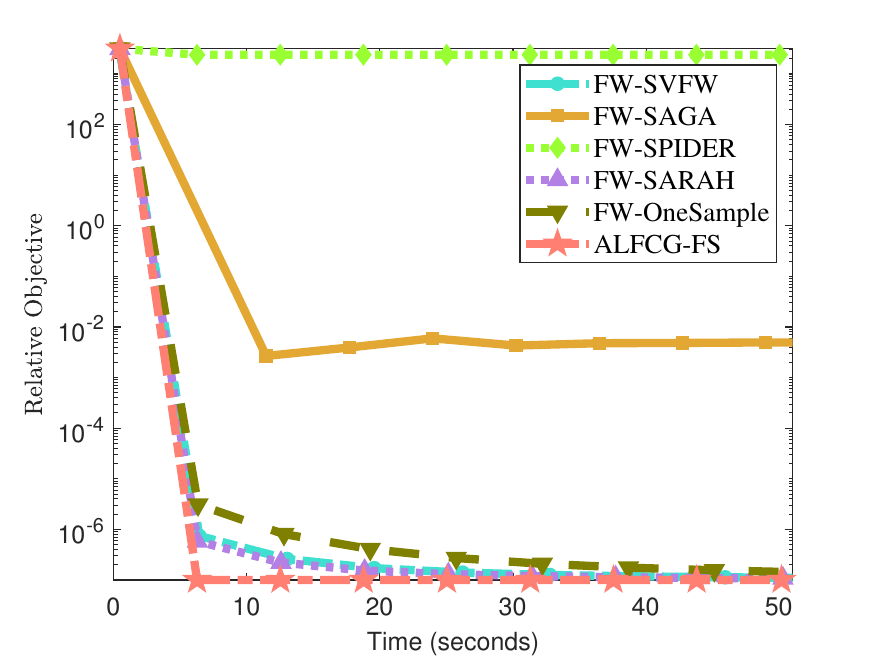} &
    \includegraphics[width=0.23\linewidth,height=2.8cm]{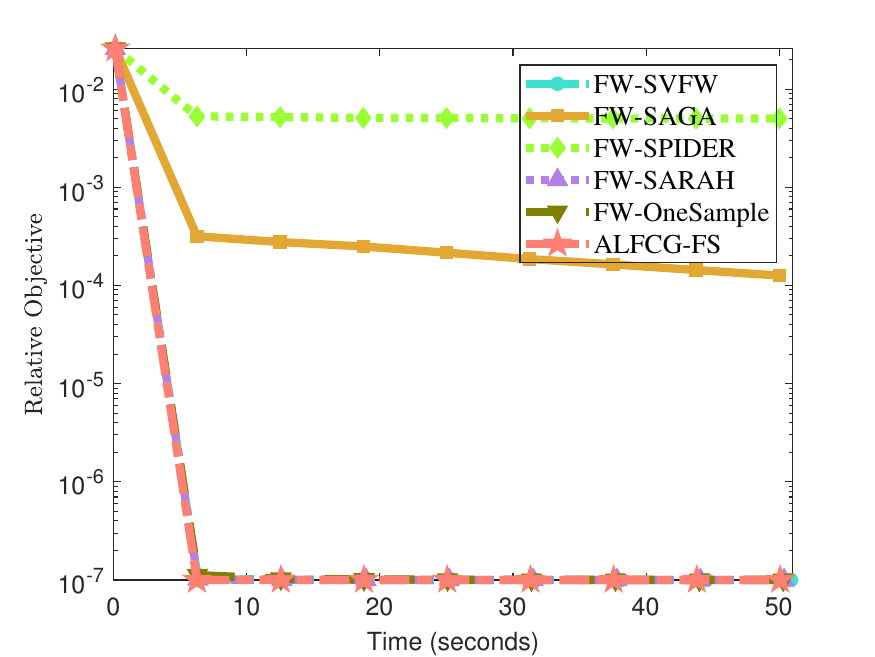} \\
    {\scriptsize (a) randn-30000-5000, $\delta_1=10$} & {\scriptsize (b) mnist-50000-768, $\delta_1=10$} & {\scriptsize (c) randn-30000-5000, $\delta_1=100$} & {\scriptsize (d) mnist-50000-768, $\delta_1=100$}
  \end{tabular}

  \centering
  \begin{tabular}{cccc}
    \includegraphics[width=0.23\linewidth,height=2.8cm]{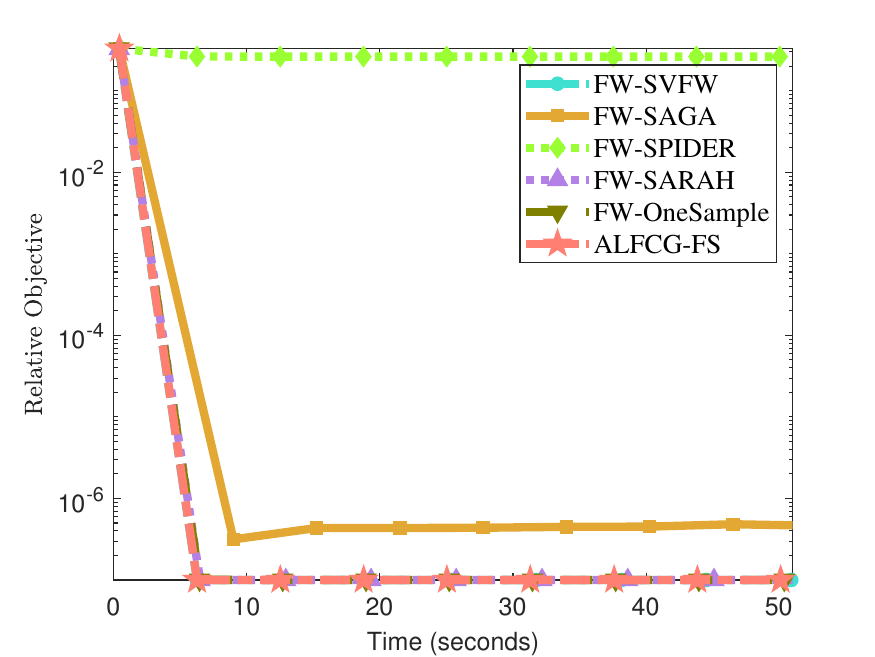} &
    \includegraphics[width=0.23\linewidth,height=2.8cm]{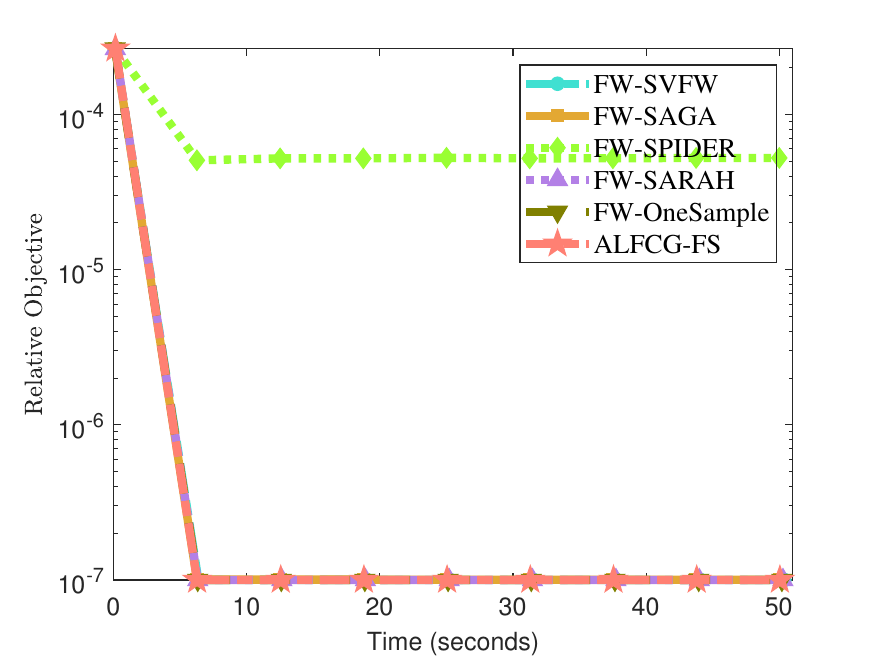} &
    \includegraphics[width=0.23\linewidth,height=2.8cm]{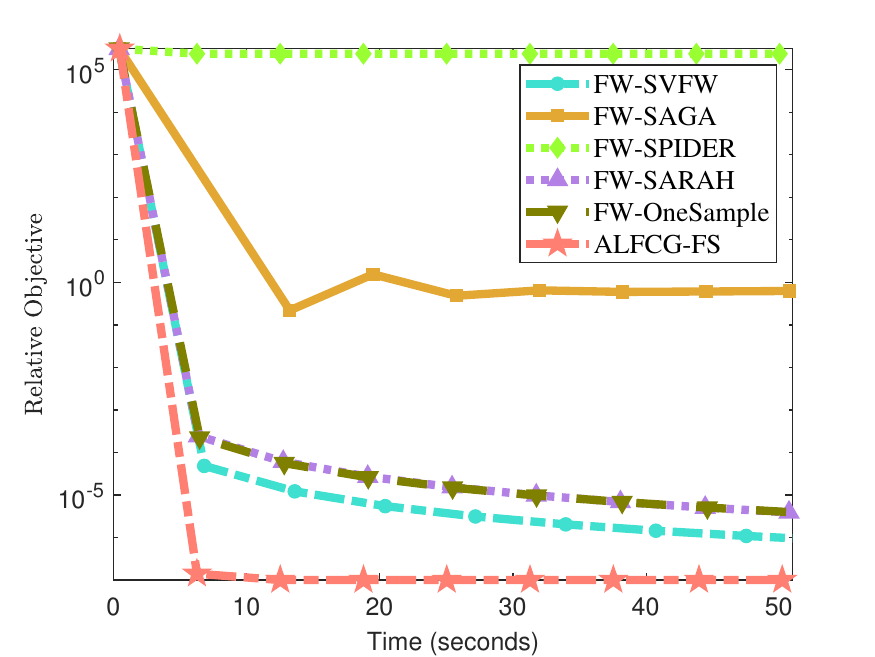} &
    \includegraphics[width=0.23\linewidth,height=2.8cm]{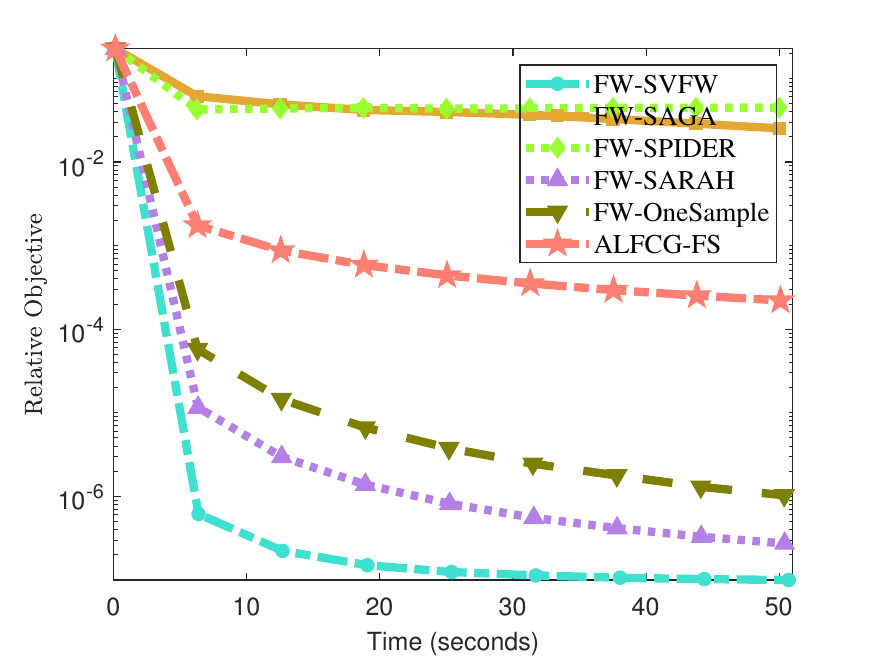} \\
    {\scriptsize (e) randn-30000-5000, $\delta_1=1$} & {\scriptsize (f) mnist-50000-768, $\delta_1=1$} & {\scriptsize (g) randn-30000-5000, $\delta_1=1000$} & {\scriptsize (h) mnist-50000-768, $\delta_1=1000$}
  \end{tabular}
    \caption{Experiment results for the nuclear norm ball constraint problem for \textbf{finite-sum setting} with varying radius parameter $\delta = \{10,100,1,1000\}$.}
  \label{fig:2}
\end{figure}


\begin{figure}[t]
  \centering
  \begin{tabular}{cccc}
    \includegraphics[width=0.23\linewidth,height=2.8cm]{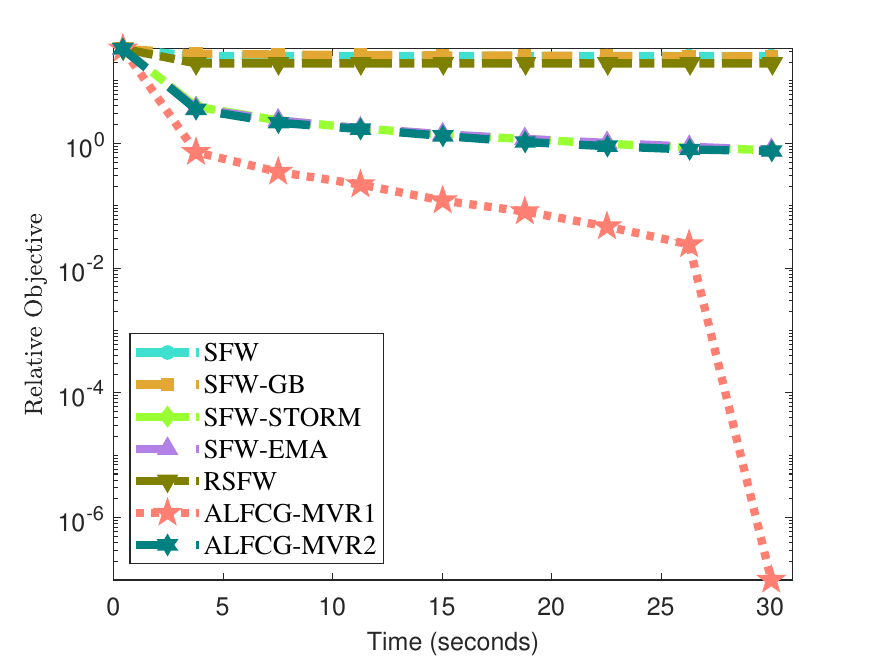} &
    \includegraphics[width=0.23\linewidth,height=2.8cm]{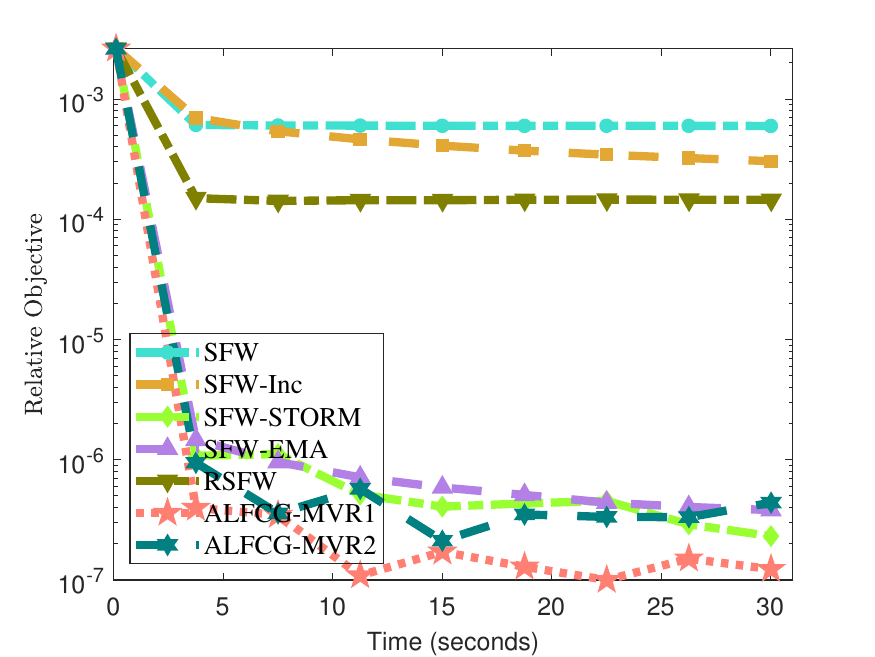} &
    \includegraphics[width=0.23\linewidth,height=2.8cm]{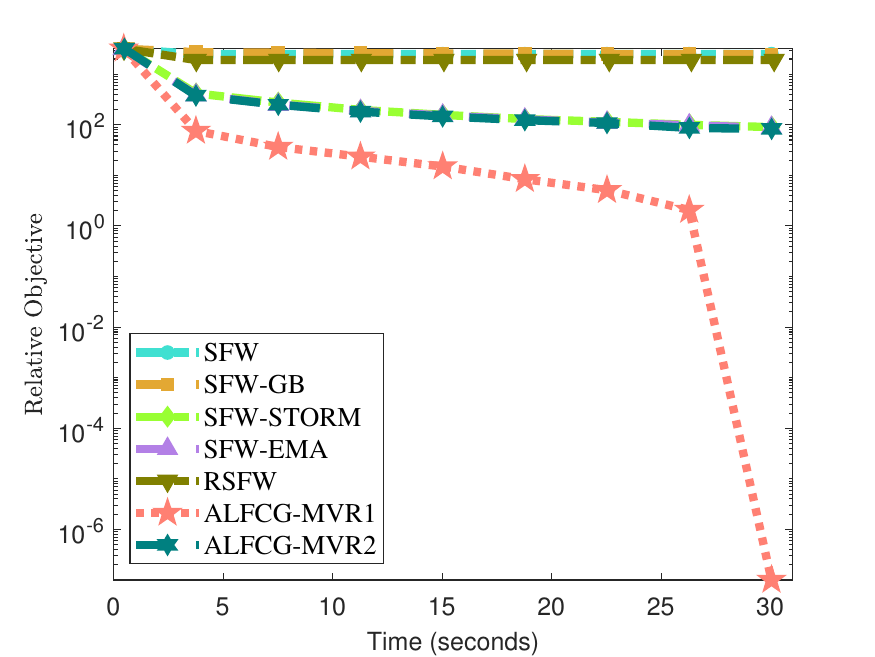} &
    \includegraphics[width=0.23\linewidth,height=2.8cm]{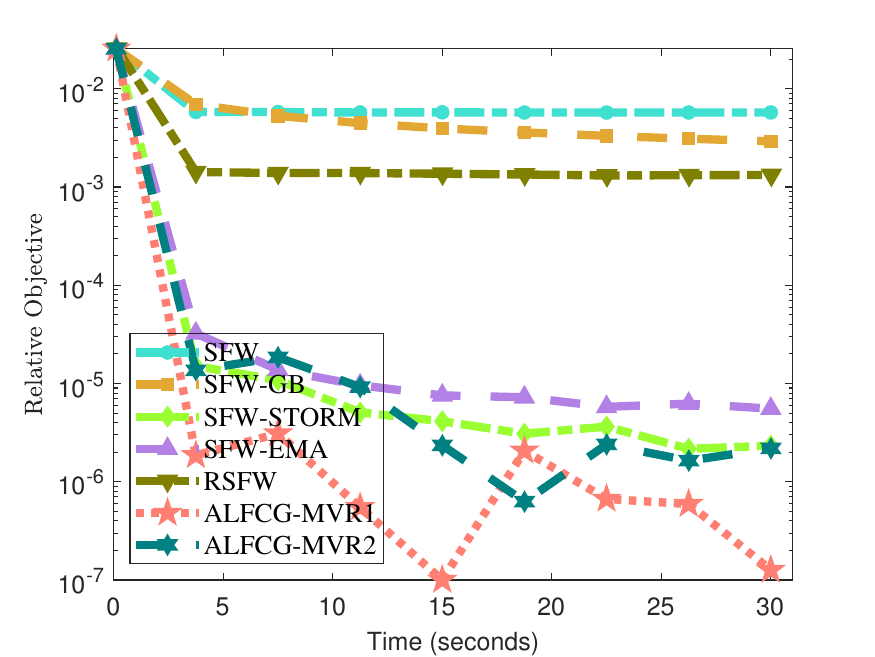} \\
    {\scriptsize (a) randn-30000-5000, $\delta_1=10$} & {\scriptsize (b) mnist-50000-768, $\delta_1=10$} & {\scriptsize (c) randn-30000-5000, $\delta_1=100$} & {\scriptsize (d) mnist-50000-768, $\delta_1=100$}
  \end{tabular}

  \centering
  \begin{tabular}{cccc}
    \includegraphics[width=0.23\linewidth,height=2.8cm]{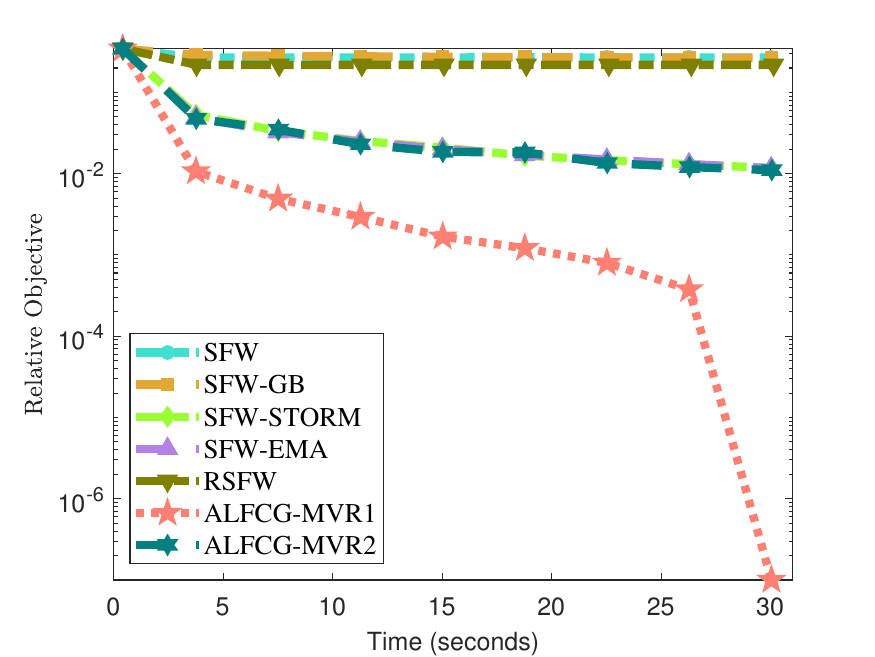} &
    \includegraphics[width=0.23\linewidth,height=2.8cm]{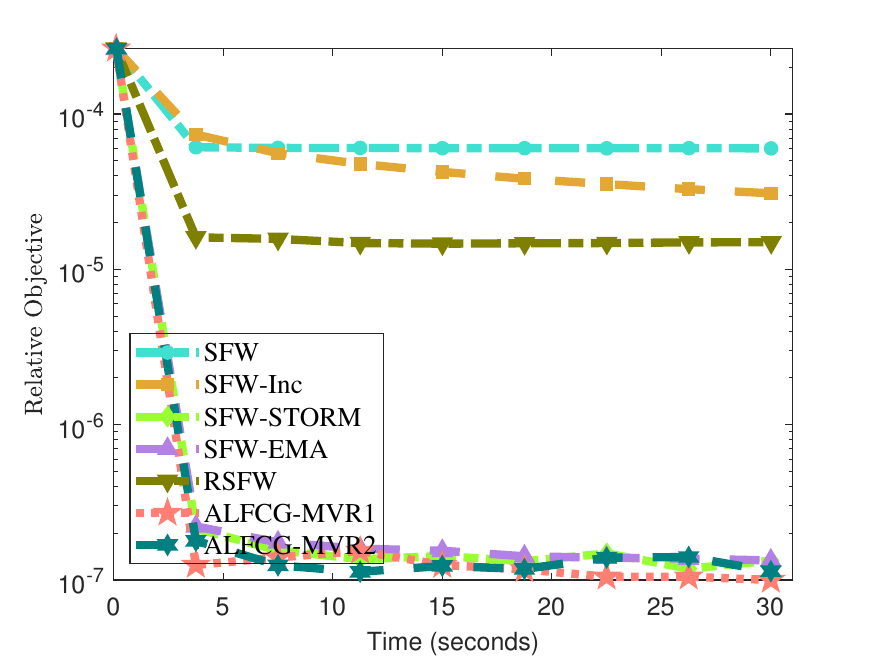} &
    \includegraphics[width=0.23\linewidth,height=2.8cm]{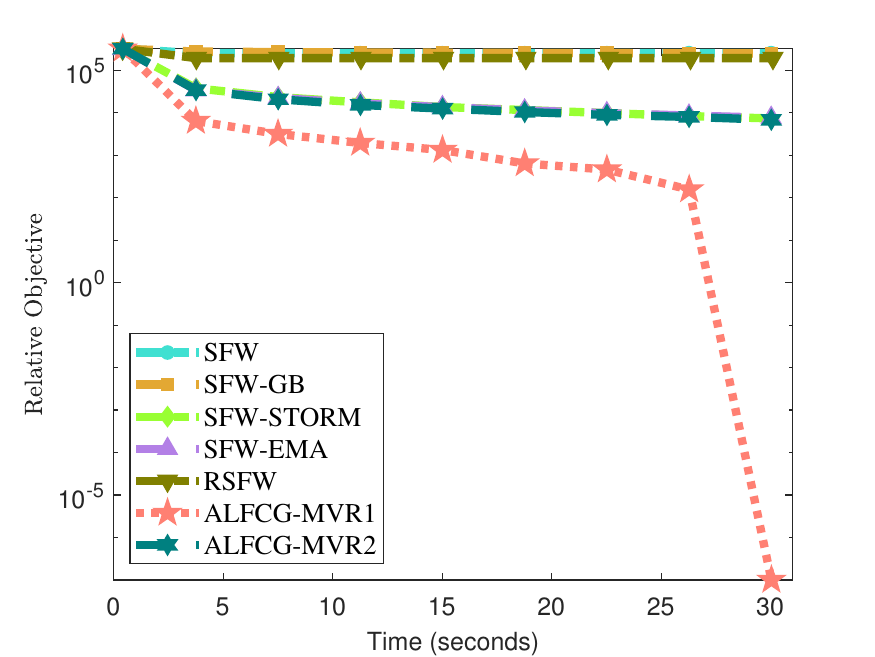} &
    \includegraphics[width=0.23\linewidth,height=2.8cm]{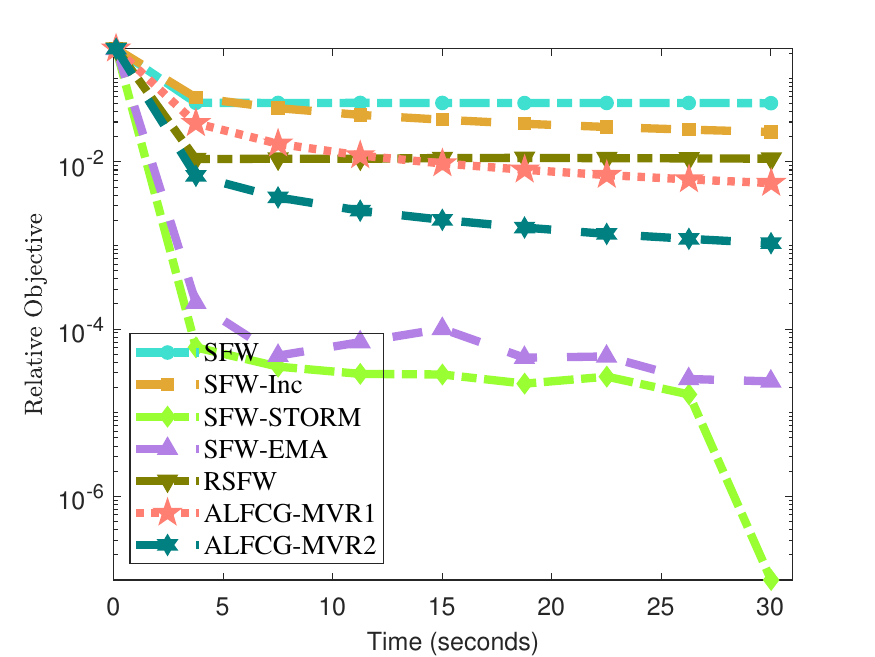} \\
    {\scriptsize (e) randn-30000-5000, $\delta_1=1$} & {\scriptsize (f) mnist-50000-768, $\delta_1=1$} & {\scriptsize (g) randn-30000-5000, $\delta_1=1000$} & {\scriptsize (h) mnist-50000-768, $\delta_1=1000$}
  \end{tabular}
    \caption{Experiment results for the nuclear norm ball constraint problem for \textbf{expectation setting} with varying radius parameter $\delta = \{10,100,1,1000\}$.}
  \label{fig:3}
\end{figure}


\begin{figure}[t]
  \centering
  \begin{tabular}{cccc}
    \includegraphics[width=0.23\linewidth,height=2.8cm]{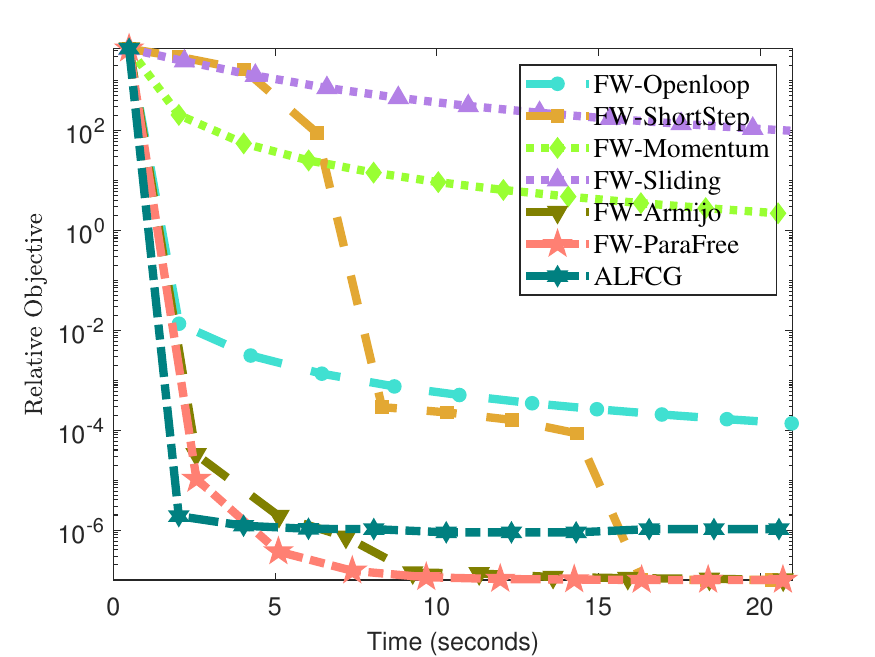} &
    \includegraphics[width=0.23\linewidth,height=2.8cm]{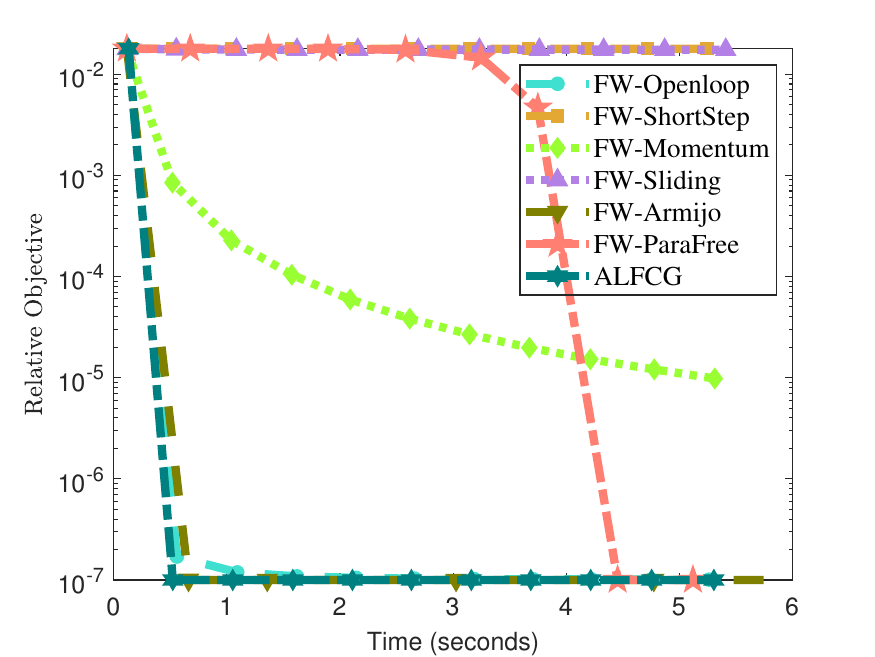} &
    \includegraphics[width=0.23\linewidth,height=2.8cm]{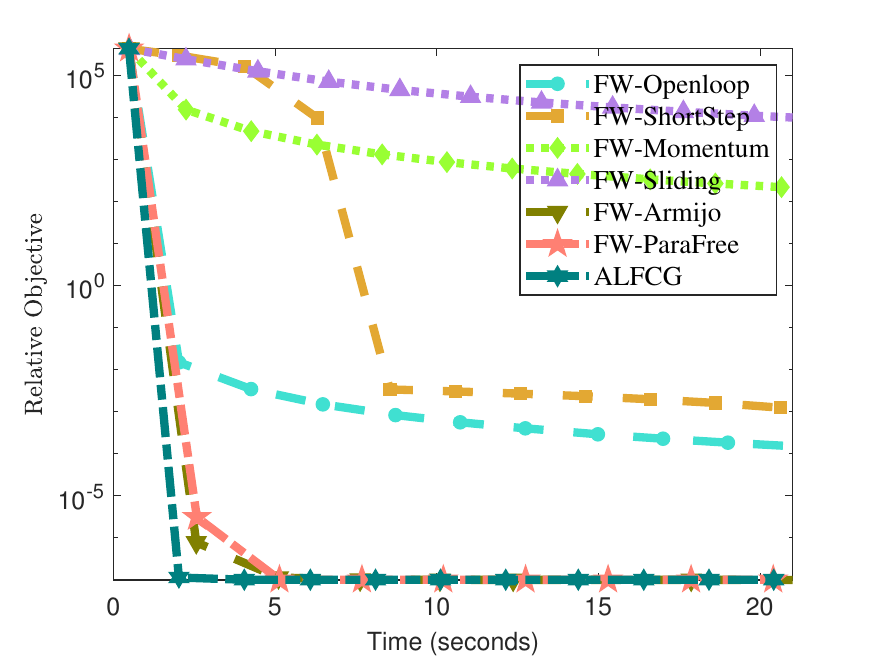} &
    \includegraphics[width=0.23\linewidth,height=2.8cm]{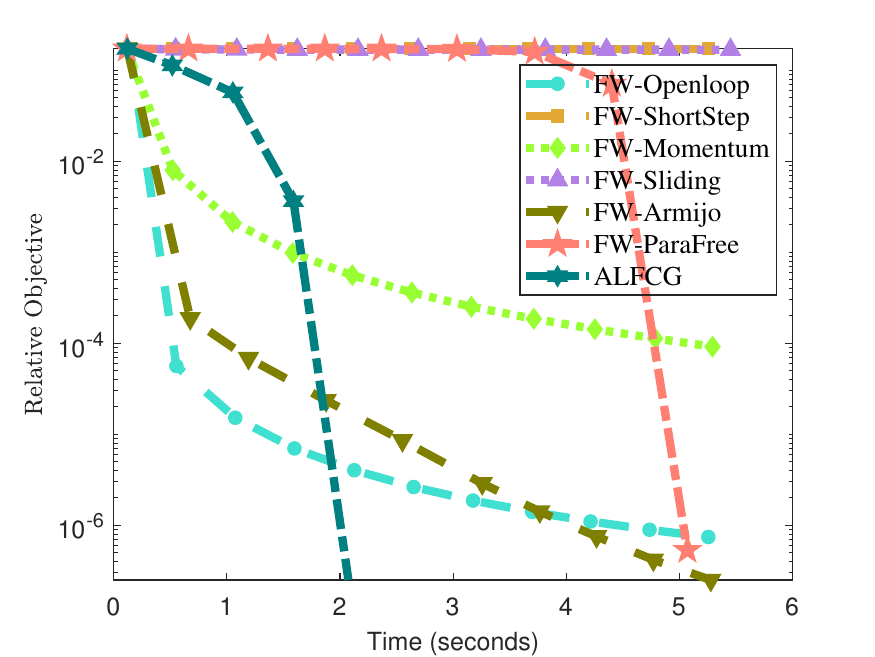} \\
    {\scriptsize (a) randn-30000-5000, $\delta_2=10$} & {\scriptsize (b) mnist-50000-768, $\delta_2=10$} & {\scriptsize (c) randn-30000-5000, $\delta_2=100$} & {\scriptsize (d) mnist-50000-768, $\delta_2=100$}
  \end{tabular}

  \centering
  \begin{tabular}{cccc}
    \includegraphics[width=0.23\linewidth,height=2.8cm]{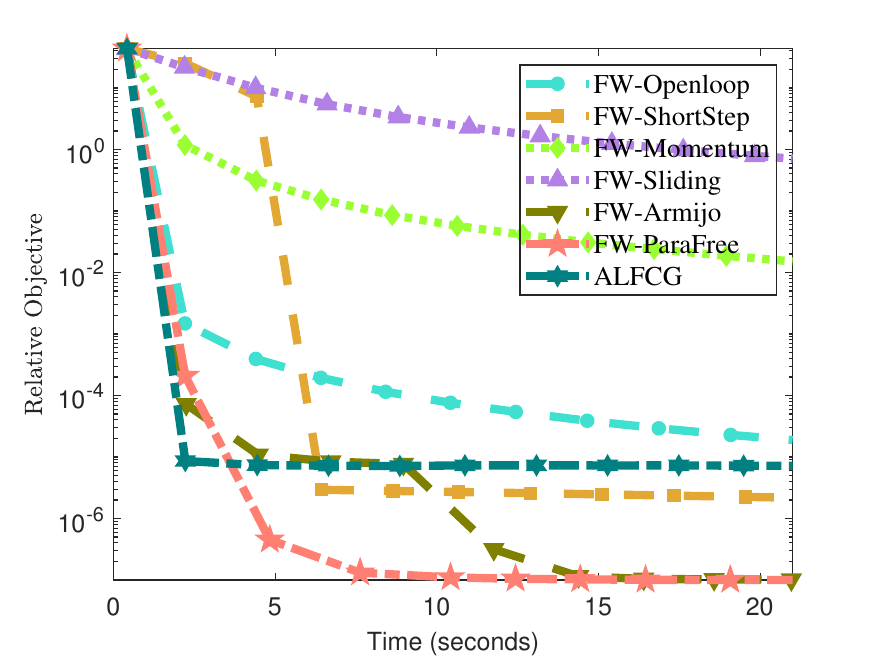} &
    \includegraphics[width=0.23\linewidth,height=2.8cm]{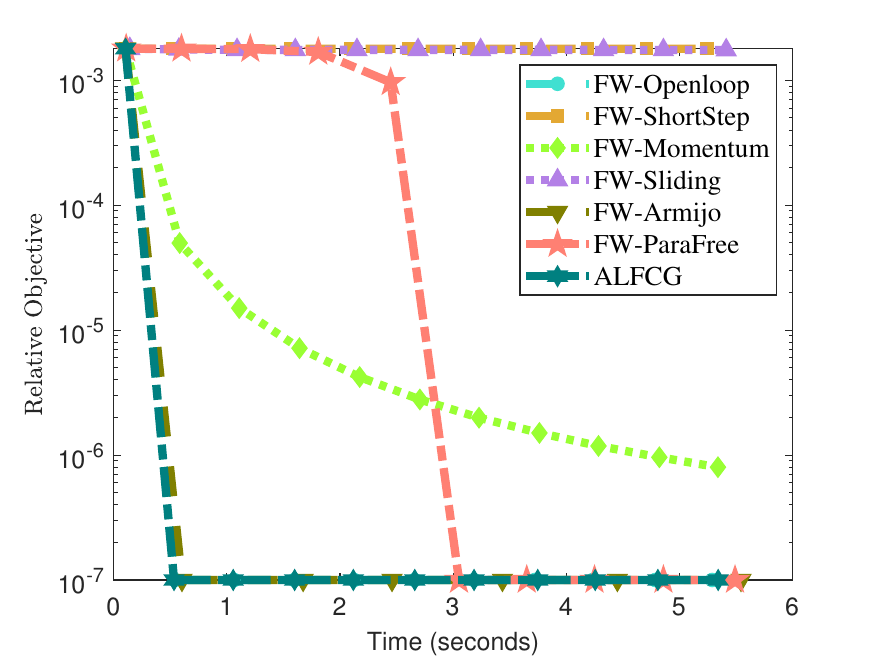} &
    \includegraphics[width=0.23\linewidth,height=2.8cm]{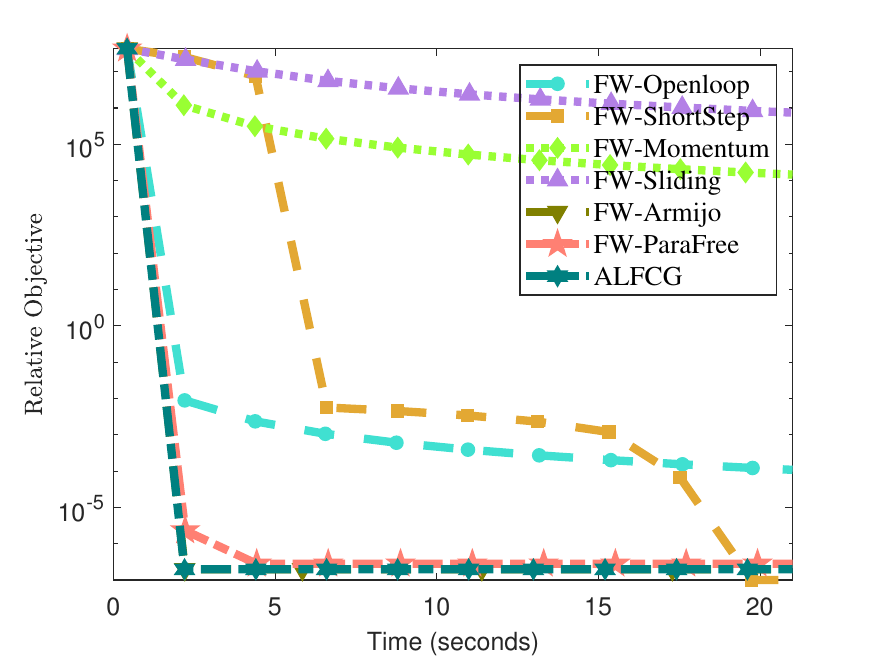} &
    \includegraphics[width=0.23\linewidth,height=2.8cm]{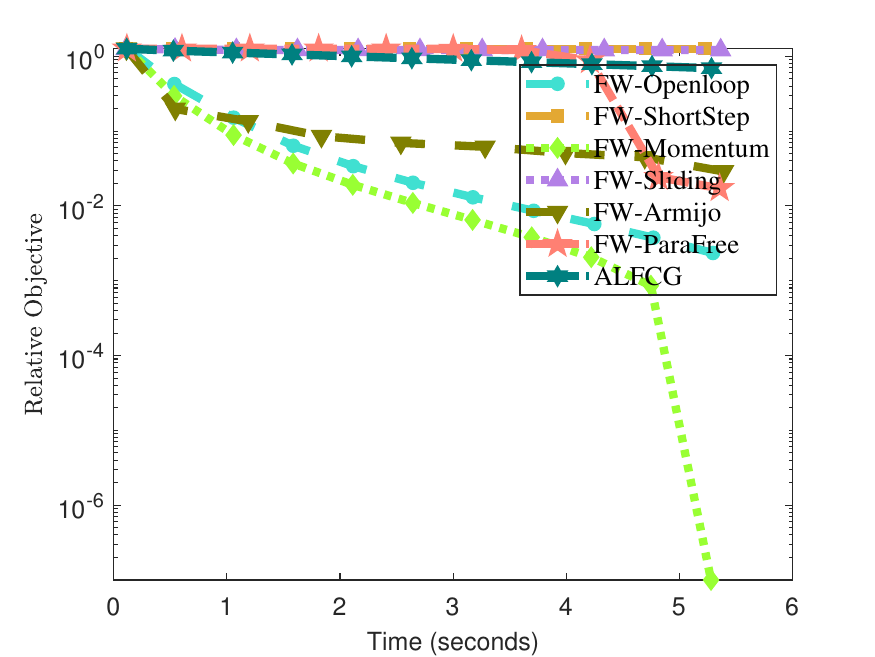} \\
    {\scriptsize (e) randn-30000-5000, $\delta_2=1$} & {\scriptsize (f) mnist-50000-768, $\delta_2=1$} & {\scriptsize (g) randn-30000-5000, $\delta_2=1000$} & {\scriptsize (h) mnist-50000-768, $\delta_2=1000$}
  \end{tabular}
    \caption{Experiment results for the $\ell_p$ norm ball constraint problem for \textbf{deterministic setting} with varying radius parameter $\delta_2 = \{10,100,1,1000\}$.}
  \label{fig:4}
\end{figure}


\begin{figure}[t]
  \centering
  \begin{tabular}{cccc}
    \includegraphics[width=0.23\linewidth,height=2.8cm]{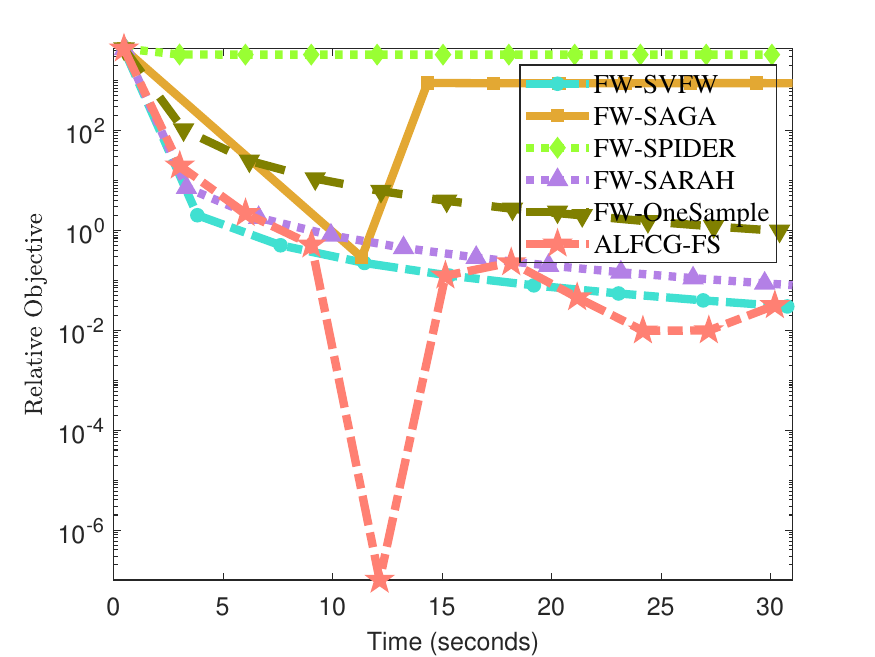} &
    \includegraphics[width=0.23\linewidth,height=2.8cm]{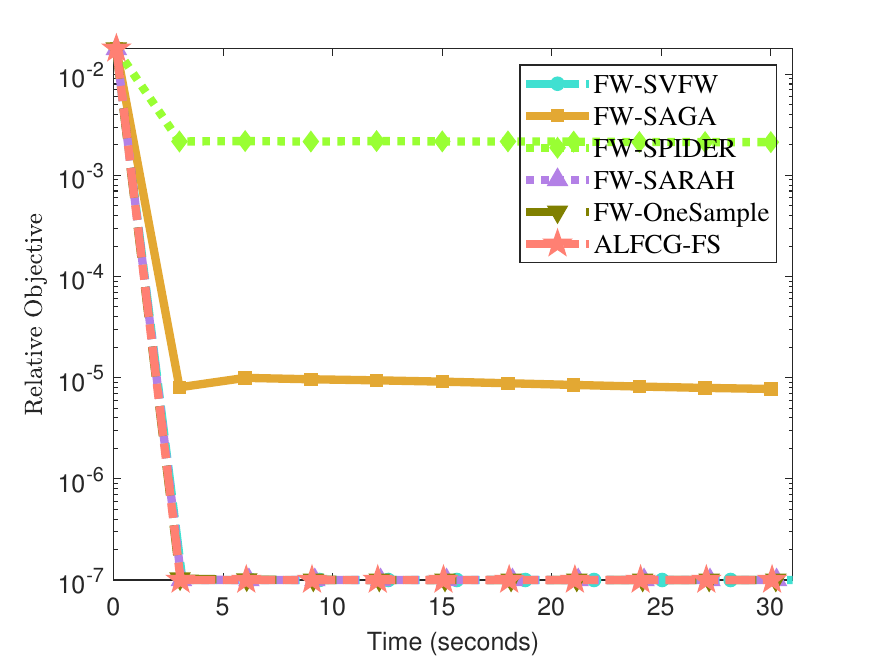} &
    \includegraphics[width=0.23\linewidth,height=2.8cm]{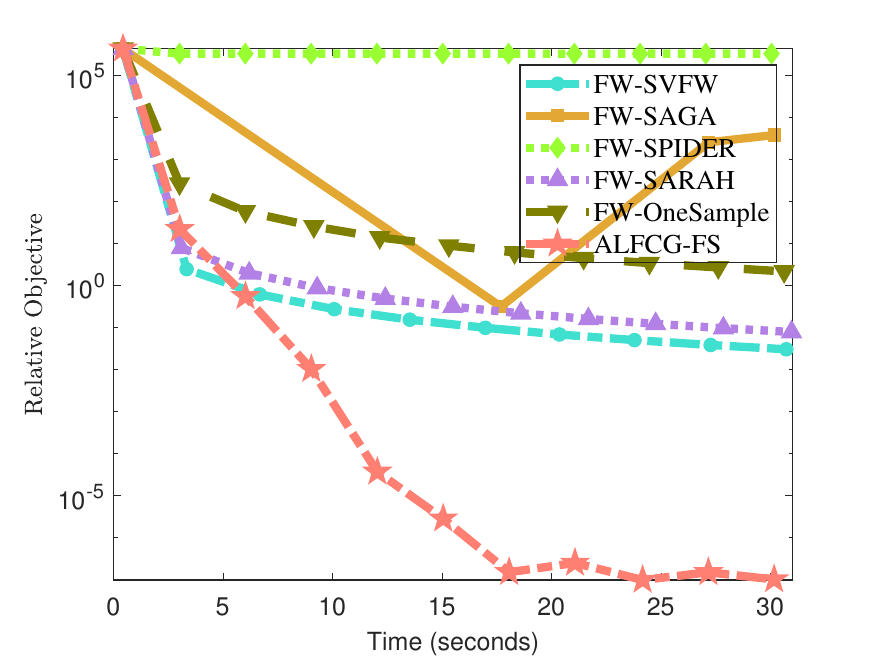} &
    \includegraphics[width=0.23\linewidth,height=2.8cm]{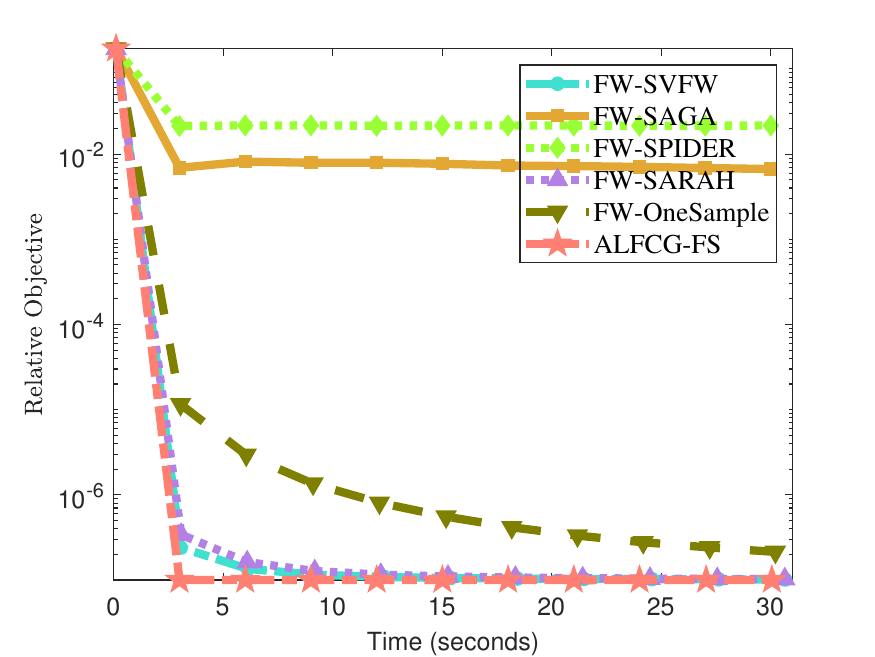} \\
    {\scriptsize (a) randn-30000-5000, $\delta_2=10$} & {\scriptsize (b) mnist-50000-768, $\delta_2=10$} & {\scriptsize (c) randn-30000-5000, $\delta_2=100$} & {\scriptsize (d) mnist-50000-768, $\delta_2=100$}
  \end{tabular}

  \centering
  \begin{tabular}{cccc}
    \includegraphics[width=0.23\linewidth,height=2.8cm]{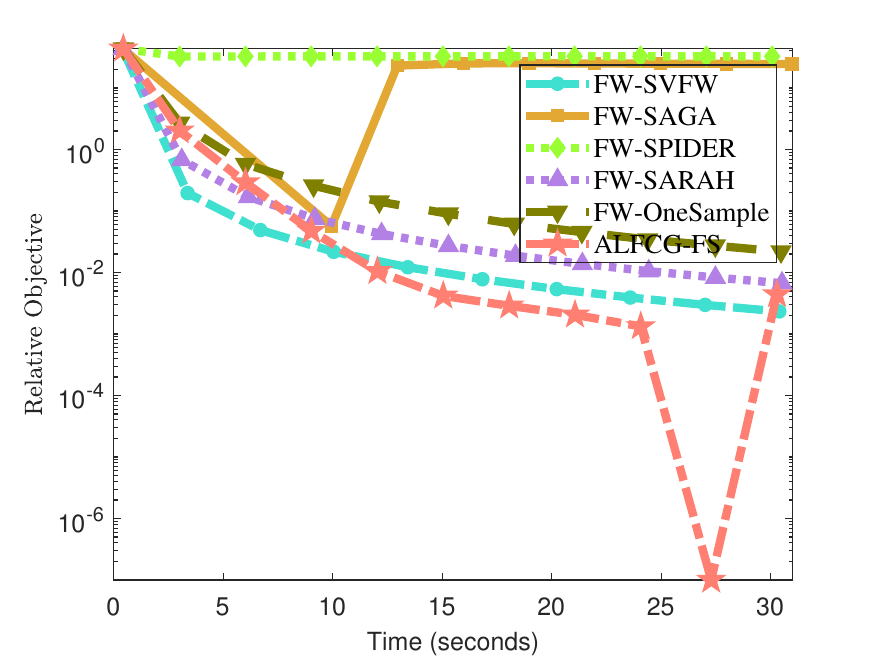} &
    \includegraphics[width=0.23\linewidth,height=2.8cm]{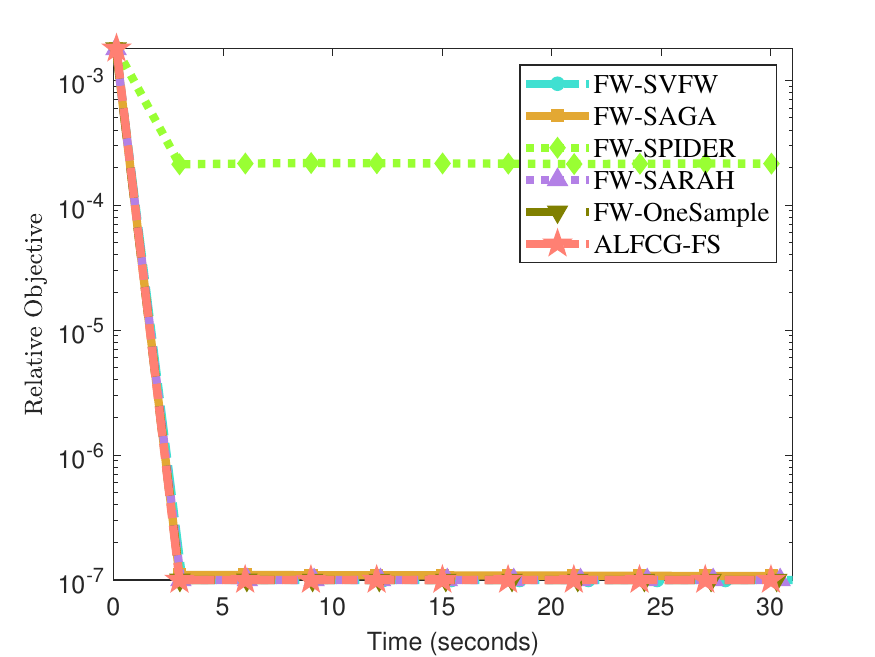} &
    \includegraphics[width=0.23\linewidth,height=2.8cm]{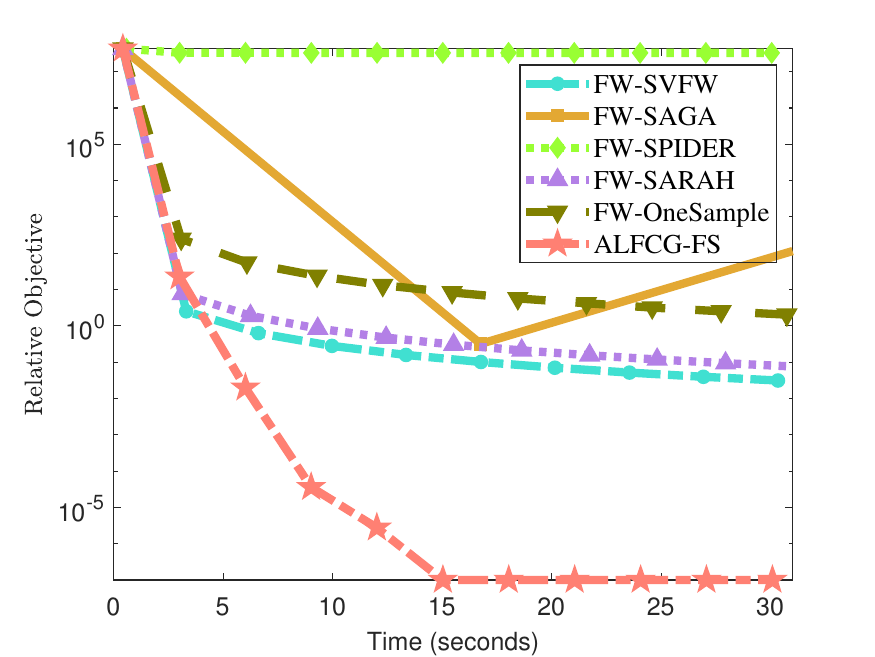} &
    \includegraphics[width=0.23\linewidth,height=2.8cm]{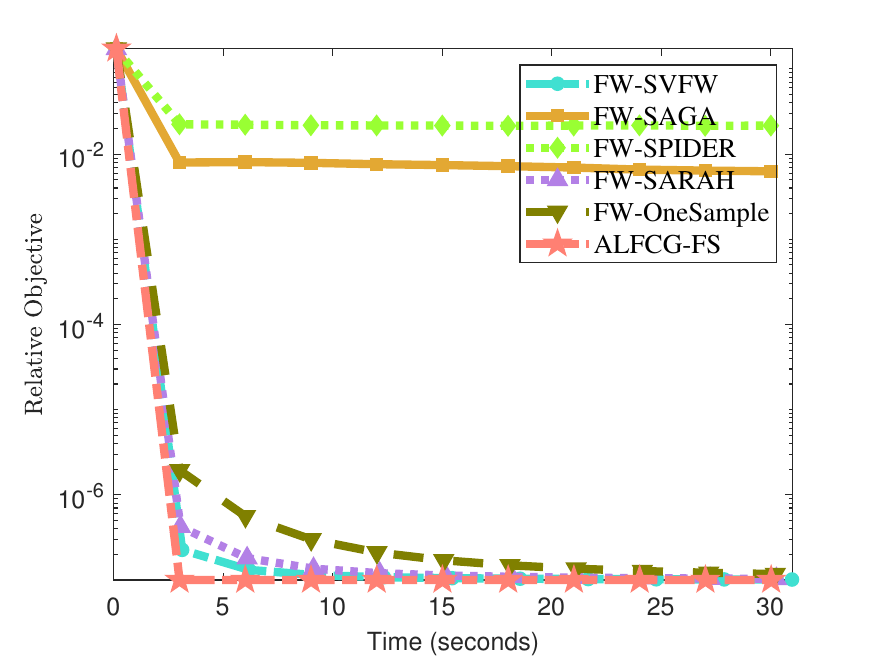} \\
    {\scriptsize (e) randn-30000-5000, $\delta_2=1$} & {\scriptsize (f) mnist-50000-768, $\delta_2=1$} & {\scriptsize (g) randn-30000-5000, $\delta_2=1000$} & {\scriptsize (h) mnist-50000-768, $\delta_2=1000$}
  \end{tabular}
    \caption{Experiment results for the $\ell_p$ norm ball constraint problem for \textbf{finite-sum setting} with varying radius parameter $\delta_2 = \{10,100,1,1000\}$.}
  \label{fig:5}
\end{figure}


\begin{figure}[t]
  \centering
  \begin{tabular}{cccc}
    \includegraphics[width=0.23\linewidth,height=2.8cm]{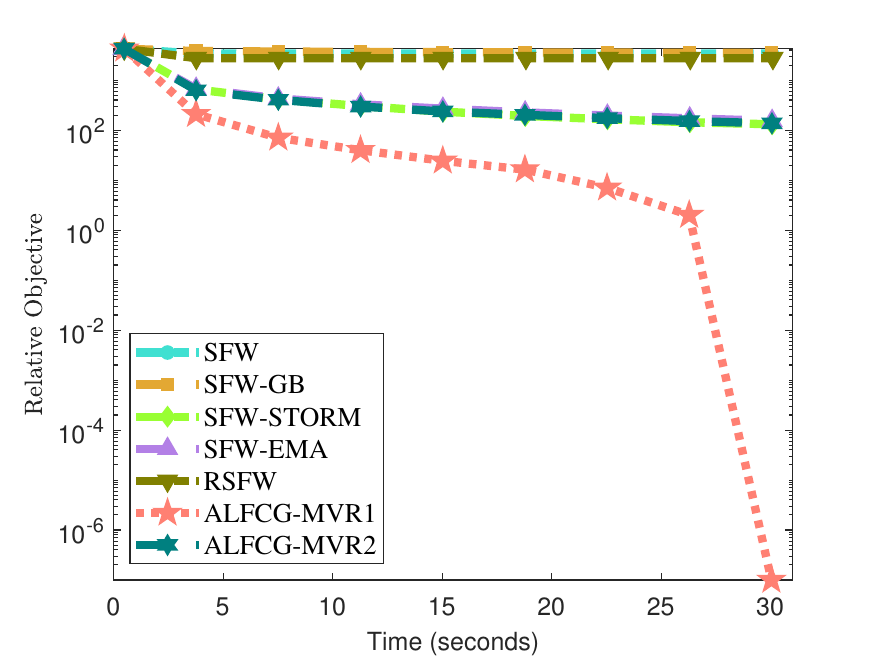} &
    \includegraphics[width=0.23\linewidth,height=2.8cm]{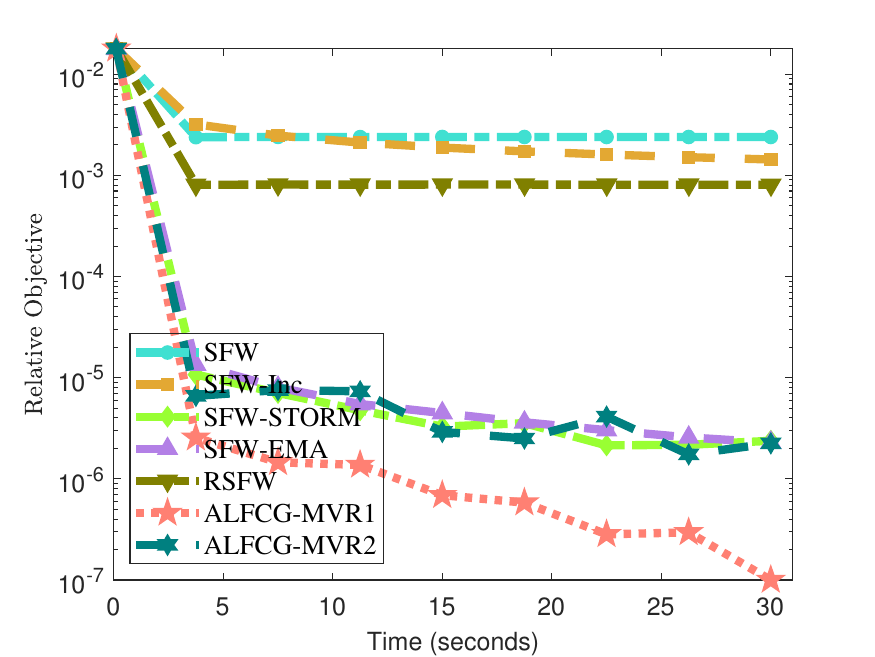} &
    \includegraphics[width=0.23\linewidth,height=2.8cm]{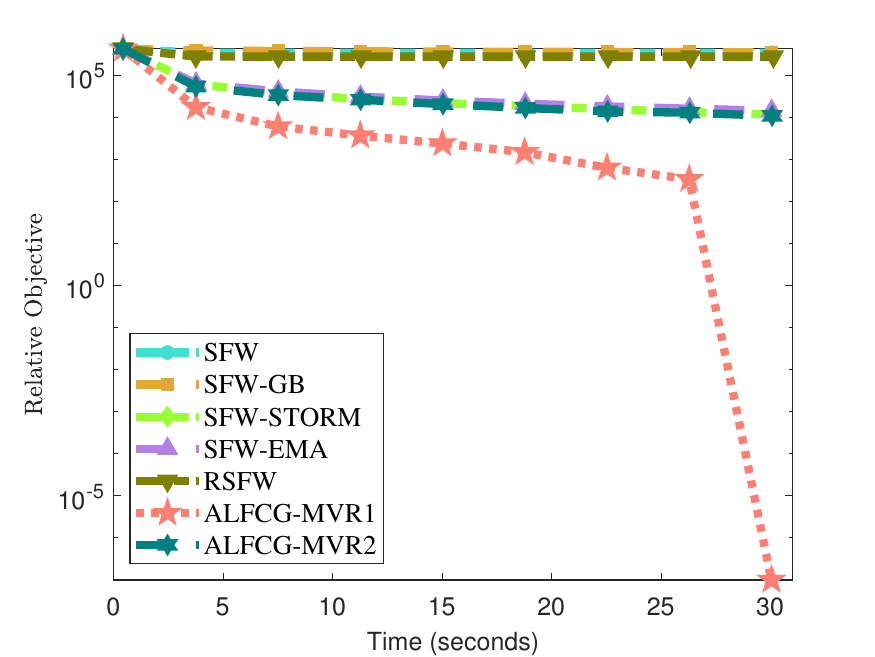} &
    \includegraphics[width=0.23\linewidth,height=2.8cm]{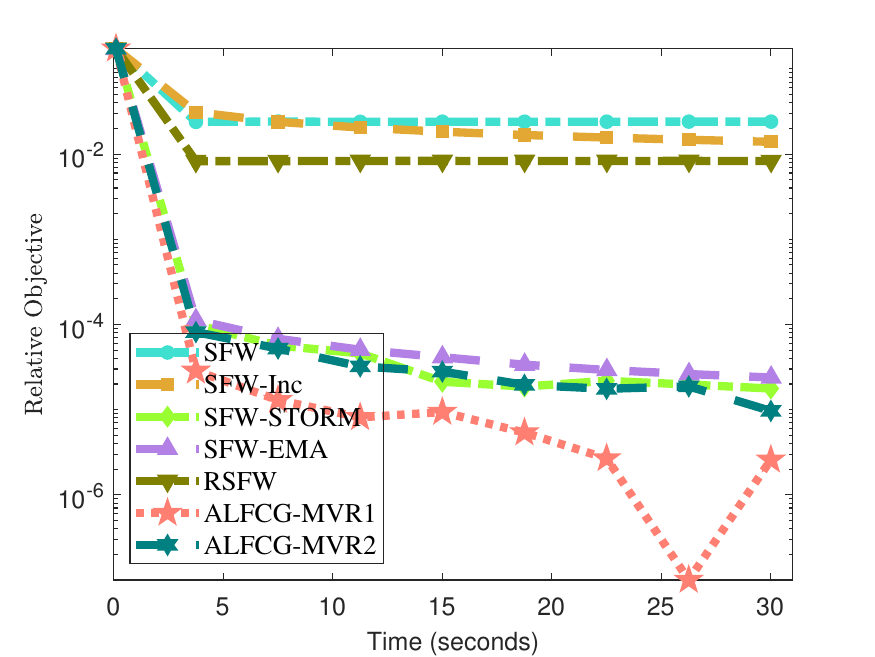} \\
    {\scriptsize (a) randn-30000-5000, $\delta_2=10$} & {\scriptsize (b) mnist-50000-768, $\delta_2=10$} & {\scriptsize (c) randn-30000-5000, $\delta_2=100$} & {\scriptsize (d) mnist-50000-768, $\delta_2=100$}
  \end{tabular}

  \centering
  \begin{tabular}{cccc}
    \includegraphics[width=0.23\linewidth,height=2.8cm]{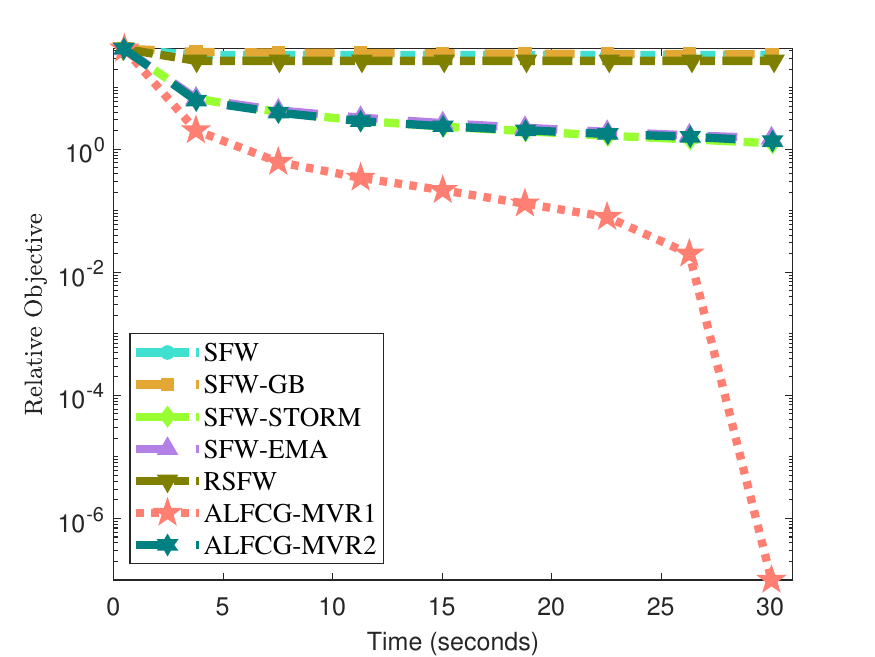} &
    \includegraphics[width=0.23\linewidth,height=2.8cm]{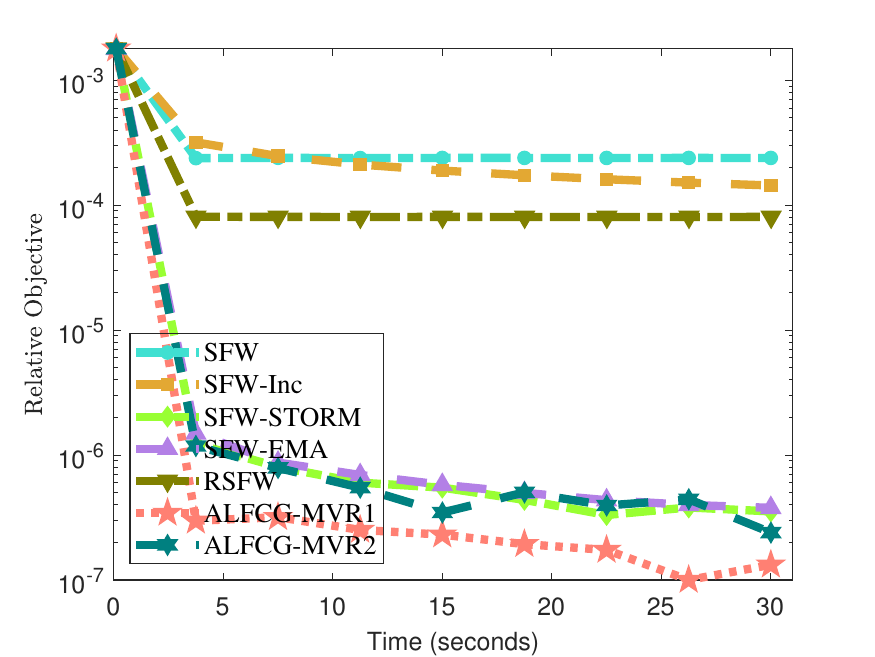} &
    \includegraphics[width=0.23\linewidth,height=2.8cm]{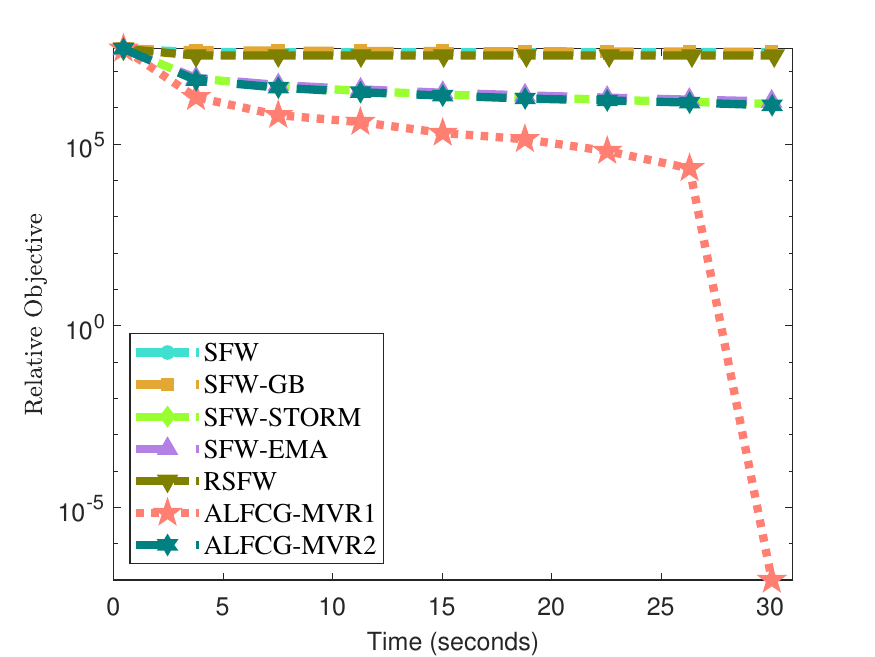} &
    \includegraphics[width=0.23\linewidth,height=2.8cm]{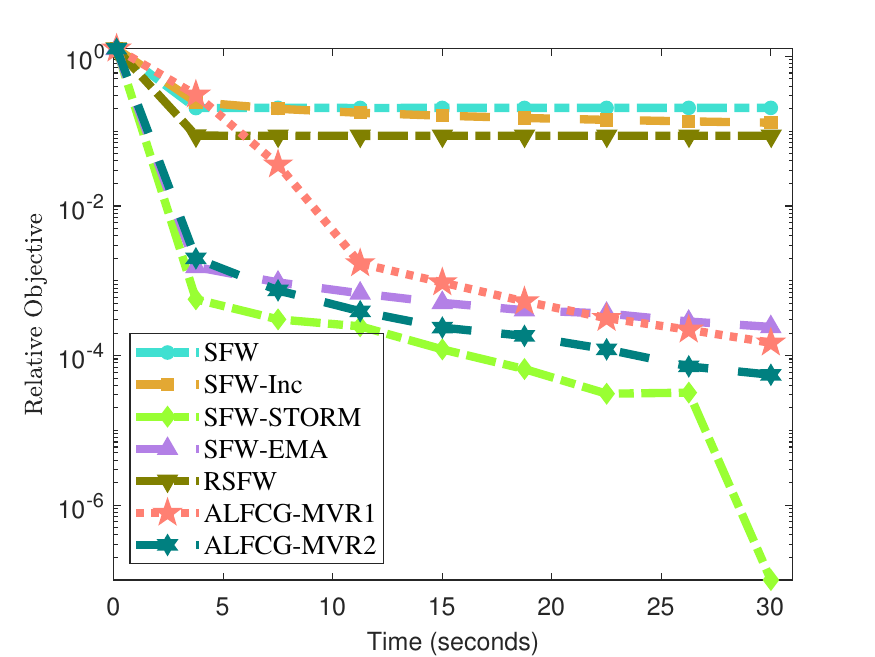} \\
    {\scriptsize (e) randn-30000-5000, $\delta_2=1$} & {\scriptsize (f) mnist-50000-768, $\delta_2=1$} & {\scriptsize (g) randn-30000-5000, $\delta_2=1000$} & {\scriptsize (h) mnist-50000-768, $\delta_2=1000$}
  \end{tabular}
    \caption{Experiment results for the $\ell_p$ norm ball constraint problem for \textbf{expectation setting} with varying radius parameter $\delta_2 = \{10,100,1,1000\}$.}
  \label{fig:6}
\end{figure}

\begin{figure}[t]
  \centering
  \begin{tabular}{cccc}
    \includegraphics[width=0.23\linewidth,height=2.8cm]{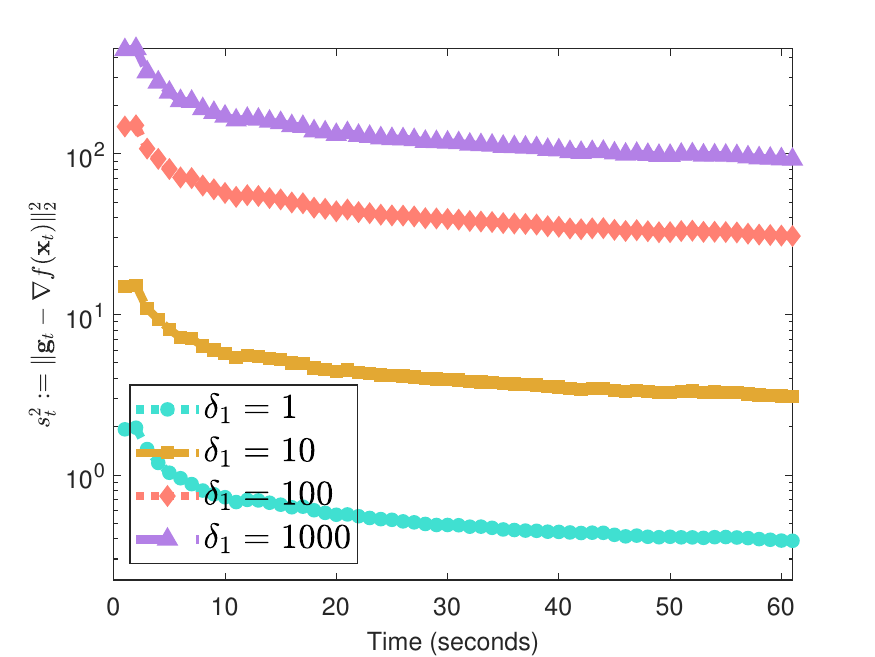} &
    \includegraphics[width=0.23\linewidth,height=2.8cm]{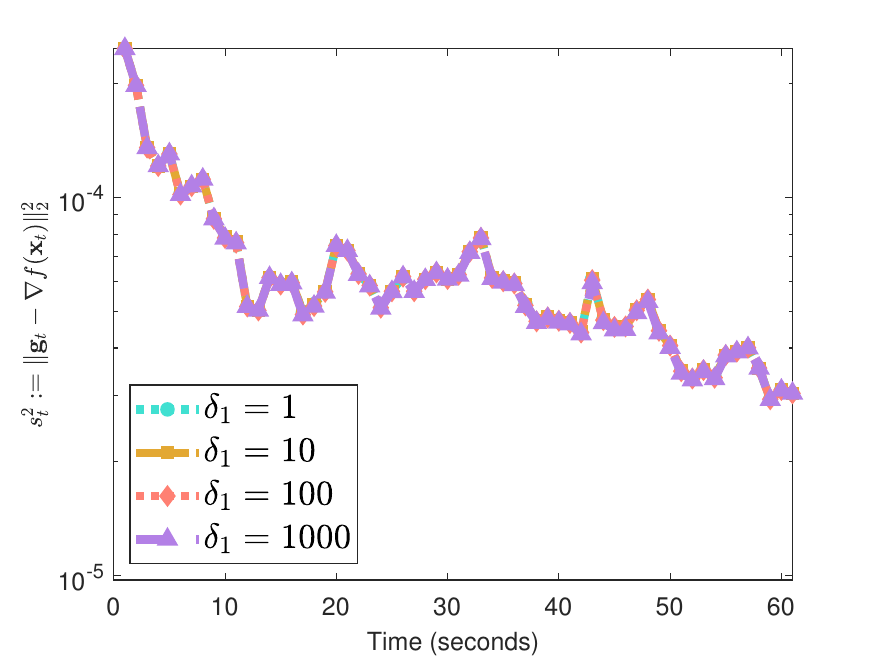} &
    \includegraphics[width=0.23\linewidth,height=2.8cm]{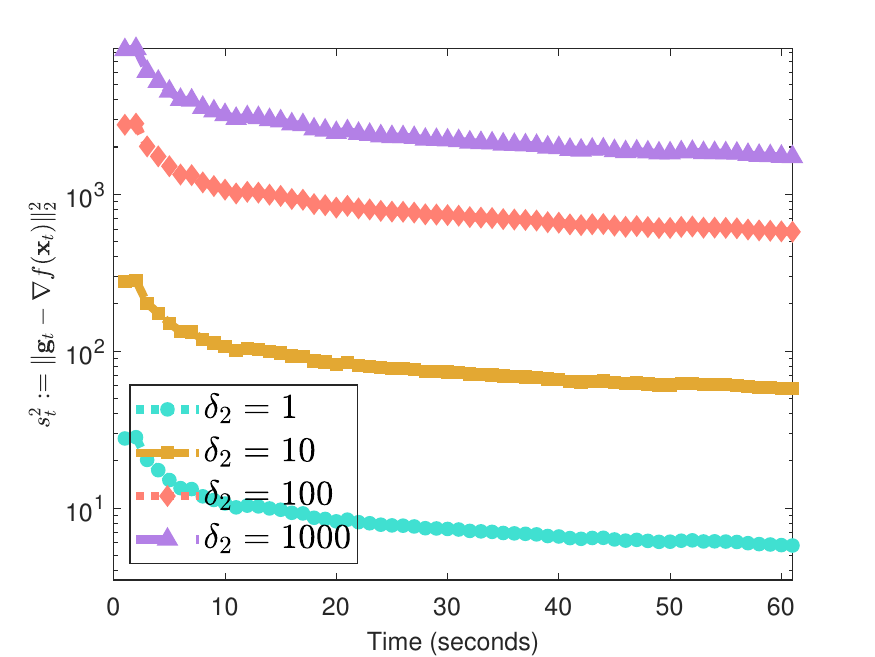} &
    \includegraphics[width=0.23\linewidth,height=2.8cm]{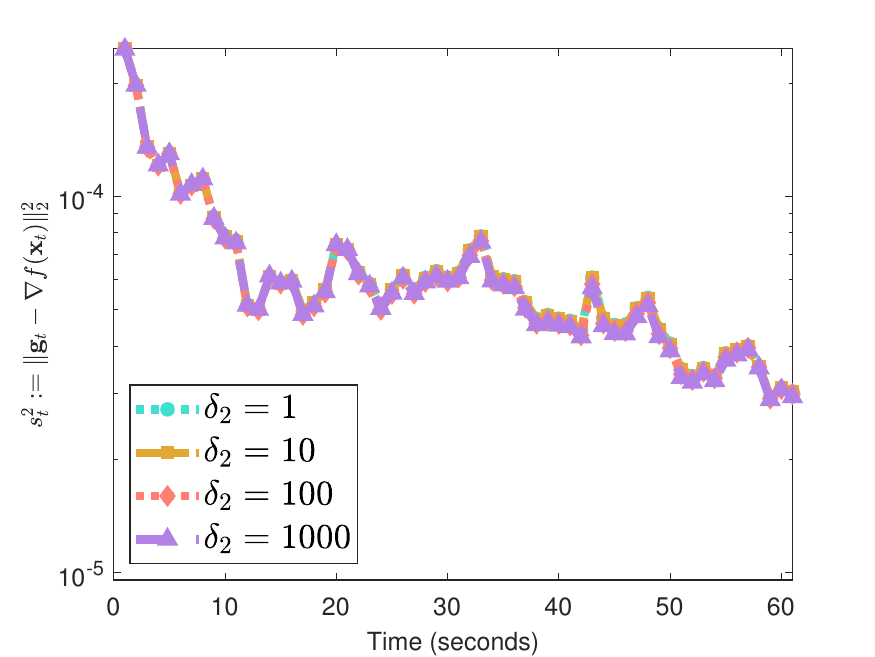} \\
    {\scriptsize (a) randn-30000-5000} & {\scriptsize (b) mnist-50000-768} & {\scriptsize (c) randn-30000-5000} & {\scriptsize (d) mnist-50000-768}
  \end{tabular}
    \caption{The evolution of $s_t^2 := \|\g_t - \nabla f(\x_t)\|_2^2$ for ALFCG-MVR2 (under the expectation setting) on the nuclear norm ball constraint problem (subfigures (a) and (b)) and $\ell_p$-norm ball constraint problem (subfigures (c) and (d)), with varying radius $\delta_1$ or $\delta_2$.}
  \label{fig:7}
\end{figure}

\end{document}

%% file: myinit.tex
\usepackage{graphicx}
\usepackage{epstopdf}

\usepackage{enumitem,amsfonts,amssymb,mathtools,bm,pifont,bbm,listings,xcolor,threeparttable,xcolor}
\usepackage{ulem,hyperref,comment,amsmath}
\usepackage{natbib}
\usepackage{multirow}

\newtheorem{assumption}[theorem]{Assumption}

\def\Rn{\mathbb{R}}  \def\zero{\mathbf{0}}

\def\fro {\mathsf{F}}

\definecolor{mydarkgreen}{RGB}{39,130,67}
\definecolor{mydarkred}{RGB}{192,25,25}

\usepackage{relsize} 

\newcommand{\beq}{\begin{eqnarray}}
\newcommand{\eeq}{\end{eqnarray}}
\newcommand{\beqq}{\begin{equation}}
\newcommand{\eeqq}{\end{equation}}
\newcommand{\bel}{\begin{align}}
\newcommand{\eel}{\end{align}}
\newcommand{\la}{\langle}
\newcommand{\ra}{\rangle}
\newcommand{\noi}{\noindent}
\newcommand{\nn}{\nonumber}
\def\noi{\noindent}
\def\nn{\nonumber}
\def\la{\langle}
\def\ra{\rangle}

\newcommand{\step}[1]{\text{\ding{\numexpr#1+171\relax}}}

\def\ts{\textstyle}